\documentclass{article}
\usepackage[mathscr]{eucal}
\usepackage[cmex10]{amsmath}
\usepackage{epsfig,epsf,psfrag}
\usepackage{amssymb,amsmath,amsthm,amsfonts,latexsym,bm}
\usepackage{amsmath,graphicx,bm,xcolor,url}
\usepackage[caption=false]{subfig} 
\usepackage{array}%array and tabular environments
\usepackage{verbatim}
\usepackage{bm}
\usepackage{cite}
\usepackage{verbatim}
\usepackage{textcomp}
\usepackage{mathrsfs}
\usepackage{multirow}
\usepackage{epstopdf}
\catcode`~=11 \def\UrlSpecials{\do\~{\kern -.15em\lower .7ex\hbox{~}\kern .04em}} \catcode`~=13 

\allowdisplaybreaks[1]
 
\newcommand{\nn}{\nonumber}

\newcommand{\calH}{\mathcal{H}}

\newcommand{\calL}{\mathcal{L}}
\newcommand{\calN}{\mathcal{N}}

\newcommand{\calW}{\mathcal{W}}
\newcommand{\calY}{\mathcal{Y}}

\newcommand{\bb}{\mathbf{b}}

\newcommand{\bG}{\mathbf{G}}

\newcommand{\bI}{\mathbf{I}}

\newcommand{\bbC}{\mathbb{C}}

\DeclareMathAlphabet{\mathbsf}{OT1}{cmss}{bx}{n}
\DeclareMathAlphabet{\mathssf}{OT1}{cmss}{m}{sl}% slanted sans serif

\DeclareSymbolFont{bsfletters}{OT1}{cmss}{bx}{n}  
\DeclareSymbolFont{ssfletters}{OT1}{cmss}{m}{n}
\DeclareMathSymbol{\bsfGamma}{0}{bsfletters}{'000}
\DeclareMathSymbol{\ssfGamma}{0}{ssfletters}{'000}
\DeclareMathSymbol{\bsfDelta}{0}{bsfletters}{'001}
\DeclareMathSymbol{\ssfDelta}{0}{ssfletters}{'001}
\DeclareMathSymbol{\bsfTheta}{0}{bsfletters}{'002}
\DeclareMathSymbol{\ssfTheta}{0}{ssfletters}{'002}
\DeclareMathSymbol{\bsfLambda}{0}{bsfletters}{'003}
\DeclareMathSymbol{\ssfLambda}{0}{ssfletters}{'003}
\DeclareMathSymbol{\bsfXi}{0}{bsfletters}{'004}
\DeclareMathSymbol{\ssfXi}{0}{ssfletters}{'004}
\DeclareMathSymbol{\bsfPi}{0}{bsfletters}{'005}
\DeclareMathSymbol{\ssfPi}{0}{ssfletters}{'005}
\DeclareMathSymbol{\bsfSigma}{0}{bsfletters}{'006}
\DeclareMathSymbol{\ssfSigma}{0}{ssfletters}{'006}
\DeclareMathSymbol{\bsfUpsilon}{0}{bsfletters}{'007}
\DeclareMathSymbol{\ssfUpsilon}{0}{ssfletters}{'007}
\DeclareMathSymbol{\bsfPhi}{0}{bsfletters}{'010}
\DeclareMathSymbol{\ssfPhi}{0}{ssfletters}{'010}
\DeclareMathSymbol{\bsfPsi}{0}{bsfletters}{'011}
\DeclareMathSymbol{\ssfPsi}{0}{ssfletters}{'011}
\DeclareMathSymbol{\bsfOmega}{0}{bsfletters}{'012}
\DeclareMathSymbol{\ssfOmega}{0}{ssfletters}{'012}

\newcommand{\tilA}{\tilde{A}}

\newcommand{\tilf}{\tilde{f}}

\newcommand{\tilg}{\tilde{g}}

\newcommand{\hatP}{\hat{P}}

\newcommand{\tilP}{\tilde{P}}

\newcommand{\hatR}{\hat{R}}

\newcommand{\tilS}{\tilde{S}}

\newcommand{\bepsilon}{\bm{\epsilon}}

\newcommand{\tepsilon}{\tilde{\epsilon}}

\newcommand{\eps}{\varepsilon}
\DeclareMathOperator{\sgn}{sgn}

\newcommand{\bone}{\mathbf{1}}

\newtheorem{theorem}{Theorem} 
\newtheorem{lemma}[theorem]{Lemma}

\newtheorem{example}[theorem]{Example} 
 
\newtheorem{remark}[theorem]{Remark}

\newcommand{\qednew}{\nobreak \ifvmode \relax \else
      \ifdim\lastskip<1.5em \hskip-\lastskip
      \hskip1.5em plus0em minus0.5em \fi \nobreak
      \vrule height0.75em width0.5em depth0.25em\fi}

\usepackage{pdfpages}
\PassOptionsToPackage{numbers, compress}{natbib}
\usepackage[final]{neurips_2022}
\newcommand{\bx}{\mathbf{x}}
\newcommand{\bX}{\mathbf{X}}
\newcommand{\by}{\mathbf{y}}
\newcommand{\bY}{\mathbf{Y}}

\newcommand{\bQ}{\mathbf{Q}}

\newcommand{\calE}{\mathcal{E}}
\newcommand{\calX}{\mathcal{X}}
\newcommand{\calS}{\mathcal{S}}
\newcommand{\calF}{\mathcal{F}}

\newcommand{\calM}{\mathcal{M}}
\newcommand{\calV}{\mathcal{V}}

\newcommand{\calG}{\mathcal{G}} % Free energy function (Markov source)
\newcommand{\tilQ}{\tilde{Q}} % Scaling matrix
\newcommand{\bbR}{\mathbb{R}} % Real Numbers
\newcommand{\bbZ}{\mathbb{Z}} % Integer
\newcommand{\bbE}{\mathbb{E}} % Expectation
\newcommand{\bbP}{\mathbb{P}} % Probability of a set

\usepackage[utf8]{inputenc} % allow utf-8 input
\usepackage[T1]{fontenc}    % use 8-bit T1 fonts
\usepackage{hyperref}       % hyperlinks
\usepackage{url}            % simple URL typesetting
\usepackage{booktabs}       % professional-quality tables
\usepackage{amsfonts}       % blackboard math symbols
\usepackage{nicefrac}       % compact symbols for 1/2, etc.
\usepackage{microtype}      % microtypography
\usepackage{xcolor}         % colors

\title{Generalization Error Bounds on Deep Learning with Markov Datasets}
\author{%
	Lan V. Truong\thanks{Use footnote for providing further information
		about author (webpage, alternative address)---\emph{not} for acknowledging
		funding agencies.} \\
	Department of Engineering\\
	University of Cambridge\\
	Cambridge, CB2 1PZ \\
	\texttt{lt407@cam.ac.uk} \\
}

\begin{document}
	\maketitle
\begin{abstract}
In this paper, we derive upper bounds on generalization errors for deep neural networks with Markov datasets. These bounds are developed based on Koltchinskii and Panchenko's approach for bounding the generalization error of combined classifiers with i.i.d. datasets. The development of new symmetrization inequalities in high-dimensional probability for Markov chains is a key element in our extension, where the absolute spectral gap of the infinitesimal generator of the Markov chain plays a key parameter in these inequalities. We also propose a simple method to convert these bounds and other similar ones in traditional deep learning and machine learning to Bayesian counterparts for both i.i.d. and Markov datasets. Extensions to $m$-order homogeneous Markov chains such as AR and ARMA models and mixtures of several Markov data services are given.
\end{abstract}

\section{Introduction} 
In statistical learning theory, understanding generalization for neural networks is among the most challenging tasks. The standard approach to this problem was developed by Vapnik \citep{Vap98}, and it is based on bounding the difference between the prediction error and the training error. These bounds are expressed in terms of the so called VC-dimension of the class. However, these bounds are very loose when the VC-dimension of the class can be very large, or even infinite. In 1998, several authors \citep{Bartlett1998,Bartlett1999} suggested another class of upper bounds on generalization error that are expressed in terms of the empirical distribution of the margin of the predictor (the classifier). Later, Koltchinskii and Panchenko \citep{Koltchinskii2002} proposed new probabilistic upper bounds on generalization error of the combination of many complex classifiers such as deep neural networks. These bounds were developed based on the general results of the theory of Gaussian, Rademacher, and empirical processes in terms of general functions of the margins, satisfying a Lipschitz condition. They improved previously known bounds on generalization error of convex combination of classifiers. 

In the context of supervised classification, PAC-Bayesian bounds have proved to be the tightest \citep{Langford2003,McAllester2004,Ambroladze2007}.  Several recent works have focused
on gradient descent based PAC-Bayesian algorithms,
aiming to minimise a generalisation bound for stochastic classifiers \citep{Dziugaite2017,Zhou2019, Biggs2021}. Most of these studies use a surrogate loss to avoid dealing with the zero-gradient of the misclassification loss. Several authors used other methods to estimate of the misclassification error with a non-zero gradient by proposing new training algorithms to evaluate the optimal output distribution in PAC-Bayesian bounds analytically \citep{McAllester1998,Eugenio2021a,Eugenio2021}.
Recently, there have been some interesting works which use information-theoretic approach to find PAC-bounds on generalization errors for machine learning \citep{XuRaginskyNIPS17, Esposito2021} and deep learning \citep{Jakubovitz2108}. 

All of the above-mentioned bounds are derived based on the assumption that the dataset is generated by an i.i.d. process with unknown distribution. However, in many applications in machine learning such as speech, handwriting, gesture recognition, and bio-informatics, the samples of data are usually correlated. Some of these datasets are time-series ones with stationary distributions such as samples via MCMC, finite-state random walks, or random walks on graph. In this work, we develop some upper bounds on generalization errors for deep neural networks with Markov or hidden Markov datasets. Our bounds are derived based on the same approach as Koltchinskii and Panchenko \citep{Koltchinskii2002}. To deal with the Markov structure of the datasets, we need to develop some new techniques in this work. The development of new symmetrization inequalities in high-dimensional probability for Markov chains is a key element in our extension, where the absolute spectral gap of the infinitesimal generator of the Markov chain plays as a key parameter in these inequalities. Furthermore, we also apply our results to $m$-order Markov chains such as AR and ARMA models and mixtures of Markov chains. Finally, a simple method to convert all our bounds for traditional deep learning to counterparts for Bayesian deep learning is given. Our method can be applied to convert other similar bounds for i.i.d. datasets in the research literature as well. Bayesian deep learning was introduced by \citep{MacKay1992,Mackay1995ProbableNA}. The key distinguishable property of a Bayesian approach is marginalization, rather than using a single setting of weights in (traditional) deep learning \cite{Wilson20}. 

Bayesian marginalization can particularly improve the accuracy and calibration of modern deep neural networks, which are typically underspecified by the data, and can represent many compelling but different solutions. Analysis of machine learning algorithms for Markov and Hidden Markov datasets already appeared in research literature \citep{Duchi2011ErgodicMD, Wang2019AML,Truong2022OnLM}. In practice, some real-world time-series datasets are not stationary Markov chains. However, we can approximate time-series datasets by stationary Markov chains in many applications. There are also some other methods of approximating non-stationary Markov chains by stationary ones via MA and ARMA models in the statistical research literature. The i.i.d. dataset is a special case of the Markov dataset with stationary distribution. 
\section{Preliminaries} \label{sec1}
\subsection{Mathematical Backgrounds}\label{sec:background}
Let a Markov chain $\{X_n\}_{n=1}^{\infty}$ on a state space $\calS$ with transition kernel $Q(x,dy)$ and the initial state $X_1 \sim \nu$, where $\calS$ is a Polish space in $\bbR$. In this paper, we consider the Markov chains which are irreducible and positive-recurrent, so the existence of a stationary distribution $\pi$ is guaranteed. An irreducible and recurrent Markov chain on an infinite state-space is called Harris chain \citep{TR1979}. A Markov chain is called \emph{reversible} if the following detailed balance condition is satisfied:
\begin{align}
\pi(dx)Q(x,dy)=\pi(dy)Q(y,dx),\qquad \forall x, y \in \calS.
\end{align} 
Define
\begin{align}
d(t)=\sup_{x \in \calS} d_{\rm{TV}}(Q^t(x,\cdot),\pi), \qquad  t_{\rm{mix}}(\eps):=\min\{t: d(t)\leq \eps\},
\end{align}
and
\begin{align}
\tau_{\min}:=\inf_{0\leq \eps\leq 1}t_{\rm{mix}}(\eps)\bigg(\frac{2-\eps}{1-\eps}\bigg)^2,\qquad 
t_{\rm{mix}}:=t_{\rm{mix}}(1/4) \label{deftaumin}.
\end{align}

Let $L_2(\pi)$ be the Hilbert space of complex valued measurable functions on $\calS$ that are square integrable w.r.t. $\pi$. We endow $L_2(\pi)$ with inner product $\langle f,g \rangle:= \int f g^* d\pi$, and norm $\|f\|_{2,\pi}:=\langle f, f\rangle_{\pi}^{1/2}$. Let $E_{\pi}$ be the associated averaging operator defined by $(E_{\pi})(x,y)=\pi(y), \forall x,y \in \calS$, and 
\begin{align}
\lambda=\|Q-E_{\pi}\|_{L_2(\pi)\to L_2(\pi)} \label{defL2gap},
\end{align} where
$
\|B\|_{L_2(\pi)\to L_2(\pi)}=\max_{v: \|v\|_{2,\pi}=1}\|Bv\|_{2,\pi}.
$
$Q$ can be viewed as a linear operator (infinitesimal generator) on $L_2(\pi)$, denoted by $\bQ$, defined as $(\bQ f)(x):=\bbE_{Q(x,\cdot)}(f)$, and the reversibility is equivalent to the self-adjointness of $\bQ$. The operator $\bQ$ acts on measures on the left, creating a measure $\mu \bQ$, that is, for every measurable subset $A$ of $\calS$, $\mu \bQ (A):=\int_{x \in \calS} Q(x,A)\mu(dx)$. For a Markov chain with stationary distribution $\pi$, we define the \emph{spectrum} of the chain as
\begin{align}
S_2:=\big\{\xi \in \bbC: (\xi \bI-\bQ) \enspace \mbox{is not invertible on} \enspace L_2(\pi)\big\}.
\end{align}
It is known that $\lambda=1-\gamma^*$ \cite{Daniel2015}, 
where
\begin{align}
\gamma^*&:=\begin{cases} 1-\sup\{|\xi|: \xi \in \calS_2, \xi \neq 1\},\nn\\
\qquad \qquad \mbox{if eigenvalue $1$ has multiplicity $1$,}\\
0,\qquad\qquad \mbox{otherwise}\end{cases}
\end{align} is the \emph{the absolute spectral gap} of the Markov chain. The absolute spectral gap can be bounded by the mixing time $t_{\rm{mix}}$ of the Markov chain by the following expression:
\begin{align}
\bigg(\frac{1}{\gamma^*}-1\bigg)\log 2 \leq t_{\rm{mix}} \leq \frac{\log (4/\pi_*)}{\gamma_*},	
\end{align}
where $\pi_*=\min_{x\in \calS} \pi_x$ is the \emph{minimum stationary probability}, which is positive if $Q^k>0$ (entry-wise positive) for some $k\geq 1$. See \cite{WK19ALT} for more detailed discussions. In \citep{Combes2019EE, WK19ALT}, the authors provided algorithms to estimate $t_{\rm{mix}}$ and $\gamma^*$ from a single trajectory. 

Define
\begin{align}
\calM_2:=\bigg\{\nu \in \calM(\calS): \bigg\|\frac{dv}{d\pi}\bigg\|_2<\infty\bigg\},
\end{align} where $\|\cdot\|_2$ is the standard $L_2$ norm in the Hilbert space of complex valued measurable functions on $\calS$. 
%We also assume that
%\begin{align}
%c:= \bigg\|\frac{dv}{d\pi}\bigg\|_{\infty}<\infty \label{defc}.
%\end{align}
%For $\{X_n\}_{n=1}^{\infty}$ being an i.i.d. sequence, which is a special case of Markov chain, or an arbitrary homogeneousness Markov chain with $X_1\sim \pi$, it holds that $c=1$. 
\subsection{Problem settings} \label{sub:setting}
In this paper, we consider a uniformly bounded class of functions: $\calF:=\big\{f: \calS \to \bbR\big\}$
such that $\sup_{f\in \calF} \big\|f\big\|_{\infty} \leq M$ for some finite constant $M$. 

Define the probability measure 
$
P(A):=\int_{A} \pi(x) dx,
$ for any measurable set $A \in \calS$.  In addition, let $P_n$ be the empirical measure based on the sample $(X_1,X_2,\cdots,X_n)$, i.e.,
$
P_n:=\frac{1}{n}\sum_{i=1}^n \delta_{X_i}.
$  We also denote $Pf:=\int_S f dP$ and $P_nf:=\int_S f dP_n$. Then, we have
\begin{align}
Pf= \int_{\calS} f(x) \pi(x) dx \quad \mbox{and}\quad P_nf=\frac{1}{n}\sum_{i=1}^n f(X_i) \label{a3b}.
\end{align}

On the Banach space of uniformly bounded functions $\calF$, define an infinity norm:
$
\|Y\|_{\calF}=\sup_{f \in \calF} |Y(f)|.
$
Let
\begin{align}
G_n(\calF):=\bbE\bigg[\bigg\|n^{-1}\sum_{i=1}^n g_i \delta_{X_i}\bigg\|_{\calF}\bigg],
\end{align}
where $\{g_i\}$ is a sequence of i.i.d. standard normal variables, independent of $\{X_i\}$. We will call $n \mapsto G_n(\calF)$ \emph{the Gaussian complexity function} of the class $\calF$. 

Similarly, we define
\begin{align}
R_n(\calF):=\bbE\bigg[\bigg\|n^{-1}\sum_{i=1}^n \eps_i \delta_{X_i}\bigg\|_{\calF}\bigg] \label{defRM},
\end{align} 
and
\begin{align}
P_n^0:=n^{-1}\sum_{i=1}^n \eps_i \delta_{X_i},
\end{align}
where $\{\eps_i\}$ is a sequence of i.i.d. Rademacher (taking values $+1$ and $-1$ with probability $1/2$ each) random variables, independent of $\{X_i\}$. We will call $n\mapsto R_n(\calF)$ \emph{the Rademacher complexity function} of the class $\calF$. 

For times-series datasets in machine learning, we can assume that feature vectors are generated by a Markov chain $\{X_n\}_{n=1}^{\infty}$ with stochastic matrix $Q$, and $\{Y_n\}_{n=1}^{\infty}$ is the corresponding sequence of labels. Furthermore, $Q$ is irreducible and recurrent on some finite set $\calS$.  An i.i.d. sequence of feature vectors can be considered as a special Markov chain where $Q(x,x')$ only depends on $x'$. In the supervised learning, the sequence of labels $\{Y_n\}_{n=1}^{\infty}$ can be considered as being generated by a Hidden Markov Model (HMM), where the emission probability $P_{Y_n|X_n}(y|x)=g(x,y)$ for all $n\geq 1$ and $g: \calS \times \calY \to \bbR_+$. It is easy to see that $\{(X_n,Y_n)\}_{n=1}^{\infty}$ is a Markov chain with the transition probability
\begin{align}
P_{X_{n+1}Y_{n+1}|X_n Y_n}(x_{n+1}, y_{n+1}|x_n,y_n)=Q(x_n,x_{n+1}) g(x_{n+1},y_{n+1}).
\end{align}  Let $\tilQ(x_1,y_1,x_2,y_2):=Q(x_1,x_2) g(x_2,y_2)$ for all $x_1,x_2 \in \calS$ and $y_1,y_2 \in \calY$, which is the transition probability of the Markov chain $\{(X_n,Y_n)\}_{n=1}^{\infty}$ on $\tilS:=\calS \times \calY$. Then, it is not hard to see that $\{(X_n,Y_n)\}_{n=1}^{\infty}$ is irreducible and recurrent on $\tilS$, so it has a stationary distribution, say $\tilde{\pi}$. The associated following probability measure is defined as
\begin{align}
P(A):=\int_{\calS \times \calY} \tilde{\pi}(x,y) dx dy,
\end{align}
and the empirical distribution $P_n$ based on the observations $\{(X_k,Y_k)\}_{k=1}^n$ is
\begin{align}
P_n:=\frac{1}{n}\sum_{k=1}^n \delta_{X_k,Y_k}.
\end{align}
\subsection{Contributions}
In this paper, we aim to develop a set of novel upper bounds on \emph{the generalization errors} for deep neural networks with Markov dataset. More specially, our target is to find a relationship between $Pf$ and $P_nf$ which holds for all $f \in \calF$ in terms of Gaussian and Rademacher complexities. Our main contributions include: 
\begin{itemize}
	\item We develop general bounds on generalization errors for machine learning (and deep learning) on Markov datasets. 
	\item Since the dataset is non-i.i.d., the standard symmetrization inequalities in high-dimensional probability can not be applied. In this work, we extend some symmetrization inequalities for i.i.d. random processes to Markov ones.
	\item We propose a new method to convert all the bounds for machine learning (and deep learning) models to Bayesian settings.
	\item Extensions to $m$-order homogeneous Markov chains such as AR and ARMA models and mixtures of several Markov services are given.
\end{itemize}
\section{Main Results} \label{sec:mainmod}
\subsection{Probabilistic Bounds for General Function Classes} \label{secsub1mod}
In this section, we develop probabilistic bounds for general function classes in terms of Gaussian and Rademacher complexities.
 
First, we prove the following key lemma, which is an extension of the symmetrization inequality for i.i.d. sequences (for example, \citep{Vaartbook1996}) to a new version for Markov sequences $\{X_n\}_{n=1}^{\infty}$ with the stationary distribution $\pi$ and the initial distribution $\nu \in \calM_2$:
\begin{lemma}\label{lem:keymod} Let $\calF$ be a class of functions which are uniformly bounded by $M$. For all $n\in \bbZ^+$, define
	\begin{align}
	A_n&:=\sqrt{\frac{2M}{n(1-\lambda)}+ \frac{64 M^2}{n^2(1-\lambda)^2}\bigg\|\frac{dv}{d\pi}-1\bigg\|_2} \label{defAtn},\\
	\tilA_n&:=\frac{M}{2n}\bigg[\sqrt{2 \tau_{\min} n\log n}+\sqrt{n}+4\bigg]\label{deftilAtn}.
	\end{align}
	Then, the following holds:
	\begin{align}
	\frac{1}{2}\bbE\big[\|P_n^0\|_{\calF}\big]-\tilA_n \leq\bbE\big[\big\|P_n-P\big\|_{\calF}\big]  \leq 2 \bbE\big[\|P_n^0\|_{\calF}\big] +A_n \label{F1eq},
	\end{align}
	where
	\begin{align}
	\|P_n^0\|_{\calF}:=	\sup_{f\in \calF}\bigg|\frac{1}{n}\sum_{i=1}^n \eps_i f(X_i)\bigg|.
	\end{align}
\end{lemma}

The proof of this lemma can be found in Appendix A of the supplement material. Compared with the i.i.d. case, the symmetrization inequality for Markov chain in Lemma \ref{lem:keymod} are different in two perspectives: (1) The expectation $\bbE\big[\|P_n^0\|_{\calF}\big]$ is now is under the joint distributions of Markov chain and Rademacher random variables and (2) The term $A_n$ appears in both lower and upper bounds to compensate for the difference between the initial distribution $\nu$ and the stationary distribution $\pi$ of the Markov chain\footnote{This difference causes a burn-in time \cite{Rudolf2011} which is the time between the initial and first time that the Markov chain is stationary.}. Later, we will see that $A_n$ represents the effects of data structures on the generalization errors in deep learning.  

By applying Lemma \ref{lem:keymod}, the following theorem can be proved. See a detailed proof in Appendix C in the supplement material.
\begin{theorem} \label{cor:cor1mod} Denote by
	\begin{align}
	B_n:=\sqrt{\frac{2}{n(1-\lambda)}+ \frac{64 }{n^2(1-\lambda)^2}\bigg\|\frac{dv}{d\pi}-1\bigg\|_2} \label{defBtnmod},
	\end{align}
Let $\varphi$ be a non-increasing function such that $\varphi(x)\geq \bone_{(-\infty,0]}(x)$ for all $x \in \bbR$. For any $t>0$,
	\begin{align}
	&\bbP\bigg(\exists f\in \calF: P\{f\leq 0\}> \inf_{\delta \in (0,1]}\bigg[P_n\varphi\bigg(\frac{f}{\delta}\bigg)+  \frac{8L(\varphi)}{\delta} R_n(\calF)\nn\\
	&\qquad + \big(t+\sqrt{\log \log_2 2\delta^{-1}}\big) \sqrt{\frac{\tau_{\min}}{n}} + B_n\bigg]\bigg) \leq \frac{\pi^2}{3}\exp(-2t^2)\label{pet1amod}
	\end{align} 
	and
	\begin{align}
	&\bbP\bigg(\exists f\in \calF: P\{f\leq 0\}> \inf_{\delta \in (0,1]}\bigg[P_n\varphi\bigg(\frac{f}{\delta}\bigg)+\frac{2 L(\varphi)\sqrt{2\pi}}{\delta}G_n(\calF)\nn\\
	&
	 \qquad +\big(t+\sqrt{\log \log_2 2\delta^{-1}}\big) \sqrt{\frac{\tau_{\min}}{n}}+  \frac{2}{\sqrt{n}}+B_n\bigg]\bigg) \leq \frac{\pi^2}{3}\exp(-2t^2) \label{pet2amod}.
	\end{align} 	
\end{theorem}
Since $B_n=O\big(1/{\sqrt{n}}\big)$, Theorem \ref{cor:cor1mod} shows that with high probability, the generalization error can be bounded by Rademacher or Gaussian complexity functions plus an $O(1/\sqrt{n})$ term, where $n$ is the length of the training set. This fact also happens in i.i.d. case \citep{Koltchinskii2002}. However, because the dependency among samples in Markov chain, the constant in $O(1/\sqrt{n})$ term is larger than the i.i.d. case. 

It follows, in particular, in the example of the voting methods of combining classifiers \cite{Bartlett1998}, from Theorem \ref{cor:cor1mod}, we achieve the following PAC-bound:
\begin{align}
P\{\tilf\leq 0\}&\leq \inf_{\delta \in (0,1]}\bigg[P_n\{\tilf\leq \delta\}+  \frac{8 C}{\delta}  \sqrt{\frac{V(\calH)}{n}} \nn\\
&\qquad \qquad  + B_n +\bigg(\sqrt{\frac{1}{2}\log \frac{\pi^2}{3\alpha}}+\sqrt{\log\log_2 2\delta^{-1}}\bigg)\sqrt{\frac{\tau_{\min}}{n}}\bigg]
\end{align} with probability at least $1-\alpha$ (PAC-Bayes bound), where $V(\calH)$ is the VC-dimension of the class $\calH$ and $C$ is some positive constant. 
\subsection{Bounding the generalization error in deep neural networks} \label{sec2:mod}
In this section, we consider the same example as \citep[Section 6]{Koltchinskii2002}. However, we assume that feature vectors in the dataset are generated by a Markov chain instead of an i.i.d. process. Let $\calH$ be a the class of all uniformly bounded functions $f: \calS \to \bbR$. $\calH$ is called the class of \emph{base functions}. 

Consider a feed-forward neural network with $l$ layers of neurons
$
V=\{v_i\} \cup \bigcup_{j=0}^l V_j
$
where $V_l=\{v_o\}$. The neurons $v_i$ and $v_o$ are called the input and the output neurons, respectively. To define the network, we will assign the labels to the neurons in the following way. Each of the base neurons is labelled by a function from the base class $\calH$. Each neuron of the $j$-th layer $V_j$, where $j\geq 1$, is labelled by a vector $w:=(w_1,w_2,\cdots,w_m) \in \bbR^m$, where $m$ is the number of inputs of the neuron. Here, $w$ will be called the vector of weights of the neuron. 

Given a Borel function $\sigma$ from $\bbR$ into $[-1,1]$ (for example, sigmoid function) and a vector $w:=(w_1,w_2,\cdots,w_m)\in \bbR^m$ where $m=|\calH|+1$, let
\begin{align}
N_{\sigma,w}:\bbR^m \to \bbR, N_{\sigma, w}(u_1,u_2,\cdots, u_m):=\sigma\bigg(\sum_{j=1}^m w_j u_j\bigg).
\end{align}
Let $\sigma_j: j\geq 1$ be functions from $\bbR$ into $[-1,1]$, satisfying the Lipschitz conditions
\begin{align}
\big|\sigma_j(u)-\sigma_j(v)\big|\leq L_j|u-v|, \qquad u, v \in \bbR. 
\end{align}
The neural network works can be formed as the following. The input neuron inputs an instance $x \in \calS$. A base neuron computes the value of the base function on this instance and outputs the value through its output edges. A neuron in the $j$-th layer ($j\geq 1$) computes and outputs through its output edges the value $N_{\sigma_j, w}(u_1,u_2,\cdots,u_m)$ (where $u_1,u_2,\cdots,u_m$ are the values of the inputs of the neuron). The network outputs the value $f(x)$ (of a function $f$ it computes) through the output edge. 

We denote by $\calN_l$ the set of such networks. We call $\calN_l$ the class of feed-forward neural networks with base $\calH$ and $l$ layers of neurons (and with sigmoid $\sigma_j$). Let $\calN_{\infty}:=\bigcup_{j=1}^{\infty} \calN_j$. Define $\calH_0:=\calH$, and then recursively
\begin{align}
\calH_j&:=\bigg\{N_{\sigma_j,w}(h_1,h_2,\cdots, h_m): m\geq 0, h_i \in \calH_{j-1}, w \in \bbR^m\bigg\} \cup \calH_{j-1}.
\end{align}
Denote $\calH_{\infty}:=\bigcup_{j=1}^{\infty}\calH_j$. Clearly, $\calH_{\infty}$ includes all the functions computable by feed-forward neural networks with base $\calH$. 

Let $\{b_j\}$ be a sequence of positive numbers. We also define recursively classes of functions computable by feed-forward neural networks with restrictions on the weights of neurons:
\begin{align}
&\calH_j(b_1,b_2,\cdots,b_j):=\bigg\{N_{\sigma_j,w}(h_1,h_2,\cdots,h_m): m\geq 0,\nn\\
&\qquad \qquad h_i \in \calH_{j-1}(b_1,b_2,\cdots,b_{j-1}), w \in \bbR^m, \|w\|_1\leq b_j\bigg\} \bigcup \calH_{j-1}(b_1,b_2,\cdots,b_{j-1}),
\end{align} where $\|w\|_1$ denotes the $1$-norm of the vector $w$.

Clearly,
\begin{align}
\calH_{\infty}=\bigcup\bigg\{\calH_j(b_1,\cdots,b_j): b_1,\cdots,b_j <+\infty \bigg\}.
\end{align}

Denote by $\tilde{\calH}$ the class of measurable functions $\tilf: \calS\times \calY \to \bbR$, where $\calY$ is the alphabet of labels. $\tilde{\calH}$ is introduced for real machine learning applications where we need to work with a new Markov chain generated from both feature vectors and their labels $\{(X_n,Y_n)\}_{n=1}^{\infty}$ instead of the feature-based Markov chain $\{X_n\}_{n=1}^{\infty}$.  See Subsection \ref{sub:setting} for detailed discussions. For binary classification, $\tilde{\calH}:=\{\tilf: f \in \calH\}$, where $\tilf(x,y)=yf(x)$. Let $\varphi$ be a function such that $\varphi(x)\geq I_{(-\infty,0]}(x)$ for all $x \in \bbR$ and $\varphi$ satisfies the Lipschitz condition with constant $L(\varphi)$. Then, the following is a direct application of Theorem \ref{cor:cor1mod}.
\begin{theorem} \label{thm:deep1mod} For any $t\geq 0$ and for all $l\geq 1$, 
	\begin{align}
	&\bbP\bigg(\exists f\in \calH(b_1,b_2,\cdots,b_l): P\{\tilf\leq 0\}> \inf_{\delta \in (0,1]}\bigg[P_n\varphi\bigg(\frac{\tilf}{\delta}\bigg)+ \frac{2\sqrt{2\pi} L(\varphi)}{\delta} \prod_{j=1}^l (2L_j b_j+1)G_n(\calH)\nn\\
	& \qquad +\big(t+\sqrt{\log \log_2 2\delta^{-1}}\big) \sqrt{\frac{\tau_{\min}}{n}} +B_n\bigg]\bigg) \leq  \frac{\pi^2}{3}\exp(-2t^2)\label{pet2axmod},
	\end{align}	 where $B_n$ is defined in \eqref{defBtnmod}.
\end{theorem}
\begin{remark}
$P\{\tilf\leq 0\}$ represents the probability of mis-classification in the deep neural network.
\end{remark}
\begin{proof} Let 
$
	\calH_l':=\calH(b_1,b_2,\cdots,b_l).
$
	As the proof of \citep[Theorem 13]{Koltchinskii2002}, it holds that
	\begin{align}
	G_n(\calH_l')& \leq \prod_{j=1}^l (2L_j b_j+1)G_n(\calH) \label{betomod}.
	\end{align}
	Hence, \eqref{pet2axmod} is a direct application of Theorem \ref{cor:cor1mod} and \eqref{betomod}.
\end{proof}
Now, given a neural network $f\in \calN_{\infty}$, let
\begin{align}
l(f):=\min\{j\geq 1:f \in \calN_j\}.
\end{align}
For any number $k$ such that $1\leq k\leq l(f)$, let $V_k(f)$ be the set of all neurons of layer $k$ in the neural network which is represented by $f$. Denote by
\begin{align}
W_k(f):=\max_{u \in V_k(f)}\|w^{(u)}\|_1 \vee b_k, \qquad k=1,2,\cdots, l(f)
\end{align} where $w^{(u)}$ is the coefficient-vector associated with the neuron $u$ in this layer.
Define
\begin{align}
\Lambda(f):=\prod_{k=1}^{l(f)}(4L_kW_k(f)+1),\qquad
\Gamma_{\alpha}(f):=\sum_{k=1}^{l(f)}\sqrt{\frac{\alpha}{2}\log(2+\log_2 W_k(f))},
\end{align} where $\alpha>0$ is the number such that $\zeta(\alpha)<3/2$, $\zeta$ being the Riemann zeta-function:
$
\zeta(\alpha):=\sum_{k=1}^{\infty}k^{-\alpha}.
$
Then, by using Theorem \ref{thm:deep1mod} with $b_k \to \infty$ and the same arguments as \citep[Proof of Theorem 14]{Koltchinskii2002}, we obtain the following result. See Section 1.3 in the supplement material for more detailed derivations.
\begin{theorem} \label{thm:deep2mod} For any $t\geq 0$ and for all $l\geq 1$, 
	\begin{align}
	&\bbP\bigg(\exists f\in \calH_{\infty}: P\{\tilf\leq 0\}> \inf_{\delta \in (0,1]}\bigg[P_n\varphi\bigg(\frac{\tilf}{\delta}\bigg) + \frac{2\sqrt{2\pi} L(\varphi)}{\delta} \Lambda(f) G_n(\calH)+\frac{2}{\sqrt{n}}\nn\\ 
	& \quad +\bigg(t+\Gamma_{\alpha}(f)+\sqrt{\log\log_2 2\delta^{-1}}\bigg) \sqrt{\frac{\tau_{\min}}{n}}+ B_n\bigg]\bigg)\leq \frac{\pi^2}{3}(3-2\zeta(\alpha))^{-1}\exp\big(-2t^2\big)  \label{pet2axbmod},
	\end{align}	 where $B_n$ is defined in \eqref{defBtnmod}.
\end{theorem}
\subsection{Generalization Error Bounds on Bayesian Deep Learning}
For Bayesian machine learning and deep learning, $\calF:=\big\{f: \calS \times \calW \to \bbR \big\}$, where $\calS$ is the state space of the Markov chain and $\calW$ is the domain of (random) coefficients. We assume that $\calS$ and $\calW$ are Polish spaces on $\bbR$, which include both discrete sets and $\bbR$. For example, in binary classification, $f(X,W)=\sgn(W^T X+b)$ where the feature $X$ and the coefficient $W$ are random vectors with specific prior distributions. In practice, the distribution of $W$ is known which depends on our design method, and the distribution of $X$ is unknown. For example, $W$ is assumed to be Gaussian in Bayesian deep neural networks \citep{Wilson20}.   

Since all the bounds on Subsections \ref{secsub1mod} and \ref{sec2:mod} hold for any function $f$ in $\calF$ at each fixed vector $W=w$, hence, they can be directly applied to Bayesian settings where $W$ is random. However, these bounds are not expected to be tight enough since we don't use the prior distribution of $W$ when deriving them. In the following, we use another approach to derive new (and tighter) bounds for Bayesian deep learning and machine learning from all the bounds in Subsections \ref{secsub1mod} and \ref{sec2:mod}. For illustration purposes, we only derive a new bound. Other bounds can be derived in a similar fashion. We assume that $W_1,W_2,\cdots,W_n$ are i.i.d. random variables as in \citep{Wilson20}.

Let
\begin{align}
\tilP_n:=\frac{1}{n}\sum_{i=1}^n \delta_{X_i,W_i},
\end{align}
and define a new probability measure $\tilP$ on $\calS\times \calW$ such that
\begin{align}
\tilP(A):=\int_A \tilde{\pi}(x,w) dx dw,
\end{align} for all (Borel) set $A$ on $\calS \times \calW$. Here, $\tilde{\pi}$ is the stationary distribution of the irreducible Markov process $\{(X_n,W_n)\}_{n=1}^{\infty}$ with stochastic matrix
$
\tilQ:=\{Q(x,w)P_W(w)\}_{x \in \calS, w \in \calW}.
$
In addition, define two new (averaging) linear functionals: 
%\begin{align}
%\hatP_n&:=\frac{1}{n}\sum_{i=1}^n \int_{\calW}  \delta_{X_i,w}dP_W(w),
%\end{align} 
%and
%\begin{align}
%\hatP:= \int_{\calS} \int_{\calW} \delta_{x,w} \tilde{\pi}(x,w) dP_W(w) dP(x),
%\end{align}
%such that
\begin{align}
\hatP_n(f)=\frac{1}{n}\sum_{i=1}^n \int_{\calW}  f(X_i,w) dP_W(w) \label{pracmod},
\end{align}
and
\begin{align}
\hatP(f):=\int_{\calS} \int_{\calW} f(x,w) \tilde{\pi}(x,w) dP_W(w) dP(x) .
\end{align}
In practice, the prior distribution of $W$ is known, so we can estimate $\hatP_n(f)$ based on the training set $\{(X_1,Y_1),(X_2,Y_2),\cdots,(X_n,Y_n)\}$, which is a Markov chain on $\calX \times \calY$ (cf.~ Section  \ref{sub:setting}). The following result can be proved.
\begin{theorem} \label{thm:mainexit} Let $\varphi$ be a sequence of function such that $\varphi(x)\geq I_{(-\infty,0]}(x)$. For any $t>0$,
	\begin{align}
	&\bbP\bigg(\exists f\in \calF: \hatP\{f\leq 0\}> \inf_{\delta\in (0,1]}\bigg[\hatP_n\varphi\bigg(\frac{f}{\delta}\bigg)+ \frac{8 L(\varphi)}{\delta} R_n(\calF)\nn\\  &\qquad +\big(t+\sqrt{\log\log_2 2\delta^{-1}}\big)\sqrt{\frac{\tau_{\min}}{n}}+ B_n\bigg]\bigg)\leq \frac{\pi^2}{3}\exp(-2t^2)\label{pet1modexit}
	\end{align}  
	and
	\begin{align}
	&\bbP\bigg(\exists f\in \calF: \hatP\{f\leq 0\}> \inf_{\delta\in (0,1]}\bigg[\hatP_n\varphi\bigg(\frac{f}{\delta}\bigg) + \frac{2L(\varphi)\sqrt{2\pi}}{\delta}  G_n(\calF)+\frac{2}{\sqrt{n}}\nn\\
	& \qquad +\big(t+\sqrt{\log\log_2 2\delta^{-1}}\big)\sqrt{\frac{\tau_{\min}}{n}}+ B_n\bigg]\bigg) \leq \frac{\pi^2}{3}\exp\big(-2t^2\big) \label{pet2modexit}.
	\end{align}
\end{theorem} 
\begin{proof} Let $W_1,W_2,\cdots,W_n$ be an $n$ samples of $W \sim P_W$ on $\calW$ (or samples of some set of random coefficients). For simplicity, we assume that $\{W_n\}_{n=1}^{\infty}$ is an i.i.d. sequence. Then, it is obvious that $\{(X_n,W_n)\}_{n=1}^{\infty}$ forms a Markov chain with probability transition probability
	\begin{align}
	\tilQ(x_n,w_n;x_{n+1},w_{n+1})&=\bbP(X_{n+1}=x_{n+1},W_{n+1}=w_{n+1}|X_n=x_n,W_n=w_n)\\
	&=Q(x_n,x_{n+1})P_W(w_{n+1}).
	\end{align}
	From Theorem \ref{cor:cor1mod}, it holds that
	\begin{align}
	&\bbP\bigg(\exists f\in \calF: \tilP\{f\leq 0\}> \inf_{\delta\in (0,1]}\bigg[\tilP_n\varphi\bigg(\frac{f}{\delta}\bigg)+ \frac{8 L(\varphi)}{\delta} R_n(\calF)\nn\\
	&\qquad +\big(t+\sqrt{\log\log_2 2\delta^{-1}}\big)\sqrt{\frac{\tau_{\min}}{n}} + B_n\bigg]\bigg) \leq \frac{\pi^2}{3}\exp(-2t^2). \end{align} 
	This means that with probability at least $1-\frac{\pi^2}{3}\exp(-2t^2)$, it holds that
	\begin{align}
	\tilP\{f\leq 0\}\leq  \inf_{\delta\in (0,1]}\bigg[\tilP_n\varphi\bigg(\frac{f}{\delta}\bigg)+ \frac{8 L(\varphi)}{\delta} R_n(\calF) +\big(t+\sqrt{\log\log_2 2\delta^{-1}}\big)\sqrt{\frac{\tau_{\min}}{n}}+B_n\bigg] \label{ubec1mod}.
	\end{align}
	From \eqref{ubec1mod}, it holds that with probability at least $1-\frac{\pi^2}{3}\exp(-2t^2)$,
	\begin{align}
	\tilP\{f\leq 0\}&\leq  \inf_{\delta\in (0,1]}\bigg[\frac{1}{n}\sum_{i=1}^n \varphi\bigg(\frac{f(X_i,W_i)}{\delta}\bigg)+ \frac{8 L(\varphi)}{\delta} R_n(\calF)\nn\\
	&\qquad  +\big(t+\sqrt{\log\log_2 2\delta^{-1}}\big)\sqrt{\frac{\tau_{\min}}{n}}+ B_n\bigg] \label{ubec2mod}.
	\end{align}
	From \eqref{ubec2mod}, with probability at least $1-\frac{\pi^2}{3}\exp(-2t^2)$, it holds that
	\begin{align}
	&\bbE_W\big[\tilP\{f\leq 0\}\big]\leq  \bbE_W\bigg[ \inf_{\delta\in (0,1]}\bigg[\frac{1}{n}\sum_{i=1}^n \varphi\bigg(\frac{f(X_i,W_i)}{\delta}\bigg)+ \frac{8 L(\varphi)}{\delta} R_n(\calF)\nn\\
	&\qquad  +\big(t+\sqrt{\log\log_2 2\delta^{-1}}\big)\sqrt{\frac{\tau_{\min}}{n}}+ B_n\bigg]\\
	%&\quad \leq  \inf_{\delta\in (0,1]}\bbE_W\bigg[\bigg[\frac{1}{n}\sum_{i=1}^n \varphi\bigg(\frac{f(X_i,W_i)}{\delta}\bigg)+ \frac{8 L(\varphi)}{\delta} R_n(\calF)\nn\\
	%&\qquad  +\big(t+\sqrt{\log\log_2 2\delta^{-1}}\big)\sqrt{\frac{\tau_{\min}}{n}}+ B_n\bigg]\\
	& \leq \inf_{\delta\in (0,1]}\bigg[\frac{1}{n}\sum_{i=1}^n \bbE_W\bigg[\varphi\bigg(\frac{f(X_i,W)}{\delta}\bigg)\bigg]+ \frac{8 L(\varphi)}{\delta} R_n(\calF) \nn\\
	&\qquad  +\big(t+\sqrt{\log\log_2 2\delta^{-1}}\big)\sqrt{\frac{\tau_{\min}}{n}}+ B_n\bigg] \label{estimod}.
	\end{align}
	From \eqref{estimod}, we obtain \eqref{pet1modexit}. Similarly, we can achieve \eqref{pet2modexit}.
\end{proof}
\section{Extension to High-Order Markov Chains}  \label{sec:ext}
In this subsection, we extend our results in previous sections to $m$-order homogeneous Markov chain. The main idea is to convert $m$-order homogeneous Markov chains to $1$-order homogeneous Markov chain and use our results in previous sections to bound the generalization error. We start with the following simple example. 
\begin{example} \label{mmarkov} [$m$-order moving average process without noise] 
	Consider the following $m$-order Markov chain
	\begin{align}
	X_k=\sum_{i=1}^m a_i X_{k-i},\qquad k\in \bbZ_+ \label{QG1}. 
	\end{align}
	Let $Y_k:= [X_{k+m-1},X_{k+m-2},\cdots, X_k]^T$. Then, from \eqref{QG1}, we obtain
	$
	Y_{k+1}=\bG Y_k, \enspace \forall k \in \bbZ_+
	$
	where
	\begin{align}
	\bG:=\begin{pmatrix} a_1&a_2&\cdots&a_{m-1}&a_m\\1&0&\cdots&0&0\\0&1&\cdots&0&0\\\vdots&\vdots&\ddots&\vdots&\vdots\\0&0&\cdots&1&0  \end{pmatrix}.
	\end{align} It is clear that $\{\bY_n\}_{n=1}^{\infty}$ is an order-$1$ Markov chain. Hence, instead of directly working with the $m$-order Markov chain $\{\bX_n\}_{n=1}^{\infty}$, we can find an upper bound for the Markov chain $\{Y_n\}_{n=1}^{\infty}$.
	
	To derive generalization error bounds for the Markov chain $\{Y_n\}_{n=1}^{\infty}$, we can use the following arguments. For all $f \in \calF$ and $(x_k,x_{k+1},\cdots,x_{k+m-1})$, by setting $\tilf(x_k,x_{k+1},\cdots,x_{k+m-1})=f(x_k)$ where $\tilf: \calS^m \to \bbR$, we obtain
	\begin{align}
	\frac{1}{n}\sum_{i=1}^n \bone\{f(X_i)\leq 0\}&=\frac{1}{n}\sum_{i=1}^n \bone\{\tilf(Y_i)\leq 0\} \label{eq244}. 
	\end{align}
	Hence, by applying all the results for $1$-order Markov chain $\{Y_n\}_{n=1}^{\infty}$, we obtain corresponding upper bounds for the sequence of $m$-order Markov chain $\{X_n\}_{n=1}^{\infty}$.
\end{example}
This approach can be extended to more general Markov chain. See Section 4 in the Supplementary Material for details.
\section{Conclusions}
In this paper, we derive upper bounds on generalization errors for machine learning and deep neural networks based on a new assumption that the dataset has Markov or hidden Markov structure. We also propose a new method to convert all these bounds to Bayesian deep learning and machine learning. Extension to $m$-order Markov chains and a mixture of Markov chains are also given. An interesting future research topic is to develop some new algorithms to evaluate performance of these bounds on real Markov datasets. 
\clearpage 
%\end{proof}
%\section{Proof of Lemma \ref{basiscex}}\label{basiscproofex}
%\bibliographystyle{unsrt}
%\section*{Software and Data}

%If a paper is accepted, we strongly encourage the publication of software and data with the
%camera-ready version of the paper whenever appropriate. This can be
%done by including a URL in the camera-ready copy. However, %\textbf{do not}
%include URLs that reveal your institution or identity in your
%submission for review. Instead, provide an anonymous URL or upload
%the material as ``Supplementary Material'' into the CMT reviewing
%system. Note that reviewers are not required to look at this material
%when writing their review.

%\section*{Acknowledgements}

%\textbf{Do not} include acknowledgements in the initial version of
%the paper submitted for blind review.

%If a paper is accepted, the final camera-ready version can (and
%probably should) include acknowledgements. In this case, please
%place such acknowledgements in an unnumbered section at the
%end of the paper. Typically, this will include thanks to reviewers
%who gave useful comments, to colleagues who contributed to the ideas,
%and to funding agencies and corporate sponsors that provided financial
%support.

\bibliographystyle{unsrtnat}
\bibliography{isitbib}
\clearpage
%\begin{titlepage}
%	\vspace*{\stretch{1.0}}
%	\begin{center}
%		\Large\textbf{Supplement Materials}\\
%		%\large\textit{A. Thor}
%	\end{center}
%	\vspace*{\stretch{2.0}}
%\end{titlepage}
%\Large\textbf{Supplement Materials}\\
\section{Supplemental Materials} \label{sec:main}
\subsection{Probabilistic Bounds for General Function Classes} \label{secsub1}
In this section, we develop probabilistic bounds for general function classes in terms of Gaussian and Rademacher complexities. 
First, we prove the following key lemma, which is an extension of the symmetrization inequality for i.i.d. sequences (for example, \citep{Vaartbook1996}) to a new version for Markov sequences $\{X_n\}_{n=1}^{\infty}$ with the stationary distribution $\pi$ and the initial distribution $\nu \in \calM_2$:
\begin{lemma}\label{lem:key} Let $\calF$ be a class of functions which are uniformly bounded by $M$. For all $n\in \bbZ^+$, define
	\begin{align}
	A_n&:=\sqrt{\frac{2M}{n(1-\lambda)}+ \frac{64 M^2}{n^2(1-\lambda)^2}\bigg\|\frac{dv}{d\pi}-1\bigg\|_2} \label{defAtn},\\
	\tilA_n&:=\frac{M}{2n}\bigg[\sqrt{2 \tau_{\min} n\log n}+\sqrt{n}+4\bigg]\label{deftilAtn}.
	\end{align}
	Then, the following holds:
	\begin{align}
	\frac{1}{2}\bbE\big[\|P_n^0\|_{\calF}\big]-\tilA_n \leq\bbE\big[\big\|P_n-P\big\|_{\calF}\big]  \leq 2 \bbE\big[\|P_n^0\|_{\calF}\big] +A_n \label{F1eq},
	\end{align}
	where
	\begin{align}
	\|P_n^0\|_{\calF}:=	\sup_{f\in \calF}\bigg|\frac{1}{n}\sum_{i=1}^n \eps_i f(X_i)\bigg|.
	\end{align}
\end{lemma}
\begin{proof}
	See Appendix \ref{prooflemkey}.	
\end{proof}
Next, we prove the following theorem.
\begin{theorem} \label{thm:main} Consider a countable family of Lipschitz function $\Phi=\{\varphi_k:k\geq 1\}$, where $\varphi_k: \bbR\to \bbR$ satisfies $\bone_{(-\infty,0]}(x)\leq \varphi_k(x)$ for all $k$. For each $\varphi \in \Phi$, denote by $L(\varphi)$ its Lipschitz constant. Define
	\begin{align}
	B_n:=\sqrt{\frac{2}{n(1-\lambda)}+ \frac{64 }{n^2(1-\lambda)^2}\bigg\|\frac{dv}{d\pi}-1\bigg\|_2} \label{defBtn}.
	\end{align}
	Then, for any $t>0$,
	\begin{align}
	&\bbP\bigg(\exists f\in \calF: P\{f\leq 0\}> \inf_{k>0}\bigg[P_n\varphi_k(f)+ 4 L(\varphi_k) R_n(\calF)+ \big(t+\sqrt{\log k}\big)\sqrt{\frac{\tau_{\min}}{n}}+ B_n\bigg]\bigg)\nn\\
	&\quad  \leq \frac{\pi^2}{3}\exp(-2t^2) \label{pet1}
	\end{align}  
	and
	\begin{align}
	&\bbP\bigg(\exists f\in \calF: P\{f\leq 0\}> \inf_{k>0}\bigg[P_n\varphi_k(f)+ \sqrt{2\pi} L(\varphi_k) G_n(\calF)\nn\\
	&\qquad \qquad \qquad +\frac{2}{\sqrt{n}} + \big(t+\sqrt{\log k}\big)\sqrt{\frac{\tau_{\min}}{n}}\bigg]\bigg)  \leq \frac{\pi^2}{3}\exp(-2t^2) \label{pet2}.
	\end{align}
\end{theorem}
\begin{proof}
	See Appendix \ref{proofthm:main}.	
\end{proof}
\begin{theorem} \label{cor:cor1} Let $\varphi$ is a non-increasing function such that $\varphi(x)\geq \bone_{(-\infty,0]}(x)$ for all $x \in \bbR$. For any $t>0$,
	\begin{align}
	&\bbP\bigg(\exists f\in \calF: P\{f\leq 0\}> \inf_{\delta \in (0,1]}\bigg[P_n\varphi\bigg(\frac{f}{\delta}\bigg)+  \frac{8L(\varphi)}{\delta} R_n(\calF)+ \big(t+\sqrt{\log \log_2 2\delta^{-1}}\big)\sqrt{\frac{\tau_{\min}}{n}}+ B_n\bigg]\bigg)\nn\\
	&\qquad \qquad  \leq \frac{\pi^2}{3}\exp(-2t^2) \label{pet1a}
	\end{align} 
	and
	\begin{align}
	&\bbP\bigg(\exists f\in \calF: P\{f\leq 0\}> \inf_{\delta \in (0,1]}\bigg[P_n\varphi\bigg(\frac{f}{\delta}\bigg)+\frac{2 L(\varphi)\sqrt{2\pi}}{\delta}G_n(\calF)\nn\\
	&\qquad \qquad +  \frac{2}{\sqrt{n}}+\big(t+\sqrt{\log \log_2 2\delta^{-1}}\big)\sqrt{\frac{\tau_{\min}}{n}}+ B_n\bigg]\bigg)\leq \frac{\pi^2}{3}\exp(-2t^2) \label{pet2a},
	\end{align} where $B_n$ is defined in \eqref{defBtn}.	
\end{theorem}
\begin{proof}
	See Appendix \ref{proofcor:cor1}.
\end{proof}
In the next statements, we use Rademacher complexities, but Gaussian complexities can be used similarly. Now, assume that $\varphi$ is a function from $\bbR$ to $\bbR$ such that $\varphi(x)\leq I_{(-\infty,0]}(x)$ for all $x \in \bbR$ and $\varphi$ satisfies the Lipschitz with constant $L(\varphi)$. Then, the following theorems can be proved by using similar arguments as the proof of Theorem \ref{cor:cor1}.
\begin{theorem} \label{cor:cor2} Let $\varphi$ is a nonincreasing function such that $\varphi(x)\leq \bone_{(-\infty,0]}(x)$ for all $x \in \bbR$. For any $t>0$,
	\begin{align}
	&\bbP\bigg(\exists f\in \calF: P\{f\leq 0\}< \sup_{\delta \in (0,1]}\bigg[P_n\varphi\bigg(\frac{f}{\delta}\bigg)+ \frac{8L(\varphi)}{\delta} R_n(\calF)\nn\\
	&\qquad \qquad  +\big(t+\sqrt{\log \log_2 2\delta^{-1}}\big)\sqrt{\frac{\tau_{\min}}{n}}+B_n \bigg]\bigg)  \leq \frac{\pi^2}{3}\exp(-2t^2) \label{pet1b}
	\end{align} 
	and
	\begin{align}
	&\bbP\bigg(\exists f\in \calF: P\{f\leq 0\}< \sup_{\delta \in (0,1]}\bigg[P_n\varphi\bigg(\frac{f}{\delta	}\bigg)+ \frac{2\sqrt{2\pi} L(\varphi)}{\delta} G_n(\calF) + \frac{2}{\sqrt{n}}\bigg]\bigg)\nn\\
	&\qquad \qquad  +\big(t+\sqrt{\log \log_2 2\delta^{-1}}\big)\sqrt{\frac{\tau_{\min}}{n}}+B_n \bigg]\bigg) \leq \frac{\pi^2}{3}\exp(-2t^2) \label{pet2b},
	\end{align}	where $B_n$ is defined in \eqref{defBtn}.
\end{theorem}
By combining Theorem \ref{cor:cor1} and Theorem \ref{cor:cor2}, we obtain the following result.
\begin{theorem} \label{thm:main2}  Let
	\begin{align}
	\Delta_n(\calF;\delta)&:=\frac{8}{\delta}R_n(\calF)+\sqrt{\frac{\tau_{\min} \log \log_2 2\delta^{-1}}{n}}+B_n.
	\end{align}
	Then, for all $t>0$,
	\begin{align}
	&\bbP\bigg(\exists f\in \calF: \big|P_n\{f\leq 0\}-P\{f\leq 0\}\big|>\inf_{\delta \in (0,1]}\bigg(P_n\{|f|\leq \delta\}+ \Delta_n(\calF;\delta)+ t\sqrt{\frac{\tau_{\min}}{n}}\bigg)\bigg)\nn\\
	&\qquad \qquad \leq \frac{2\pi^2}{3}\exp(-2t^2) \label{a11}
	\end{align}
	and
	\begin{align} 
	&\bbP\bigg(\exists f\in \calF: \big|P_n\{f\leq 0\}-P\{f\leq 0\}\big|>\inf_{\delta \in (0,1]}\bigg(P\{|f|\leq \delta\}+ \Delta_n(\calF;\delta)+ t\sqrt{\frac{\tau_{\min}}{n}}\bigg)\nn\\
	&\qquad \qquad \leq \frac{2\pi^2}{3}\exp(-2t^2) \label{a12}.
	\end{align} 
\end{theorem}
\begin{proof} Equation \eqref{a11} is drawn by setting $\varphi(x)=\bone\{x\leq 0\}+(1-x)\bone\{0 \leq x \leq 1\}$ in Theorem \ref{cor:cor1} and Theorem \ref{cor:cor2}. Equation \eqref{a12} is drawn by setting $\varphi(x)=\bone\{x\leq -1\}-x \bone\{-1\leq x\leq 0\}$ in these theorems.
\end{proof}
\subsection{Conditions on Random Entropies and $\gamma$-Margins}
As \citep{Koltchinskii2002}, given a metric space $(T,d)$, we denote by $H_d(T;\eps)$ the $\eps$-entropy of $T$ with respect to $d$, that is
\begin{align}
H_d(T;\eps):=\log N_d(T;\eps),
\end{align} where $N_d(T;\eps)$ is the minimal number of balls of radius $\eps$ covering $T$. Let $d_{P_n,2}$ denote the metric of the space $L_2(\calS;dP_n)$:
\begin{align}
d_{P_n,2}(f,g):=\big(P_n|f-g|^2\big)^{1/2}.
\end{align}

For each $\gamma \in (0,1]$, define
\begin{align}
\delta_n(\gamma;f):=\sup\bigg\{\delta \in (0,1): \delta^{\frac{\gamma}{2}}P(f\leq \delta)\leq n^{-\frac{1}{2}+\frac{\gamma}{4}}\bigg\} \label{defdeltan}
\end{align}
and
\begin{align}
\hat{\delta}_n(\gamma;f):=\sup\bigg\{\delta \in (0,1): \delta^{\frac{\gamma}{2}}P_n(f\leq \delta)\leq n^{-\frac{1}{2}+\frac{\gamma}{4}}\bigg\} \label{defhatdeltan}.
\end{align}
We call $\delta_n(\gamma;f)$ and $\hat{\delta}_n(\gamma;f)$, respectively, the \emph{$\gamma$-margin} and \emph{empirical $\gamma$-margin} of $f$.
\begin{theorem} \label{thm:main3} Suppose that for some $\alpha \in (0,2)$ and some constant $D>0$,
	\begin{align}
	H_{d_{P_n},2}(\calF;u)\leq D u^{-\alpha},\qquad u>0 \qquad a.s. \label{cond31},
	\end{align}
	Then, for any $\gamma\geq \frac{2\alpha}{2+\alpha}$, there exists some constants $\zeta,\upsilon>0$ such that when $n$ is large enough,
	\begin{align}
	\bbP\bigg[\forall f\in \calF: \zeta^{-1}\hat{\delta}_n(\gamma;f)\leq \delta_n(\gamma;f)\leq \zeta \hat{\delta}_n(\gamma;f)\bigg]\geq 1-2\upsilon \big(\log_2\log_2 n\big) \exp\big\{-n^{\frac{\gamma}{2}}/2\big\} \label{muta}.
	\end{align}
\end{theorem}
\begin{proof}
	See Appendix \ref{markb} for a detailed proof.
\end{proof}

\subsection{Convergence rates of empirical margin distributions}
First, we prove the following lemmas.
\begin{lemma} \label{lem13} For any class $\calF$ of bounded measurable functions from $\calS \to \bbR$, with probability at least $1-2\exp\big(-2t^2\big)$, the following holds:
	\begin{align}
	\sup_{f \in \calF} \sup_{y \in \bbR} \big|P_n(f\leq y)-P(f\leq y)\big|\leq  \sqrt{B_n} +t\sqrt{\frac{\tau_{\min}}{n}}\label{ba1},
	\end{align} 
\end{lemma} where $B_n$ is defined in \eqref{defBtn}.
\begin{remark} By setting $t=\sqrt{2\log n}$, \eqref{ba1} shows that $\sup_{f \in \calF} \sup_{y \in \bbR} \big|P_n(f\leq y)-P(f\leq y)\big|\to 0$ as $n\to \infty$.
\end{remark}
\begin{proof}
	See Appendix \ref{proofoflem13}.	
\end{proof}
Now, for each $f \in \calF$, define
\begin{align}
F_f(y):=P\{f\leq y\},\qquad F_{n,f}:=P_n\{f\leq y\}, \qquad y \in \bbR. 
\end{align}
Let $L$ denote the L\'{e}vy distance between the distributions in $\bbR$:
\begin{align}
L(F,G):=\inf\{\delta>0: F(t)\leq G(t+\delta)+\delta, \qquad G(t)\leq F(t+\delta)+\delta,\qquad \forall t \in \bbR\}.
\end{align}
\begin{lemma} \label{lem:lem14} Let $M>0$ and $\calF$ be a class of measurable functions from $\calS$ into $[-M,M]$. Let $\varphi$ be equal to $1$ for $x\leq 0$, $0$ for $x\geq 1$ and linear between them. Define
	\begin{align}
	\tilde{\calG}_{\varphi}:=\bigg\{\varphi \circ \bigg(\frac{f-y}{\delta}\bigg)-1: f \in \calF,\quad y \in [-M,M]\bigg\}
	\end{align}	for some $\delta>0$. Recall the definition of $B_n$ in \eqref{defBtn}. Then, for all $t>0$ and $\delta>0$, the following holds:
	\begin{align}
	\bbP\bigg\{\sup_{f \in \calF} L(F_f,F_{f,n})\geq \delta+ \bbE\big[\|P_n^0\|_{\tilde{\calG}_{\varphi}} \big]+ B_n+ t\sqrt{\frac{\tau_{\min}}{n}}\bigg\}\leq 2\exp(-2t^2) \label{etopa}.
	\end{align}	
	Especially, for all $t>0$, we have
	\begin{align}
	\bbP\bigg\{\sup_{f \in \calF} L(F_f,F_{f,n})\geq 4\sqrt{\bbE[\|P_n^0\|_{\calF}]+M/\sqrt{n}} +B_n+ t\sqrt{\frac{\tau_{\min}}{n}}\bigg\} \leq 2\exp(-2t^2).
	\end{align}
\end{lemma}
\begin{proof}
	See Appendix \ref{prooflem:lem14}.	
\end{proof}

In what follows, for a function $f$ from $\calS$ into $\bbR$ and $M>0$, we denote by $f_M$ the function that is equal to $f$ if $|f|\leq M$, is equal to $M$ if $f>M$ and is equal to $-M$ if $f<-M$. We set
\begin{align}
\calF_M:=\big\{f_M: f\in \calF\}.
\end{align}
As always, a function $\calF$ from $\calS$ into $[0,\infty)$ is called an envelope of $\calF$ iff $|f(x)|\leq F(x)$ for all $f \in \calF$ and all $x \in \calS$.

We write $\calF \in \rm{GC}(P)$ iff $\calF$ is a Glivenko-Cantelli class with respect to $P$ (i.e., $\|P_n-P\|_{\calF}\to 0$ as $n\to \infty$ a.s.). We write $\calF \in \rm{BCLT}(P)$ and say that $\calF$ satisfies the Bounded Central Limit Theorem for $P$ iff
\begin{align}
\bbE\big[\big\|P_n-P\big\|_{\calF}\big]=O(n^{-1/2}).
\end{align}
Based on Lemma and Lemma \ref{lem:lem14}, we prove the following theorems.
\begin{theorem} \label{thm1} Suppose that
	\begin{align}
	\sup_{f \in \calF} P\{|f|\geq M\} \to 0 \qquad \mbox{as} \qquad M\to \infty \label{cond41}.
	\end{align}
	Then, the following two statements are equivalent:
	\begin{itemize}
		\item (i)  $\calF_M \in \rm{GC}(P)$ for all $M>0$\\
		and
		\item (ii) $\sup_{f\in \calF} L(F_{n,f},F_f)\to 0$ \quad a.s. \quad $n\to \infty$.
	\end{itemize}
\end{theorem}
\begin{proof}
	See Appendix \ref{proofthm1}.	
\end{proof}
Next, the following theorems hold.
\begin{theorem}\citep[Theorem 7]{Koltchinskii2002} The following two statements are equivalent:
	\begin{itemize}
		\item (i) $\calF\in \rm{GC}(P)$ for all $M>0$\\
		and
		\item (ii) there exists a $P$-integrable envelope for the class $\calF^{(c)}=\{f-Pf: f \in \calF\}$ and 
		$\sup_{f \in \calF}  L(F_{n,f},F_f) \to 0$ \qquad $n\to \infty$ and
		\begin{align}
		\sup_{f\in \calF} L(F_{n,f},F_f)\to 0 \qquad a.s. \qquad n\to \infty.
		\end{align}
	\end{itemize}
\end{theorem}
Now, we prove the following theorem.
\begin{theorem}\label{thm4} Suppose that the class $\calF$ is uniformly bounded. If $\calF \in \rm{BCLT}(P)$, then
	\begin{align}
	\sup_{f\in \calF} L(F_{n,f},F_f)=O_P \bigg(\bigg(\frac{\log n}{n}\bigg)^{1/4}\bigg) \qquad n\to \infty.
	\end{align}
	Moreover, for some $\alpha \in (0,2)$ and for some $D>0$
	\begin{align}
	H_{d_{P_n,2}}(\calF;u)\leq D  u^{-\alpha} \log n, \qquad u>0, \qquad a.s. \label{cond42},
	\end{align}
	then
	\begin{align}
	\sup_{f\in \calF} L(F_{n,f},F_f)=O\big(  n^{-\frac{1}{2+\alpha}} \log n\big)\qquad n\to \infty, \qquad a.s.,
	\end{align} 
\end{theorem}
\begin{proof}
	Appendix \ref{thm4proof}.
\end{proof}
\subsection{Bounding the generalization error of convex combinations of classifiers}
We start with an application of the inequalities in Subsection \ref{secsub1} to bounding the generalization error in general classification problems. Assume that the labels take values in a finite set $\calY$ with $|\calY|=K$. Consider a class $\tilde{\calF}$ of functions from $\calS\times \calY$ into $\bbR$. A function $f \in \tilde{\calF}$  predicts a label $y\in \calY$ for an example $x \in \calS$ iff
\begin{align}
f(x,y)>\max_{y'\neq y} f(x,y').
\end{align} In practice, $f(x,y)$ can be set equal to $P(y|x)$, so $\calF$ can be assumed to be uniformly bounded. The margin of a labelled example $(x,y)$ is defined as
\begin{align}
m_{f,y}(x):=f(x,y)-\max_{y'\neq y} f(x,y'),
\end{align} so $f$ mis-classifies the label example $(x,y)$ if and only if $m_{f,y}\leq 0$. Let
\begin{align}
\calF:=\big\{f(\cdot,y):y \in \calY, f \in \tilde{\calF}\big\}.
\end{align} 

Then, we can show the following theorem.
\begin{theorem} For all $t>0$, it holds that
	\begin{align}
	&\bbP\bigg(\exists f\in \calF: P\{m_{f,y} \leq 0\}> \inf_{\delta \in (0,1]}\bigg[P_n\{m_{f,y}\leq \delta\}+  \frac{8}{\delta} (2K-1)R_n(\calF)\nn\\
	&\qquad \qquad +\big(t+\sqrt{\log\log_2 2\delta^{-1}}\big)\sqrt{\frac{\tau_{\min}}{n}}+ B_n\bigg]\bigg)\leq \frac{\pi^2}{3}\exp(-2t^2) \label{pet1abc}, 
	\end{align} where $B_n$ is defined in \eqref{defBtn}. Here,
	\begin{align}
	P\{m_{f,y} \leq 0\}:=\sum_{x \in \calS} \pi(x)\bone\{m_{f,y}(x) \leq 0\},
	\end{align}
	and
	\begin{align}
	P_n\{m_{f,y} \leq \delta \}=\frac{1}{n}\sum_{k=1}^n \bone\{m_{f,y}(X_k)\leq \delta \}
	\end{align}
	is the empirical distribution of the Markov process $\{m_{f,y}(X_n)\}_{n=1}^{\infty}$ given $f$.
\end{theorem}
\begin{proof}
	First, we need to bound the Rademacher's complexity for the class of functions $\big\{m_{f,y}: f \in \tilde{\calF}\big\}$. Observe that
	\begin{align}
	\bbE\bigg[\sup_{f \in \tilde{\calF}}\bigg|n^{-1}\sum_{j=1}^n \eps_j m_{f,y}(X_j)\bigg|\bigg].
	\end{align}
	By \citep[Proof of Theorem 11]{Koltchinskii2002}, we have
	\begin{align}
	\bbE\bigg[\sup_{f \in \tilde{\calF}}\bigg|n^{-1}\sum_{j=1}^n \eps_j m_{f,y}(X_j)\bigg|\bigg]\leq (2K-1)R_n(\calF),
	\end{align} where $R_n(\calF)$ is \emph{the Rademacher complexity function} of the class $\calF$. Now, assume that this class of function is uniformly bounded as in practice. Hence, by Theorem \ref{cor:cor1} for $\varphi$ that is equal to $1$ on $(-\infty,0]$, is equal to $0$ on $[1,+\infty)$ and is linear in between, we obtain \eqref{pet1abc}.
\end{proof}

In addition, by using the fact that $(X_1,Y_1)-(X_2,Y_2)\cdots -(X_n,Y_n)$ forms a Markov chain with stationary distribution $\tilde{\pi}$ (see the discussion on Section 2.2 in the main document), by applying Theorem \ref{cor:cor1}, we obtain the following result:
\begin{theorem} \label{appthm:1} Let $\varphi$ is a nonincreasing function such that $\varphi(x)\geq \bone_{(-\infty,0]}(x)$ for all $x \in \bbR$. For any $t>0$,
	\begin{align}
	&\bbP\bigg(\exists f\in \tilde{\calF}: P\{\tilf\leq 0\}> \inf_{\delta \in (0,1]}\bigg[P_n\varphi\bigg(\frac{\tilf}{\delta}\bigg)+  \frac{8L(\varphi)}{\delta} R_n(\calH)\nn\\
	&\qquad \qquad +\big(t+\sqrt{\log\log_2 2\delta^{-1}}\big)\sqrt{\frac{\tau_{\min}}{n}}+ B_n\bigg]\bigg) \leq \frac{\pi^2}{3}\exp(-2t^2) \label{pet1ac},
	\end{align}  where $B_n$ is defined in \eqref{defBtn}.
\end{theorem}
As in \citep{Koltchinskii2002}, in the voting methods of combining classifiers, a classifier produced at each iteration is a convex combination $\tilf$ of simple base classifiers from the class $\calH$. In addition, the Rademacher complexity can be bounded above by
\begin{align}
R_n(\calH)\leq C \sqrt{\frac{V(\calH)}{n}}
\end{align} for some constant $C>0$, where $V(\calH)$ is the VC-dimesion of $\calH$. Let $\varphi$ be equal to $1$ on $(-\infty,0]$, is equal to $0$ on $[1,+\infty)$ and is linear in between. By setting $t_{\alpha}=\sqrt{\frac{1}{2}\log \frac{\pi^2}{3\alpha}}$, from Theorem \ref{appthm:1}, with probability at least $1-\alpha$, it holds that 
\begin{align}
P\{\tilf\leq 0\}&\leq \inf_{\delta \in (0,1]}\bigg[P_n\{\tilf\leq \delta\}+  \frac{8 C}{\delta}  \sqrt{\frac{V(\calH)}{n}}+ \big(t_{\alpha}+\sqrt{\log\log_2 2\delta^{-1}}\big)\sqrt{\frac{\tau_{\min}}{n}}+ B_n\bigg] \label{PACbnd},
\end{align} which extends the result of Bartlett et al. \citep{Bartlett1998} to Markov dataset (PAC-bound).
\subsection{Bounding the generalization error in neural network learning}
In this section, we consider the same example as \citep[Section 6]{Koltchinskii2002}. However, we assume that feature vectors in this dataset is generated by a Markov chain instead of an i.i.d. process. Let $\calH$ be a class of measurable functions from $\calS \to \bbR$ (base functions). Let $\tilde{\calH}$ be the set of function $\tilf: \calS \times \calY \to \bbR$. The introduction of $\tilde{\calH}$ is to deal with the new Markov chain $\{(X_n,Y_n)\}_{n=1}^{\infty}$ which is generated by both feature vectors and their labels instead of the feature-based Markov chain $\{X_n\}_{n=1}^{\infty}$. 

Consider a feed-forward neural network where the set $V$ of all the neurons is divided into layers
\begin{align}
V=\{v_i\} \cup \bigcup_{j=0}^l V_j
\end{align}
where $V_l=\{v_o\}$. The neurons $v_i$ and $v_o$ are called the input and the output neurons, respectively. To define the network, we will assign the labels to the neurons in the following way. Each of the base neurons is labelled by a function from the base class $\calH$. Each neuron of the $j$-th layer $V_j$, where $j\geq 1$, is labelled by a vector $w:=(w_1,w_2,\cdots,w_m) \in \bbR^m$, where $m$ is the number of inputs of the neuron, i.e. $m=|\calH|+1$. Here, $w$ will be called the vector of weights of the neuron. 

Given a Borel function $\sigma$ from $\bbR$ into $[-1,1]$ (for example, sigmoid function) and a vector $w:=(w_1,w_2,\cdots,w_m)\in \bbR^m$, let
\begin{align}
N_{\sigma,w}:\bbR^m \to \bbR, \qquad N_{\sigma, w}(u_1,u_2,\cdots, u_m):=\sigma\bigg(\sum_{j=1}^m w_j u_j\bigg).
\end{align}
For $w \in \bbR^m$,
\begin{align}
\|w\|_1:=\sum_{i=1}^m |w_i|.
\end{align}
Let $\sigma_j: j\geq 1$ be functions from $\bbR$ into $[-1,1]$, satisfying the Lipschitz conditions
\begin{align}
\big|\sigma_j(u)-\sigma_j(v)\big|\leq L_j|u-v|, \qquad u, v \in \bbR. 
\end{align}
The network works the following way. The input neuron inputs an instance $x \in \calS$. A base neuron computes the value of the base function on this instance and outputs the value through its output edges. A neuron in the $j$-th layer ($j\geq 1$) computes and outputs through its output edges the value $N_{\sigma_j, w}(u_1,u_2,\cdots,u_m)$ (where $u_1,u_2,\cdots,u_m$ are the values of the inputs of the neuron). The network outputs the value $f(x)$ (of a function $f$ it computes) through the output edge. 

We denote by $\calN_l$ the set of such networks. We call $\calN_l$ the class of feed-forward neural networks with base $\calH$ and $l$ layers of neurons (and with sigmoid $\sigma_j$). Let $\calN_{\infty}:=\bigcup_{j=1}^{\infty} \calN_j$. Define $\calH_0:=\calH$, and then recursively
\begin{align}
\calH_j:=\bigg\{N_{\sigma_j,w}(h_1,h_2,\cdots, h_m): m\geq 0, h_i \in \calH_{j-1}, w \in \bbR^m \bigg\} \cup \calH_{j-1}.
\end{align}
Denote $\calH_{\infty}:=\bigcup_{j=1}^{\infty}\calH_j$. Clearly, $\calH_{\infty}$ includes all the functions computable by feed-forward neural networks with base $\calH$. 

Let $\{b_j\}$ be a sequence of positive numbers. We also define recursively classes of functions computable by feed-forward neural networks with restrictions on the weights of neurons:
\begin{align}
&\calH_j(b_1,b_2,\cdots,b_j)\\
&\qquad :=\bigg\{N_{\sigma_j,w}(h_1,h_2,\cdots,h_m): n\geq 0,h_i \in \calH_{j-1}(b_1,b_2,\cdots,b_{j-1}), w \in \bbR^m, \|w\|_1\leq b_j\bigg\}\nn\\
&\qquad \qquad \qquad  \bigcup \calH_{j-1}(b_1,b_2,\cdots,b_{j-1}). 
\end{align}
Clearly,
\begin{align}
\calH_{\infty}=\bigcup\bigg\{\calH_j(b_1,b_2,\cdots,b_j): b_1,b_2,\cdots,b_j <+\infty \bigg\}.
\end{align}
As in the previous section, let $\varphi$ be a function such that $\varphi(x)\geq I_{(-\infty,0]}$ for all $x \in \bbR$ and $\varphi$ satisfies the Lipschitz condition with constant $L(\varphi)$. Then, the following is a direct application of Theorem \ref{cor:cor1}.
\begin{theorem} \label{thm:deep1} For any $t\geq 0$ and for all $l\geq 1$, 
	\begin{align}
	&\bbP\bigg(\exists f\in \calH(b_1,b_2,\cdots,b_l): P\{\tilf\leq 0\}\nn\\
	&\qquad > \inf_{\delta \in (0,1]}\bigg[P_n\varphi\bigg(\frac{\tilf}{\delta}\bigg)+ \frac{2\sqrt{2\pi} L(\varphi)}{\delta} \prod_{j=1}^l (2L_j b_j+1)G_n(\calH)\nn\\
	&\qquad \qquad +\frac{2}{\sqrt{n}}+\big(t+\sqrt{\log\log_2 2\delta^{-1}}\big)\sqrt{\frac{\tau_{\min}}{n}}+B_n\bigg]\bigg)\leq  \frac{\pi^2}{3}\exp\big(-2t^2\big) \label{pet2ax},
	\end{align}	 where $B_n$ is defined in \eqref{defBtn}.
\end{theorem}
\begin{proof} Let 
	\begin{align}
	\calH_l':=\calH(b_1,b_2,\cdots,b_l).
	\end{align}
	As the proof of \citep[Theorem 13]{Koltchinskii2002}, it holds that
	\begin{align}
	G_n(\calH_l')&:=\bbE\bigg\|n^{-1}\sum_{i=1}^n g_i \delta_{X_i}\bigg\|_{\calH_l'} \leq \prod_{j=1}^l (2L_j b_j+1)G_n(\calH) \label{beto}.
	\end{align}
	Hence, \eqref{pet2ax} is a direct application of Theorem \ref{cor:cor1} and \eqref{beto}.
\end{proof}
Now, given a neural network $f\in \calN_{\infty}$, let
\begin{align}
l(f):=\min\{j\geq 1:f \in \calN_j\}.
\end{align}
For a number $k$, $1\leq k\leq l(f)$, let $V_k(f)$ denote the set of all neurons of layer $k$ in the graph representing $f$. Denote
\begin{align}
W_k(f):=\max_{u \in V_k(f)}\|w^{(u)}\|_1 \vee b_k, \qquad k=1,2,\cdots, l(f)
\end{align} where $w^{(u)}$ is the coefficient vector associated with the neuron $u$ of the layer $k$, and let
\begin{align}
\Lambda(f)&:=\prod_{k=1}^{l(f)}(4L_kW_k(f)+1),\\
\Gamma_{\alpha}(f)&:=\sum_{k=1}^{l(f)}\sqrt{\frac{\alpha}{2}\log(2+\log_2 W_k(f))},
\end{align} where $\alpha>0$ is the number such that $\zeta(\alpha)<3/2$, $\zeta$ being the Riemann zeta-function:
\begin{align}
\zeta(\alpha):=\sum_{k=1}^{\infty}k^{-\alpha}.
\end{align}
Then, by using the same arguments as \citep[Proof of Theorem 14]{Koltchinskii2002}, we obtain the following result.
\begin{theorem} \label{thm:deep2} For any $t\geq 0$ and for all $l\geq 1$, 
	\begin{align}
	&\bbP\bigg(\exists f\in \calH_{\infty}: P\{\tilf\leq 0\}> \inf_{\delta \in (0,1]}\bigg[P_n\varphi\bigg(\frac{\tilf}{\delta}\bigg)+ \frac{2\sqrt{2\pi} L(\varphi)}{\delta} \Lambda(f) G_n(\calH)\nn\\
	&\qquad +\frac{2}{\sqrt{n}}+\big(t+\Gamma_{\alpha}(f)+\sqrt{\log\log_2 2\delta^{-1}}\big)\sqrt{\frac{\tau_{\min}}{n}}+ B_n\bigg]\bigg)\nn\\
	&\qquad \qquad  
	\leq \frac{\pi^2}{3}\big(3-2\zeta(\alpha)\big)^{-1}\exp\big(-2t^2\big)  \label{pet2axb},
	\end{align}	 where $B_n$ is defined in \eqref{defBtn}.
\end{theorem}
\begin{proof} Let 
	\begin{align}
	\Delta_k:=\begin{cases}[2^{k-1},2^k) \quad \mbox{for} \quad k \in \bbZ, k\neq 0,1\\
	[1/2,2) \quad \mbox{for} k=1. 
	\end{cases}
	\end{align}
	Then, using the following partition:
	\begin{align}
	\{f \in \calH_{\infty}\}=\bigcup_{l=0}^{\infty}\bigcup_{k_1\in \bbZ\setminus \{0\}}\bigcup_{k_2\in \bbZ\setminus \{0\}}\cdots \bigcup_{k_l\in \bbZ\setminus \{0\}}  \bigg\{f \in \calH_{\infty}: l(f)=l, W_j(f)\in \Delta_{k_j}, \enspace \forall j \in [l] \bigg\}.
	\end{align}
	On each subset $\bigg\{f \in \calH_{\infty}: l(f)=l, W_j(f)\in \Delta_{k_j}, \enspace \forall j \in [l] \bigg\}$, we can lower bound $\Lambda(f)$ and $\Gamma_{\alpha}(f)$ by 
	\begin{align}
	\Lambda(f) &\geq \prod_{j=1}^l(2L_j 2^{k_j}+1),\\
	\Gamma_{\alpha}(f)&\geq \sum_{j=1}^l \sqrt{\frac{\alpha}{2}\log(|k_j|+1)}. 
	\end{align}
	By replacing $t$ by $t+\sum_{j=1}^l \sqrt{\frac{\alpha}{2}\log (k_j+1)}$ and using Theorem \ref{cor:cor1} to bound the probability of each event and then using the union bound, we can show that
	\begin{align}
	&\bbP\bigg(\exists f\in \calH_{\infty}: P\{\tilf\leq 0\}> \inf_{\delta \in (0,1]}\bigg[P_n\varphi\bigg(\frac{\tilf}{\delta}\bigg)+ \frac{2\sqrt{2\pi} L(\varphi)}{\delta} \Lambda(f) G_n(\calH)\nn\\
	&\qquad +\frac{2}{\sqrt{n}}+\big(t+\Gamma_{\alpha}(f)+\sqrt{\log\log_2 2\delta^{-1}}\big)\sqrt{\frac{\tau_{\min}}{n}}+ B_n\bigg]\bigg)\\
	&\leq \bbP\bigg(\exists f\in \calH_{\infty}: P\{\tilf\leq 0\}\nn\\
	&\qquad > \inf_{\delta \in (0,1]}\bigg[P_n\varphi\bigg(\frac{\tilf}{\delta}\bigg)+ \frac{2\sqrt{2\pi} L(\varphi)}{\delta} \Lambda(f) G_n(\calH)\nn\\
	&\qquad \qquad +\bigg(t+\sum_{j=1}^l \sqrt{\frac{\alpha}{2}\log (k_j+1)}+\sqrt{\log\log_2 2\delta^{-1}}\bigg)\sqrt{\frac{\tau_{\min}}{n}} +\frac{2}{\sqrt{n}}+ B_n\bigg]\bigg)\\
	&\qquad \leq \sum_{l=0}^{\infty}\sum_{k_1 \in \bbZ\setminus \{0\}}\cdots \sum_{k_l \in \bbZ\setminus \{0\}} 2\exp\bigg(-2\bigg(t+\sum_{j=1}^l \sqrt{\frac{\alpha}{2}\log (k_j+1)}\bigg)^2\bigg)\\
	&\qquad \leq \frac{\pi^2}{3}\big(3-2\zeta(\alpha)\big)^{-1} \exp\big(-2t^2\big), 
	\end{align} where the last equation is followed by using some algebraic manipulations.
\end{proof}
\section{Extension to High-Order Markov Chains}  \label{sec:ext}
In this section, we extend our results in previous sections to $m$-order Markov chains and a mixture of $m$ independent Markov services. 
\subsection{Extend to $m$-order Markov chain} \label{sec:ext1}
In this subsection, we extend our results in previous sections to $m$-order homogeneous Markov chain. The main idea is to convert $m$-order homogeneous Markov chains to $1$-order homogeneous Markov chain and use our results in previous sections to bound the generalization error. 

We start with the following simple example. 
\begin{example} \label{mmarkov} [$m$-order moving average process without noise] 
	Consider the following $m$-order Markov chain
	\begin{align}
	X_k=\sum_{i=1}^m a_i X_{k-i},\qquad k\in \bbZ_+ \label{QG1}. 
	\end{align}
	Let $Y_k:= [X_{k+m-1},X_{k+m-2},\cdots, X_k]^T$. Then, from \eqref{QG1}, we obtain
	\begin{align}
	Y_{k+1}=\bG Y_k, \qquad \forall k \in \bbZ_+
	\end{align}
	where
	\begin{align}
	\bG:=\begin{pmatrix} a_1&a_2&\cdots&a_{m-1}&a_m\\1&0&\cdots&0&0\\0&1&\cdots&0&0\\\vdots&\vdots&\ddots&\vdots&\vdots\\0&0&\cdots&1&0  \end{pmatrix}.
	\end{align} It is clear that $\{\bY_n\}_{n=1}^{\infty}$ is an order-$1$ Markov chain. Hence, instead of directly working with the $m$-order Markov chain $\{\bX_n\}_{n=1}^{\infty}$, we can find an upper bound for the Markov chain $\{Y_n\}_{n=1}^{\infty}$.
	
	To derive generalization error bounds for the Markov chain $\{Y_n\}_{n=1}^{\infty}$, we can use the following arguments. For all $f \in \calF$ and $(x_k,x_{k+1},\cdots,x_{k+m-1}) \in \calS^m$, by setting $\tilf(x_k,x_{k+1},\cdots,x_{k+m-1})=f(x_k)$ where $\tilf: \calS^m \to \bbR$, we obtain
	\begin{align}
	\frac{1}{n}\sum_{i=1}^n \bone\{f(X_i)\leq 0\}&=\frac{1}{n}\sum_{i=1}^n \bone\{\tilf(Y_i)\leq 0\} \label{eq244}. 
	\end{align}
	Hence, by applying all the results for $1$-order Markov chain $\{Y_n\}_{n=1}^{\infty}$, we obtain corresponding upper bounds for the sequence of $m$-order Markov chain $\{X_n\}_{n=1}^{\infty}$.
\end{example}
This approach can be extended to more general $m$-order Markov chain $X_k=g(X_{k-1},X_{k-2}\cdots,X_{k-m})$ where $g: \calS^m \to \bbR$. More specifically, for any tuple $(x_1,x_2,\cdots,x_m) \in \calS^m$, observe that
\begin{align}
dg= \frac{\partial g}{\partial x_1} dx_1+\frac{\partial g}{\partial x_2} dx_2+ \cdots + \frac{\partial g}{\partial x_m} dx_m \label{bag1}.  
\end{align}
Hence, if $\frac{\partial g}{\partial x_i}=\alpha_i$ for some constant $\alpha_i$ and for each $i \in [m]$, from \eqref{bag1}, we have
\begin{align}
g(x_1,x_2,\cdots,x_m)=g(c_1,c_2,\cdots,c_m)+ \sum_{i=1}^m a_i x_i + \sum_{i=1}^m a_i \nu_i \label{beg1},
\end{align} where $\nu_i$'s are constants.  One specific example where the function $g:\calS^m \to \bbR$ satisfies this property is $g(x_1,x_2,\cdots,x_m)=a_1 x_1 +a_2 x_2 + \cdots a_m x_m$, where $a_1,a_2,\cdots, a_m$ are constants as in Example \ref{mmarkov}.  

Now, by choosing $u=(g(c_1,c_2,\cdots,c_m)+\sum_{i=1}^m \alpha_i \nu_i\big)/(1-\sum_{i=1}^m a_i \nu_i)$, from \eqref{beg1}, we have
\begin{align}
g(x_1,x_2,\cdots,x_m)+u=\sum_{i=1}^m a_i \big(x_i+u\big) \label{batch}.
\end{align}
By setting $Y_k=\begin{bmatrix}X_{k+m-1}+u& X_{k+m-2}+u& \cdots & X_k+u \end{bmatrix}^T$, from \eqref{batch}, we have:
\begin{align}
Y_{k+1}=\begin{pmatrix} a_1&a_2&\cdots&a_{m-1}&a_m\\1&0&\cdots&0&0\\0&1&\cdots&0&0\\\vdots&\vdots&\ddots&\vdots&\vdots\\0&0&\cdots&1&0  \end{pmatrix}\begin{bmatrix} 
\end{bmatrix}Y_k.
\end{align}

In a more general setting, if $X_k=g(X_{k-1},\cdots,X_{k-m},V_k)$ for some random variable $V_k$ which is independent of $\{X_{k-i}\}_{i=1}^m$ such as the Autoregressive model (AR), where
\begin{align}
X_k=c+\sum_{i=1}^m a_i X_{t-i} + V_k \label{AR},	
\end{align} we can use the following conversion procedure. First, by using Taylor's approximation (to the first-order), we obtain
\begin{align}
g(x_1,x_2,\cdots,x_m,\xi)&\approx g(c_1,c_2,\cdots,c_m,\xi_0)+ \sum_{i=1}^m \frac{\partial g}{\partial x_i}\bigg|_{(c_1,c_2,\cdots,c_m,\zeta_0)} (x_i-c_i)\nn\\
&\qquad +\frac{\partial g}{\partial v}\bigg|_{(c_1,c_2,\cdots,c_m,\xi_0)}(\xi-\xi_0)
\end{align} for some good choice of $(c_1,c_2,\cdots,c_m, \xi_0) \in \calS^m \times \calV$, where $\calV$ is the alphabet of $V_k$. Using the above trick with $Y_k=\begin{bmatrix}X_{k+m-1}+u& X_{k+m-2}+u& \cdots & X_k+u\end{bmatrix}^T$, $a_i=\frac{\partial g}{\partial x_i}\big|_{(c_1,c_2,\cdots,c_m,\xi_0)}$, we can replace the recursion $X_k=g(X_{k-1},\cdots,X_{k-m},V_k)$ by the following equivalent recursion:
\begin{align}
Y_{k+1}=\begin{pmatrix} a_1&a_2&\cdots&a_{m-1}&a_m\\1&0&\cdots&0&0\\0&1&\cdots&0&0\\\vdots&\vdots&\ddots&\vdots&\vdots\\0&0&\cdots&1&0  \end{pmatrix}Y_k+ \begin{pmatrix} \frac{\partial g}{\partial v}\big|_{(c_1,c_2,\cdots,c_m,\xi_0)}V_k \\0\\0 \\ \vdots \\0 \end{pmatrix} \label{baga}.
\end{align}
Since $V_k$ is independent of $\{X_{k+m-i}\}_{i=2}^{m+1}$ or $Y_k$, \eqref{baga} models a new $1$-order Markov chain $\{Y_k\}_{k=1}^{\infty}$. Then, by using the the same arguments to obtain \eqref{eq244}, we can derive  bounds on generalization error for this model. 

For a general $m$-order homogeneous Markov chain, it holds that
\begin{align}
P_{X_k|X_{k-1}=x_1,X_{k-2}=x_2,\cdots, X_{k-m}=x_m} \sim T_{(x_1,x_2,\cdots,x_m)}
\end{align}  for all $(x_1,x_2,\cdots,x_m) \in \calS^m$, where $T_{(x_1,x_2,\cdots,x_m)}$ is a random variable which depends only on $x_1,x_2,\cdots,x_m$ and does not depend on $k$. Hence, we can represent the 
\begin{align}
X_k =\tilg(X_{k-1},X_{k-2},\cdots, X_{k-m},T_{(X_{k-1},X_{k-2},\cdots, X_{k-m})}), 
\end{align} where
$T_{(X_{k-1},X_{k-2},\cdots, X_{k-m})}=f(\eps_k,V_k,V_{k-1},\cdots, V_{k-q}, X_{k-1}, X_{k-2},\cdots, X_{k-m})$. Here, $\eps_k$ represents new noise at time $k$ which is independent of the past. Hence, in a general $m$-order homogeneous Markov chain, we can represent the $m$-order homogeneous Markov chain by the following recursion:
\begin{align}
X_k=g(X_{k-1},X_{k-2},\cdots, X_{k-m},\eps_k,V_k,V_{k-1},\cdots, V_{k-q}\big) \label{eq175},
\end{align}  where $\eps_k$ represents new noise at time $k$ and $q \in \bbZ_+$. By using Taylor expansion to the first order, we can approximate the Markov chain in \eqref{eq175} by an Autoregressive Moving-Average Model (ARMA$(m,q)$) model as following:
\begin{align}
X_k=c+\eps_k+ \sum_{i=1}^m a_i X_{k-i}+ \sum_{i=1}^q \theta_i \eps_{k-i} \label{tibo},
\end{align} where $c$ and $a_i$'s are constants, and $\{\eps_k\}_{k=1}^{\infty}$ are i.~i.~d. Gaussian random variables $\calN(0,\sigma^2)$. For this model, let
\begin{align}
V_k:= \sum_{i=1}^k \eps_i
\end{align}
and
\begin{align}
Y_k:=\begin{pmatrix} X_{k+m-1}+u \\X_{k+m-2}+u\\\vdots \\X_k+u\\ V_{k+m-1}\\ V_{k+m-2}\\ \vdots \\ V_{k+m-q}\\V_{k+m-q-1} \end{pmatrix},
\end{align}
where
\begin{align}
u:=\frac{c}{1-\sum_{i=1}^m a_i}.
\end{align}
Let $V_k:=\sum_{i=1}^k \eps_i$ for all $k\geq 1$. Observe that
\begin{align}
V_{k+m}=\eps_{k+m}+ V_{k+m-1} \label{baget1}.
\end{align}
On the other hand, we have
\begin{align}
X_{k+m}&=c+\eps_{k+m}+ \sum_{i=1}^m a_i X_{k+m-i} +\sum_{i=1}^q \theta_i \eps_{k+m-i}\\
&=c+\eps_{k+m}+ \sum_{i=1}^m a_i X_{k+m-i} + \sum_{i=1}^q \theta_i \big(V_{k+m-i}-V_{k+m-1-i}\big)\\
&=c+\eps_{k+m}+ \sum_{i=1}^m a_i X_{k+m-i}+ \theta_1 V_{k+m-1}+ \sum_{i=1}^{q-1} \big(\theta_{i+1}-\theta_i\big) V_{k+m-1-i}-\theta_q V_{k+m-q-1} \label{baget2}.
\end{align}
Then, we have
\begin{align}
Y_{k+1}=\bG Y_k + \begin{pmatrix} \eps_{k+m} \\0 \\\vdots \\0 \\ \eps_{k+m}\\ 0 \\ \vdots \\ 0 \end{pmatrix} \label{eqbet},
\end{align} where
\begin{align}
\bG:=\begin{pmatrix} \bG_{11}&\bG_{12}\\ \bG_{21}&\bG_{22}
\end{pmatrix}.
\end{align}
Here, 
\begin{align}
\bG_{11}:=\begin{pmatrix} a_1&a_2&\cdots&a_{m-1}&a_m \\1&0&\cdots&0&0\\0&1&\cdots&0&0\\\vdots&\vdots&\ddots&\vdots&\vdots\\0&0&\cdots&1&0   \end{pmatrix}_{m \times q+1},
\end{align}
\begin{align}
\bG_{12}:=\begin{pmatrix} \theta_1&\theta_2-\theta_1&\cdots&\theta_q-\theta_{q-1}&-\theta_q \\0&0&\cdots&0&0\\0&0&\cdots&0&0\\\vdots&\vdots&\ddots&\vdots&\vdots\\0&0&\cdots&0&0   \end{pmatrix}_{m \times q+1},
\end{align}
\begin{align}
\bG_{21}:=\begin{pmatrix} 0&0&\cdots&0&0\\0&0&\cdots&0&0\\0&0&\cdots&0&0\\\vdots&\vdots&\ddots&\vdots&\vdots\\0&0&\cdots&0&0  \end{pmatrix}_{q+1\times m},
\end{align}
and
\begin{align}
\bG_{22}:=\begin{pmatrix} 1&0&\cdots&0&0 \\1&0&\cdots&0&0\\0&1&\cdots&0&0\\\vdots&\vdots&\ddots&\vdots&\vdots\\0&0&\cdots&1&0   \end{pmatrix}_{q+1 \times q+1}.
\end{align}
Since $\eps_{k+m}$ is independent of $Y_k$, \eqref{eqbet} models a $1$-order Markov chain. Hence, we can use the above arguments to derive new generalization error bounds for the $m$-order homogeneous Markov chain where ARMA model is a special case. 
\subsection{Mixture of $m$ Services}
In this section, we consider the case that $Y_k=\sum_{l=1}^m \alpha_l X_k^{(l)}$ for all $k=1,2,\cdots$, where $\{X_k^{(l)}\}_{k=1}^{\infty}$ are independent Markov chains on $\calS$ with stationary distribution for all $l\in [m]$. This setting usually happens in practice, for example, video is a mixture of voice, image, and text, where each service can be modelled as a high-order Markov chain and the order of the Markov chain depends on the type of service.

Let
\begin{align}
Z_k:=\begin{pmatrix} \alpha_1 X_k^{(1)}+ \alpha_2 X_k^{(2)}+\cdots \alpha_m X_k^{(m)}\\   \alpha_2 X_k^{(2)}+\cdots \alpha_m X_k^{(m)}\\ \vdots \\\alpha_m X_k^{(m)} \end{pmatrix}=\begin{pmatrix} Y_k\\   \alpha_2 X_k^{(2)}+\cdots \alpha_m X_k^{(m)}\\ \vdots \\\alpha_m X_k^{(m)} \end{pmatrix}.
\end{align}
Then, it holds that
\begin{align}
Z_k=\bG X_k
\end{align}
where
\begin{align}
\bG:=\begin{pmatrix} \alpha_1&\alpha_2&\cdots& \alpha_{m-1}&\alpha_m\\ 0&\alpha_2&\cdots& \alpha_{m-1}&\alpha_m\\ \vdots&\vdots&\vdots& \vdots&\vdots \\ 0&0&\cdots& 0&\alpha_m\end{pmatrix},
\end{align}
and
\begin{align}
X_k:=\begin{pmatrix}X_k^{(1)}\\X_k^{(2)}\\\vdots\\ X_k^{(m)}\end{pmatrix} \label{hifi}.
\end{align}
It is obvious that $\bG$ is non-singular since  $\det(\bG)=\prod_{l=1}^m \alpha_l\neq 0$. Therefore, for fixed pair $(x,y) \in \calS^m \times \calS^m$, we have
\begin{align}
\bbP\big(Z_{k+1}=y\big|Z_k=x\big)&=\bbP(X_{k+1}=\bG^{-1} y\big|X_k=\bG^{-1}x) \label{mage}.
\end{align}
Now, assume that
$\bG^{-1}x=\begin{pmatrix} \beta_x^{(1)},\beta_x^{(2)},\cdots \beta_x^{(m)})\end{pmatrix}$ and $\bG^{-1}y=\begin{pmatrix} \beta_y^{(1)},\beta_y^{(2)},\cdots \beta_y^{(m)})\end{pmatrix}$. Then, from \eqref{hifi} and \eqref{mage}, we have
\begin{align}
\bbP\big(Z_{k+1}=y\big|Z_k=x\big)&= \bbP\bigg(\bigcap_{l=1}^m \big\{X_{k+1}^{(l)}=\beta_y^{(l)}\big\}\bigg| \bigcap_{l=1}^m \big\{X_k^{(l)}=\beta_x^{(l)}\big\} \bigg)\\
&=\prod_{l=1}^m \bbP\big(X_{k+1}^{(l)}=\beta_y^{(l)}\big|X_k^{(l)}=\beta_x^{(l)}\big)\\
&=\prod_{l=1}^m Q_l(\beta_x^{(l)}, \beta_y^{(l)}),
\end{align} where $Q_l$ is the transition probability of the Markov chain $l$. It follows that $\{Z_k\}_{k=1}^{\infty}$ is a $1$-order Markov chain. It is easy to see that $\{Z_k\}_{k=1}^{\infty}$ has stationary distribution if all the Markov chains $\{X_k^{(l)}\}_{l=1}^m$ have stationary distributions.

Now, as Subsection \ref{sec:ext1}, to derive generalization error bounds for the Markov chain $\{X_n\}_{n=1}^{\infty}$, we can use the following arguments. For all $f \in \calF$ and by setting $\tilf(z_1,z_2,\cdots,z_m):=f(z_1)$ where $\tilf: \calS^m \to \bbR$, we obtain $\tilf(GX_k)=f(Y_k)$ and
\begin{align}
\frac{1}{n}\sum_{i=1}^n \bone\{f(Y_i)\leq 0\}&=\frac{1}{n}\sum_{i=1}^n \bone\{\tilf(\bG X_i)\leq 0\} \label{eq244mod}. 
\end{align}
Hence, by applying all the results for $1$-order Markov chain $\{Z_n\}_{n=1}^{\infty}$ where $Z_n=\bG X_n$, we obtain corresponding upper bounds for the sequence of $m$-order Markov chain $\{X_n\}_{n=1}^{\infty}$.
\appendix
\section{Proof of Lemma \ref{lem:key}} \label{append1}
Before going to prove Lemma \ref{lem:key}, we observe the following interesting fact.
\begin{lemma} \label{aumat:lem} Let $\{X_n\}_{n=1}^{\infty}$ be  an arbitrary process on a Polish space $\calS$, and let $\{Y_n\}_{n=1}^{\infty}$ be a independent copy (replica) of $\{X_n\}_{n=1}^{\infty}$. Denote by $\bX=(X_1,X_2,\cdots,X_n), \bY=(Y_1,Y_2,\cdots,Y_n)$, and $\calF$ a class of uniformly bounded functions from  $\calS \to \bbR$. Let $\bepsilon:=(\eps_1,\eps_2,\cdots,\eps_n)$ be a vector of i.i.d. Rademacher's random variables. Then, the following holds:
	\begin{align}
	\bbE_{\bepsilon }\bigg[\bbE_{\bX,\bY} \bigg[\sup_{f \in \calF} \bigg|\sum_{i=1}^n\eps_i (f(X_i)-f(Y_i))\bigg|\bigg]\bigg]=\bbE_{\bX,\bY}\bigg[\sup_{f \in \calF} \bigg|\sum_{i=1}^n (f(X_i)-f(Y_i))\bigg|\bigg]\label{facs}.
	\end{align}
\end{lemma}
\begin{remark} Our lemma generalizes a similar fact for i.i.d. processes. In the case that $\{X_n\}_{n=1}^{\infty}$ is an i.i.d. random process, \eqref{facs} holds with equality since $P_{X^n,Y^n}(x_1,x_2,\cdots,x_n,y_1,y_2,\cdots,y_n)$ is invariant under permutation. However, for the Markov case, this fact does not hold in general. Hence, in the following, we provide a new proof for \eqref{facs}, which works for any process $\{X_n\}_{n=1}^{\infty}$ by making use of the properties of  Rademacher's process.
\end{remark}
\begin{proof} Let $g(x,y):=f(x)-f(y)$ and $\calG:=\{g:\calS \times \calS \to \bbR: g(x,y):=f(x)-f(y)\enspace \mbox{for some}\enspace f \in \calF\}$. Then, it holds that
	\begin{align}
	\bbE_{\bepsilon }\bigg[\bbE_{\bX,\bY} \bigg[\sup_{f \in \calF} \bigg|\sum_{i=1}^n\eps_i (f(X_i)-f(Y_i))\bigg|\bigg]\bigg]&=\bbE_{\bepsilon }\bigg[\bbE_{\bX,\bY} \bigg[\sup_{g \in \calG} \bigg|\sum_{i=1}^n\eps_i g(X_i,Y_i)\bigg|\bigg]\bigg] \label{eka1}.
	\end{align}	
	Observe that $g(X_i,Y_i)=-g(Y_i,X_i)$ for all $i \in [n]$. 
	
	For all $j\in [n]$, denote by
	\begin{align}
	\calN_j:=[n]\setminus \{j\},
	\end{align}
	and
	\begin{align}
	\eps_{\calN_j}:=\{\eps_i: i \in \calN_j\}.
	\end{align}
	Now, for each $j\in [n]$, observe that $\eps_j$ is independent of $X_1^n,Y_1^n, \eps_{\calN_j}$. Hence, we have
	\begin{align}
	&\bbE_{\eps_1^n,X_1^n,Y_1^n} \bigg[\sup_{g \in \calG} \bigg|\sum_{i=1}^n \eps_i g(X_i,Y_i)\bigg|\bigg]\nn\\
	&\qquad =\frac{1}{2}\bbE_{\eps_{\calN_j}, X_1^n, Y_1^n}\bigg[\sup_{g \in \calG} \bigg|\sum_{i\in \calN_j}\eps_i g(X_i,Y_i)+ g(X_j,Y_j)\bigg|\bigg]\nn\\
	&\qquad \qquad+ \frac{1}{2}\bbE_{\eps_{\calN_j}, X_1^n, Y_1^n}\bigg[\sup_{g \in \calG} \bigg|\sum_{i \in \calN_j}\eps_i g(X_i,Y_i)- g(X_j,Y_j)\bigg|\bigg]\\
	&\qquad =\frac{1}{2}\bbE_{\eps_{\calN_j}, X_1^n, Y_1^n}\bigg[\sup_{g \in \calG} \bigg|\sum_{i \in \calN_j} \eps_i g(X_i,Y_i)+ g(X_j,Y_j)\bigg|\bigg]\nn\\
	&\qquad \qquad+ \frac{1}{2}\bbE_{\eps_{\calN_j}, X_1^n, Y_1^n}\bigg[\sup_{g \in \calG} \bigg|\sum_{i \in \calN_j}\eps_i g(X_i,Y_i)+g(Y_j,X_j)\bigg|\bigg] \label{apat1}
	\end{align} where \eqref{apat1} follows from $g(X_j,Y_j)=-g(Y_j,X_j)$.

	Now, by setting $\tepsilon_i:=-\eps_i$ for all $i\in n$. Then, we obtain
	\begin{align}
	&\bbE_{\eps_{\calN_j}, X_1^n, Y_1^n}\bigg[\sup_{g \in \calG} \bigg|\sum_{i \in \calN_j} \eps_i g(X_i,Y_i)+g(Y_j,X_j)\bigg|\bigg]\nn\\
	&\qquad \qquad=\bbE_{\tilde{\eps}_{\calN_j}, X_1^n, Y_1^n}\bigg[\sup_{g \in \calG} \bigg|\sum_{i\in \calN_j}\tepsilon_i g(X_i,Y_i)+g(Y_j,X_j)\bigg|\bigg] \label{aqua2}\\
	&\qquad \qquad =\bbE_{\eps_{\calN_j}, X_1^n, Y_1^n}\bigg[\sup_{g \in \calG} \bigg|\sum_{i\in \calN_j}\eps_i g(Y_i,X_i)+g(Y_j,X_j)\bigg|\bigg]  \label{aqua3} \\
	&\qquad\qquad=\bbE_{\eps_{\calN_j}, X_1^n, Y_1^n}\bigg[\sup_{g \in \calG} \bigg|\sum_{i\in \calN_j}\eps_i g(X_i,Y_i)+g(X_j,Y_j)\bigg|\bigg] \label{aqua5}, 
	\end{align} where \eqref{aqua2} follows from $(\tepsilon_i: i \in \calN_j)$  has the same distribution as $(\eps_i:i \in \calN_j)$, \eqref{aqua3} follows from $g(X_i,Y_i)=-g(Y_i,X_i)$ and $\tepsilon_i=-\eps_i$, and \eqref{aqua5} follows from $g(X_i,Y_i)=-g(Y_i,X_i)$ for all $i\in [n]$.
	
	From \eqref{apat1} and \eqref{aqua5}, we have
	\begin{align}
	\bbE_{\eps_1^n, X_1^n, Y_1^n}\bigg[\sup_{g \in \calG} \bigg|\sum_{i=1}^n\eps_i g(X_i,Y_i)\bigg|\bigg]=\bbE_{\eps_{\calN_j}, X_1^n, Y_1^n}\bigg[\sup_{g \in \calG} \bigg|\sum_{i\in \calN_j}\eps_i g(X_i,Y_i)+g(X_j,Y_j)\bigg|\bigg] \quad \forall j\in [n]\label{apat4}.
	\end{align}
	
	It follows that
	\begin{align}
	&\bbE_{\eps_1^n, X_1^n, Y_1^n}\bigg[\sup_{g \in \calG} \bigg|\sum_{i=1}^n\eps_i g(X_i,Y_i)\bigg|\bigg]\nn\\
	&\qquad =\bbE_{\eps_1^{n-1}, X_1^n, Y_1^n}\bigg[\sup_{g \in \calG} \bigg|\sum_{i=1}^{n-1}\eps_i g(X_i,Y_i)+g(X_n,Y_n)\bigg|\bigg]\label{B1}\\
	&\qquad =\frac{1}{2}\bbE_{\eps_1^{n-2}, X_1^n, Y_1^n}\bigg[\sup_{g \in \calG} \bigg|\sum_{i=1}^{n-2}\eps_i g(X_i,Y_i)+g(X_{n-1},Y_{n-1})+g(X_n,Y_n)\bigg|\bigg] \nn\\
	&\qquad\qquad + \frac{1}{2}\bbE_{\eps_1^{n-2}, X_1^n, Y_1^n}\bigg[\sup_{g \in \calG} \bigg|\sum_{i=1}^{n-2}\eps_i g(X_i,Y_i)-g(X_{n-1},Y_{n-1})+g(X_n,Y_n)\bigg|\bigg]\\
	&\qquad =\frac{1}{2}\bbE_{\eps_1^{n-2}, X_1^n, Y_1^n}\bigg[\sup_{g \in \calG} \bigg|\sum_{i=1}^{n-2}\eps_i g(X_i,Y_i)+g(X_{n-1},Y_{n-1})+g(X_n,Y_n)\bigg|\bigg]\nn\\ &\qquad\qquad + \frac{1}{2}\bbE_{\eps_1^{n-2}, X_1^n, Y_1^n}\bigg[\sup_{g \in \calG} \bigg|\sum_{i=1}^{n-2}\eps_i g(X_i,Y_i)+g(Y_{n-1},X_{n-1})+g(X_n,Y_n)\bigg|\bigg] \label{ato7},
	\end{align} where \eqref{B1} follows from setting $j=n$ in \eqref{apat4}, and \eqref{ato7} follows from $g(Y_{n-1},X_{n-1})=-g(X_{n-1},Y_{n-1})$.
	
	Now, for any fixed tuple $(x_1^{n-1},y_1^{n-1},\eps_1^{n-1}) \in \calS^{n-1}\times \calS^{n-1}\times \{-1,1\}^{n-1}$, observe that 
	\begin{align}
	&P_{X_n,Y_n|X_1^{n-1},Y_1^{n-1},\eps_1^{n-2}}(x_n,y_n|x_1^{n-1},y_1^{n-1},\eps_1^{n-2})\nn\\
	&\qquad \qquad= P_{X_n|X_1^{n-1}}(x_n|x_1^{n-1})P_{Y_n|Y_1^{n-1}}(y_n|y_1^{n-1})\\
	&\qquad \qquad=
	P_{Y_n|Y_1^{n-1}}(x_n|x_1^{n-1})P_{X_n|X_1^{n-1}}(y_n|y_1^{n-1})\\
	&\qquad\qquad=P_{Y_n,X_n|Y_1^{n-1},X_1^{n-1},\eps_1^{n-2}}(x_n,y_n|x_1^{n-1},y_1^{n-1},\eps_1^{n-2}) \label{eq131}.
	\end{align}
	On the other hand, we also have
	\begin{align}
	&P_{X_1^{n-1},Y_1^{n-1},\eps_1^{n-2}}(x_1^{n-1},y_1^{n-1},\eps_1^{n-2})\nn\\
	&\qquad =P_{\eps_1^{n-2}}(\eps_1^{n-2}) P_{X_1^{n-1}}(x_1^{n-1})P_{Y_1^{n-1}}(y_1^{n-1})\\
	&\qquad= P_{\eps_1^{n-2}}(\eps_1^{n-2}) P_{Y_1^{n-1}}(x_1^{n-1})P_{X_1^{n-1}}(y_1^{n-1})\\
	&\qquad =P_{Y_1^{n-1},X_1^{n-1},\eps_1^{n-2}}(x_1^{n-1},y_1^{n-1},\eps_1^{n-2}) \label{AQ}.
	\end{align} 
	Hence, from \eqref{eq131} and \eqref{AQ}, we obtain
	\begin{align}
	&P_{ X_1^n,Y_1^n,\eps_1^{n-2}}\big(x_1^n,y_1^n,\eps_1^{n-2}\big)\nn\\
	&\qquad =P_{X_1^{n-1} Y_1^{n-1} \eps_1^{n-2}}(x_1^{n-1},y_1^{n-1},\eps_1^{n-2}) P_{X_n Y_n|X_1^{n-1}Y_1^{n-1}\eps_1^{n-2}}\big(x_n,y_n|x_1^{n-1},y_1^{n-1},\eps_1^{n-2}\big) \\
	&\qquad = P_{Y_1^{n-1},X_1^{n-1},\eps_1^{n-2}}(x_1^{n-1},y_1^{n-1},\eps_1^{n-2})P_{Y_n,X_n|Y_1^{n-1},X_1^{n-1},\eps_1^{n-2}}\big(x_n,y_n|x_1^{n-1},y_1^{n-1},\eps_1^{n-2}\big) \\
	&\qquad=P_{ Y_1^n,X_1^n,\eps_1^{n-2}}\big(x_1^n,y_1^n,\eps_1^{n-2} \big) \label{BT}. 
	\end{align}
	Now, from \eqref{BT}, we also have
	\begin{align}
	P_{X_1^n,Y_1^n}(x_1^n,y_1^n)=P_{Y_1^n X_1^n}(x_1^n,y_1^n) \label{BT0}.
	\end{align}
	It follows from \eqref{BT0} that
	\begin{align}
	P_{X_{n-1},Y_{n-1}}(x_{n-1},y_{n-1})&=\sum_{x_1^{n-2},y_1^{n-2},x_n,y_n} P_{X_1^n,Y_1^n}(x_1^n,y_1^n)\\
	&=\sum_{x_1^{n-2},y_1^{n-2},x_n,y_n} P_{Y_1^n X_1^n}(x_1^n,y_1^n)\\
	&=P_{Y_{n-1}X_{n-1}}(x_{n-1},y_{n-1}) \label{BT4}.
	\end{align}
	Hence, from \eqref{BT} and \eqref{BT4}, we have
	\begin{align}
	&P_{X_n,Y_n, X_1^{n-2},Y_1^{n-2},\eps_1^{n-2}|X_{n-1},Y_{n-1}}\big(x_n,y_n,x_1^{n-2},y_1^{n-2},\eps_1^{n-2}|x_{n-1},y_{n-1}\big)\\
	&\qquad=\frac{P_{X_1^n, Y_1^n, \eps_1^{n-2}}(x_1^n,y_1^n,\eps_1^{n-2})}{P_{X_{n-1} Y_{n-1}}(x_{n-1},y_{n-1})}\\
	&\qquad= \frac{P_{ Y_1^n,X_1^n,\eps_1^{n-2}}\big(x_1^n,y_1^n,\eps_1^{n-2} \big) }{P_{Y_{n-1} X_{n-1}}(x_{n-1},y_{n-1})} \label{BT3}\\
	&\qquad=P_{X_n,Y_n, X_1^{n-2},Y_1^{n-2},\eps_1^{n-2}|X_{n-1},Y_{n-1}}\big(y_n,x_n,y_1^{n-2},x_1^{n-2},\eps_1^{n-2}|y_{n-1},x_{n-1}\big) \label{keymoi}.
	\end{align} 
	From \eqref{keymoi}, for each fixed $(x_{n-1},y_{n-1}) \in \calS \times \calS$, we have
	\begin{align}
	&\bbE_{\eps_1^{n-2}, X_1^{n-2}, Y_1^{n-2},X_n,Y_n}\bigg[\sup_{g \in \calG} \bigg|\sum_{i=1}^{n-2}\eps_i g(X_i,Y_i)+g(Y_{n-1},X_{n-1})\nn\\
	&\qquad\qquad  +g(X_n,Y_n)\bigg|\bigg|X_{n-1}=x_{n-1},Y_{n-1}=y_{n-1}\bigg]\nn\\
	&\qquad=\bbE_{\eps_1^{n-2}, X_1^{n-2}, Y_1^{n-2},X_n,Y_n}\bigg[\sup_{g \in \calG} \bigg|\sum_{i=1}^{n-2}\eps_i g(X_i,Y_i)+g(y_{n-1},x_{n-1})\nn\\
	&\qquad \qquad +g(X_n,Y_n)\bigg|\bigg|X_{n-1}=x_{n-1},Y_{n-1}=y_{n-1}\bigg]\nn\\
	&\qquad=\bbE_{\eps_1^{n-2}, X_1^{n-2}, Y_1^{n-2},X_n,Y_n}\bigg[\sup_{g \in \calG} \bigg|\sum_{i=1}^{n-2}\eps_i g(Y_i,X_i)+g(y_{n-1},x_{n-1})\nn\\
	&\qquad \qquad +g(Y_n,X_n)\bigg|\bigg|Y_{n-1}=x_{n-1},X_{n-1}=y_{n-1}\bigg] \label{latta}\\
	&\qquad= \bbE_{\eps_1^{n-2}, X_1^{n-2}, Y_1^{n-2},X_n,Y_n}\bigg[\sup_{g \in \calG} \bigg|\sum_{i=1}^{n-2}\eps_i g(X_i,Y_i)+g(x_{n-1},y_{n-1})\nn\\
	&\qquad \qquad +g(X_n,Y_n)\bigg|\bigg|Y_{n-1}=x_{n-1},X_{n-1}=y_{n-1}\bigg] \label{TB10}
	\end{align} where \eqref{latta} follows from \eqref{keymoi}, and \eqref{TB10} follows from the fact that $g(x,y)=-g(y,x)$ for all $x,y \in \calS \times \calS$. 
	
	From \eqref{TB10} and \eqref{BT4}, we obtain
	\begin{align}
	&\bbE_{\eps_1^{n-2}, X_1^n, Y_1^n}\bigg[\sup_{g \in \calG} \bigg|\sum_{i=1}^{n-2}\eps_i g(X_i,Y_i)+g(Y_{n-1},X_{n-1})+g(X_n,Y_n)\bigg|\bigg]\nn\\
	&\qquad =  \bbE_{\eps_1^{n-2}, X_1^n,Y_1^n}\bigg[\sup_{g \in \calG} \bigg|\sum_{i=1}^{n-2}\eps_i g(X_i,Y_i)+g(X_{n-1},Y_{n-1})+g(X_n,Y_n)\bigg|\bigg] \label{TB11}.
	\end{align}
	From \eqref{ato7} and \eqref{TB11}, we obtain
	\begin{align}
	&\bbE_{\eps_1^n, X_1^n, Y_1^n}\bigg[\sup_{g \in \calG} \bigg|\sum_{i=1}^n\eps_i g(X_i,Y_i)\bigg|\bigg]\nn\\
	&\qquad \qquad= \bbE_{\eps_1^{n-2}, X_1^n, Y_1^n}\bigg[\sup_{g \in \calG} \bigg|\sum_{i=1}^{n-2}\eps_i g(X_i,Y_i)+g(X_{n-1},Y_{n-1})+g(X_n,Y_n)\bigg|\bigg].
	\end{align}
	By using induction, we finally obtain
	\begin{align}
	\bbE_{\eps_1^n, X_1^n, Y_1^n}\bigg[\sup_{g \in \calG} \bigg|\sum_{i=1}^n\eps_i g(X_i,Y_i)\bigg|\bigg]= \bbE_{X_1^n, Y_1^n}\bigg[\sup_{g \in \calG} \bigg|\sum_{i=1}^n g(X_i,Y_i)\bigg|\bigg],
	\end{align} or equation \eqref{facs} holds.
\end{proof}
Next, recall the following result which was developed base on the spectral method \citep{Pascal2001}:
\begin{lemma} \citep[Theorems 3.41]{Rudolf2011}  \label{lem:berryessen} Let $X_1,X_2,\cdots,X_n$ be a stationary Markov chain on some Polish space with $L_2$ spectral gap $\lambda$ defined in Section 2.1 in the main document and the initial distribution $\nu \in \calM_2$. Let $f \in \calF$ and define
	\begin{align}
	S_{n,n_0}(f)=\frac{1}{n}\sum_{j=1}^n f(X_{j+n_0})
	\end{align} for all $n_0\geq 0$. Then, it holds that
	\begin{align}
	\bbE\bigg[\bigg|S_{n,n_0}(f)-\bbE_{\pi}[f(X)]\bigg|^2\bigg]\leq \frac{2M}{n(1-\lambda)}+ \frac{64 M^2}{n^2(1-\lambda)^2}\lambda^{n_0}\bigg\|\frac{dv}{d\pi}-1\bigg\|_2.
	\end{align}
\end{lemma}
In addition, we recall the following McDiarmid's inequality for Markov chain:
\begin{lemma} \citep[Cor.~2.10]{Daniel2015} \label{lem:HD} Let $\bX:=(X_1,X_2,\cdots,X_n)$ be a homogeneous Markov chain, taking values in a Polish state space $\Lambda=\underbrace{\calS\times \calS \times \cdots \times \calS}_{n \enspace \rm{times}}$, with mixing time $t_{\rm{mix}}(\eps)$ (for $0\leq \eps \leq 1$). Recall the definition of $\tau_{\min}$ in Eq.~(3) in the main document. Suppose that $f: \Lambda \to \bbR$ satisfies 
	\begin{align}
	f(\bx)-f(\by)\leq \sum_{i=1}^n c_i \bone\{x_i\neq y_i\}
	\end{align} for every $\bx, \by \in \Lambda$, for some $c \in \bbR_+^n$. Then, for any $t\geq 0$,
	\begin{align}
	\bbP_{\bX}\big[\big|f(\bX)-\bbE[f(\bX)]\big|\geq t\big]\leq 2\exp \bigg(-\frac{2t^2}{\|c\|_2^2 \tau_{\min}}\bigg).
	\end{align}
\end{lemma}
Now, we return to the proof of Lemma \ref{lem:key}. 
\begin{proof}[Proof of Lemma \ref{lem:key}] \label{prooflemkey}For each $f\in \calF$, observe that
	\begin{align}
	&\frac{1}{n}\sum_{i=1}^n f(X_i)-\int_{\calS}  \pi(x) f(x)dx\nn\\
	&= \frac{1}{n}\sum_{i=1}^n f(X_i)- \bbE[f(X_i)] + \frac{1}{n}\sum_{i=1}^n \bbE[f(X_i)] -\int_{\calS} \pi(x) f(x) dx \label{C1}.
	\end{align}	
	On the other hand, we have
	\begin{align}
	&\bigg|\frac{1}{n}\sum_{i=1}^n \bbE[f(X_i)]-\int_{\calS}  \pi(x) f(x)dx\bigg|\nn\\
	&\qquad =  \bbE\bigg[\bigg|S_{n,0}(f)-\bbE_{\pi}[f(X)]\bigg|\bigg] \\
	&\qquad \leq \sqrt{\bbE\bigg[\bigg|S_{n,0}(f)-\bbE_{\pi}[f(X)]\bigg|^2 \bigg]} \\
	&\qquad \leq \sqrt{\frac{2M}{n(1-\lambda)}+ \frac{64 M^2}{n^2(1-\lambda)^2}\bigg\|\frac{dv}{d\pi}-1\bigg\|_2}\\
	&\qquad =A_n \label{keyg},
	\end{align} where \eqref{keyg} follows from Lemma \ref{lem:berryessen} with $n_0=0$.
	
	By using $|a+b|\leq |a|+|b$, from \eqref{C1} and \eqref{keyg}, we obtain
	\begin{align}
	&\bbE\big[\big\|P_n-P\big\|_{\calF}\big]\nn\\
	&\qquad \leq \bbE\bigg[\sup_{f\in \calF}\bigg|\frac{1}{n}\sum_{i=1}^n f(X_i)- \bbE[f(X_i)]\bigg|\bigg] + A_n\label{eq15}.
	\end{align} 
	On the other hand, let $Y_1,Y_2,\cdots, Y_n$ is a replica of $X_1,X_2,\cdots, X_n$. It holds that
	\begin{align}
	&\bbE\bigg[\sup_{f\in \calF}\bigg|\frac{1}{n}\sum_{i=1}^n f(X_i)- \bbE[f(X_i)]\bigg|\bigg]\nn\\
	&\qquad =\bbE_{\bX}\bigg[\sup_{f\in \calF}\bigg|\frac{1}{n}\sum_{i=1}^n f(X_i)- \bbE[f(Y_i)]\bigg|\bigg] \\
	&\qquad =\bbE_{\bX}\bigg[\sup_{f\in \calF}\bigg|\bbE_{\bY}\bigg[\frac{1}{n}\sum_{i=1}^n f(X_i)- f(Y_i)\bigg]\bigg|\bigg] \label{a0} \\
	&\qquad \leq \bbE_{\bX} \bigg[\bbE_{\bY}\bigg[\sup_{f\in \calF}\bigg|\frac{1}{n}\sum_{i=1}^n f(X_i)- f(Y_i)\bigg|\bigg]\bigg] \label{a1}.
	\end{align}
	
	Now, by Lemma \ref{aumat:lem} and the triangle inequality for infinity norm, we have
	\begin{align}
	&\bbE\bigg[\sup_{f\in \calF}\bigg|\frac{1}{n}\sum_{i=1}^n f(X_i)-f(Y_i)\bigg|\bigg]=\bbE_{\eps} \bbE_{\bX,\bY}\bigg[\bigg\|\frac{1}{n}\sum_{i=1}^n \eps_i\big(f(X_i)- f(Y_i)\big)\bigg\|_{\calF}\bigg]\\
	&\qquad \leq \bbE_{\eps} \bbE_{\bX}\bigg[\bigg\|\frac{1}{n}\sum_{i=1}^n \eps_i\big(f(X_i)\big)\bigg\|_{\calF}\bigg]+ \bbE_{\eps} \bbE_{\bY}\bigg[\bigg\|\frac{1}{n}\sum_{i=1}^n \eps_i\big(f(Y_i)\big)\bigg\|_{\calF}\bigg]\\
	&\qquad = 2\bbE\big[\|P_n^0\|_{\calF}\big]
	\label{a4},
	\end{align} where \eqref{a4} follows from the fact that $\bY$ is a replica of $\bX$.
	
	From \eqref{eq15} and \eqref{a4}, we finally obtain
	\begin{align}
	\bbE\big[\big\|P_n-P\big\|_{\calF}\big]  \leq 2\bbE\big[\|P_n^0\|_{\calF}\big] +A_n\label{xpess}.
	\end{align} 
	Now, by using the triangle inequality, we have
	\begin{align}
	\bbE[\|P_n^0\|_{\calF}]&=\bbE_{\bX,\boldsymbol{\eps}}\bigg[\sup_{f\in \calF} \bigg|\frac{1}{n}\sum_{i=1}^n \eps_i f(X_i)\bigg|\bigg]\\
	&\leq \bbE\bigg[\sup_{f\in \calF} \bigg|\frac{1}{n}\sum_{i=1}^n \eps_i\big(f(X_i)-\bbE_{\bY}\big[f(Y_i)\big]\big)\bigg| +  \bbE\bigg[\sup_{f\in \calF} \bigg|\frac{1}{n}\sum_{i=1}^n \eps_i\bbE_{\bY}[f(Y_i)]\bigg|\bigg]\\
	&=\bbE_{\bX,\boldsymbol{\eps}}\bigg[\sup_{f\in \calF}\bigg|\bbE_{\bY}\bigg[\frac{1}{n}\sum_{i=1}^n \eps_i\big(f(X_i)-f(Y_i)\big)\bigg]\bigg|\bigg]+  \bbE\bigg[\sup_{f\in \calF} \bigg|\frac{1}{n}\sum_{i=1}^n \eps_i\bbE_{\bY}[f(Y_i)]\bigg|\bigg]\\
	&\leq \bbE_{\bX,\bY,\boldsymbol{\eps}}\bigg[\sup_{f\in \calF}\bigg|\frac{1}{n}\sum_{i=1}^n \eps_i\big(f(X_i)-f(Y_i)\big)\bigg|\bigg]+  \bbE\bigg[\sup_{f\in \calF} \bigg|\frac{1}{n}\sum_{i=1}^n \eps_i\bbE_{\bY}[f(Y_i)]\bigg|\bigg]\\
	&= \bbE_{\bX,\bY}\bigg[\sup_{f\in \calF}\bigg|\frac{1}{n}\sum_{i=1}^n \big(f(X_i)-f(Y_i)\big)\bigg|\bigg]+  \bbE\bigg[\sup_{f\in \calF} \bigg|\frac{1}{n}\sum_{i=1}^n \eps_i\bbE_{\bY}[f(Y_i)]\bigg|\bigg] \label{maq1},
	\end{align} where \eqref{maq1} follows from Lemma \ref{aumat:lem}.
	
	Now, we have
	\begin{align}
	&\bbE_{\bX,\bY}\bigg[\sup_{f\in \calF}\bigg|\frac{1}{n}\sum_{i=1}^n \big(f(X_i)-f(Y_i)\big)\bigg|\bigg]\nn\\
	&\qquad =\bbE_{\bX,\bY}\bigg[\sup_{f\in \calF}\bigg|\frac{1}{n}\sum_{i=1}^n \big(f(X_i)-Pf\big)-\big(f(Y_i)-Pf\big)\big)\bigg|\bigg]\\
	&\qquad \leq \bbE_{\bX,\boldsymbol{\eps}}\bigg[\sup_{f\in \calF} \bigg|\frac{1}{n}\sum_{i=1}^n f(X_i)-Pf\bigg|\bigg]+\bbE_{\bY,\boldsymbol{\eps}}\bigg[\sup_{f\in \calF} \bigg|\frac{1}{n}\sum_{i=1}^n f(Y_i)-Pf\bigg|\bigg]\\
	&\qquad= 2\bbE_{\bX,\boldsymbol{\eps}}\bigg[\sup_{f\in \calF} \bigg|\frac{1}{n}\sum_{i=1}^n f(X_i)-Pf\bigg|\bigg]\\
	&\qquad= 2\bbE\big[\|P_n-P\|_{\calF}\big] \label{magut1}.
	\end{align}
	On the other hand, by using the triangle inequality, we also have 
	\begin{align}
	&\bbE\bigg[\sup_{f\in \calF} \bigg|\frac{1}{n}\sum_{i=1}^n \eps_i\bbE_{\bY}[f(Y_i)]\bigg|\bigg]\nn\\
	&\qquad \leq \bbE\bigg[\sup_{f\in \calF} \bigg|\frac{1}{n}\sum_{i=1}^n \eps_i\big(\bbE_{\bY}[f(Y_i)] -\bbE_{\pi}[f(Y)]\big)\bigg|\bigg]+ \bbE\bigg[\sup_{f\in \calF} \bigg|\frac{1}{n}\sum_{i=1}^n \eps_i \bbE_{\pi}[f(Y)]\bigg|\bigg] \label{magut2}.
	\end{align}
	Now, observe that
	\begin{align}
	\bbE\bigg[\sup_{f\in \calF} \bigg|\frac{1}{n}\sum_{i=1}^n \eps_i \bbE_{\pi}[f(Y)]\bigg|\bigg]&\leq \bigg(\sup_{f\in \calF} \big|\bbE_{\pi}[f(Y)]\big|\bigg)\bbE_{\boldsymbol{\eps}}\bigg[\sum_{i=1}^n \frac{1}{n}\sum_{i=1}^n \eps_i\bigg]\\
	&\leq M \bbE_{\boldsymbol{\eps}}\bigg[\sum_{i=1}^n \frac{1}{n}\sum_{i=1}^n \eps_i\bigg]\\
	&\leq M \sqrt{\bbE_{\boldsymbol{\eps}}\bigg[\bigg(\frac{1}{n}\sum_{i=1}^n \eps_i\bigg)^2\bigg]}\\
	&=\frac{M}{\sqrt{n}} \label{gu1}.
	\end{align}
	In addition, for each fixed $(\eps_1,\eps_2,\cdots,\eps_n) \in \{-1,+1\}^n$ and $f \in \calF$, we have
	\begin{align}
	&\bigg|\frac{1}{n}\sum_{i=1}^n \eps_i\big(\bbE_{\bY}[f(Y_i)] -\bbE_{\pi}[f(Y)]\big)\bigg|\nn\\
	&\qquad \leq \frac{1}{n}\bbE_{\bY}\bigg|\sum_{i=1}^n\eps_i\big( f(Y_i)-\bbE_{\pi}[f(Y)]\big)\bigg|\\
	&\qquad=\frac{1}{n}\bbE_{\bY}\big[\big|g_{\eps}(\bY)\big|\big] \label{g0},
	\end{align} where
	\begin{align}
	g_{\eps}(\by):=\sum_{i=1}^n\eps_i\big( f(y_i)-\bbE_{\pi}[f(Y)]\big), \qquad \forall \by \in \bbR^n.
	\end{align}
	Now, for all $\bx,\by \in \bb^n$, we have
	\begin{align}
	g_{\eps}(\bx)-g_{\eps}(\by)&=\sum_{i=1}^n\eps_i\big( f(x_i)-f(y_i)\big)\\
	&\leq \sum_{i=1}^n |f(x_i)-f(y_i)|\\
	&\leq \sum_{i=1}^n 2M \bone\{x_i\neq y_i\}.
	\end{align}
	Hence, by Lemma \ref{lem:HD}, we have
	\begin{align}
	\bbP_{\bY}\big[\big|g_{\eps}(\bY)\big|\geq M\sqrt{2 \tau_{\min} n \log n} \big]\leq \frac{2}{n}.
	\end{align}
	It follows that
	\begin{align}
	\bbE_{\bY}\big[\big|g_{\eps}(\bY)\big|\big]&\leq \bbE_{\bY}\big[\big|g_{\eps}(\bY)\big|\big|\big|g_{\eps}(\bY)\big|< M\sqrt{2 \tau_{\min} n \log n} \big]\nn\\
	&\qquad + 2nM \bbP_{\bY}\big[\big|g_{\eps}(\bY)\big|\geq M\sqrt{2 \tau_{\min} n \log n} \big] \label{K1}\\
	&\leq M\sqrt{2 \tau_{\min} n \log n}+ 2nM \bigg(\frac{2}{n}\bigg) \label{K2}\\
	&= M\sqrt{2 \tau_{\min} n \log n}+4M \label{g1},
	\end{align} where \eqref{K1} follows from the total law of expectation and $|g_{\eps}(\by)|\leq \sum_{i=1}^n \big|f(y_i)-\bbE_{\pi}[f(Y)]\big|\leq 2nM$ for all $\by \in \bbR^n$. 
	
	From \eqref{g0} and \eqref{g1}, we have
	\begin{align}
	\bigg|\frac{1}{n}\sum_{i=1}^n \eps_i\big(\bbE_{\bY}[f(Y_i)] -\bbE_{\pi}[f(Y)]\big)\bigg|\leq \frac{1}{n}\bigg(M\sqrt{2 \tau_{\min} n \log n}+4M\bigg) \label{x1}
	\end{align} for all  $f \in \calF, (\eps_1,\eps_2,\cdots,\eps_n) \in \{-1,+1\}^n$.
	
	From \eqref{x1}, we obtain
	\begin{align}
	\bbE_{\boldsymbol{\eps}}\bigg[\sup_{f\in \calF} \bigg|\frac{1}{n}\sum_{i=1}^n \eps_i\big(\bbE_{\bY}[f(Y_i)] -\bbE_{\pi}[f(Y)]\big)\bigg|\bigg]\leq \frac{1}{n}\bigg(M\sqrt{2 \tau_{\min} n \log n}+4M\bigg) \label{x2}.
	\end{align} 
	From \eqref{magut2}, \eqref{gu1}, and \eqref{x2}, we have
	\begin{align}
	&\bbE\bigg[\sup_{f\in \calF} \bigg|\frac{1}{n}\sum_{i=1}^n \eps_i\bbE_{\bY}[f(Y_i)]\bigg|\bigg]\nn\\
	&\qquad \leq \frac{1}{n}\bigg(M\sqrt{2 \tau_{\min} n \log n}+4M\bigg)+ \frac{M}{\sqrt{n}} \label{magut3}.
	\end{align}
	From \eqref{maq1}, \eqref{magut1}, and \eqref{magut3}, we obtain
	\begin{align}
	\bbE[\|P_n^0\|_{\calF}] \leq 2\bbE\big[\|P_n-P\|_{\calF}\big]+ \frac{1}{n}\bigg(M\sqrt{2 \tau_{\min} n \log n}+4M\bigg)+ \frac{M}{\sqrt{n}}, 
	\end{align} and we finally have
	\begin{align}
	\bbE\big[\|P_n-P\|_{\calF}\big]\geq \frac{1}{2}\bbE[\|P_n^0\|_{\calF}]- \tilA_n. 
	\end{align}
\end{proof}
\section{Proof of Theorem \ref{thm:main}} \label{proofthm:main}
The proof of Theorem \ref{thm:main} is based on \citep{Koltchinskii2002}. First, we prove \eqref{pet1}. Without loss of generality, we can assume that each $\varphi \in \Phi$ takes its values in $[0,1]$ (otherwise, it can be redefined as $\varphi \wedge 1$). Then, it is clear that $\varphi(x)=1$ for $x\leq 0$. Hence, for each fixed $\varphi \in \Phi$ and $f \in \calF$, we obtain
\begin{align}
P\{f\leq 0\}&\leq P\varphi(f)\\
&\leq P_n\varphi(f)+ \|P_n-P\|_{\calG_{\varphi}}, 
\end{align} where
\begin{align}
\calG_{\varphi}:=\big\{\varphi\cdot f: f \in \calF\big\}.
\end{align}
Now, let 
$g(\bx)=\sup_{f\in \calG_{\varphi}} \big|\frac{1}{n}\sum_{i=1}^n f(x_i)-Pf\big|$. Then, for all $\bx\neq \by$, we have
\begin{align}
\big|g(\bx)-g(\by)\big|&=\bigg|\sup_{f\in \calG_{\varphi}} \bigg|\frac{1}{n}\sum_{i=1}^n f(x_i)-Pf\bigg|-\sup_{f\in \calG_{\varphi}} \bigg|\frac{1}{n}\sum_{i=1}^n f(y_i)-Pf\bigg|\bigg|\\
&\leq \sup_{f\in \calG_{\varphi}} \bigg| \bigg|\frac{1}{n}\sum_{i=1}^n f(x_i)-Pf\bigg|- \bigg|\frac{1}{n}\sum_{i=1}^n f(y_i)-Pf\bigg|\bigg|\\
&\leq \sup_{f\in \calG_{\varphi}} \bigg| \bigg(\frac{1}{n}\sum_{i=1}^n f(x_i)-Pf\bigg)- \bigg(\frac{1}{n}\sum_{i=1}^n f(y_i)-Pf\bigg)\bigg|\\
&\leq \sup_{f\in \calG_{\varphi}}\bigg|\frac{1}{n}\sum_{i=1}^n f(x_i)-f(y_i)\bigg|\\
&\leq \sup_{f\in \calG_{\varphi}} \frac{1}{n}\sum_{i=1}^n \big|f(x_i)-f(y_i)\big|\\
&\leq \frac{1}{n}\sum_{i=1}^n \bone\{x_i\neq y_i\}.
\end{align} 
Hence, for $t>0$, by Lemma \ref{lem:HD} with we have
\begin{align}
\bbP\bigg\{\|P_n-P\|_{\calG_{\varphi}} \geq \bbE\big[\|P_n-P\|_{\calG_{\varphi}}\big]+
t\sqrt{\frac{\tau_{\min}}{n}}\bigg\}\leq 2\exp\big(-2t^2\big) \label{A1}. 
\end{align} 
Hence, with probability at least $1-2\exp\big(-2t^2\big)$ for all $f \in \calF$
\begin{align}
P\{f\leq 0\}\leq P_n \varphi(f)+ \bbE[\|P_n-P\|_{\calG_{\varphi}}]+ t\sqrt{\frac{\tau_{\min}}{n}} \label{A2}.
\end{align}
Now, by Lemma \ref{lem:key} with $M=1$ (since $\sup_{f\in \calG_{\varphi}}\|f\|_{\infty}=1$), it holds that
\begin{align}
\bbE\bigg[\big\|P_n-P\big\|_{\calG_{\varphi}}\bigg]&\leq  2\bbE\big[\|P_n^0\big\|_{\calG_{\varphi}} \big] + A_n\big|_{M=1}\\
&=2\bbE\bigg[\bigg\|n^{-1}\sum_{i=1}^n\eps_i \delta_{X_i}\bigg\|_{\calG_{\varphi}}\bigg]+ B_n
\label{A3}.
\end{align}
Since $(\varphi-1)/L(\varphi)$ is contractive and $\varphi(0)-1=0$, by using Talagrand's contraction lemma \citep{LedouxT1991book, Truong2022OnRC}, we obtain 
\begin{align}
\bbE_{\eps}\bigg\|n^{-1}\sum_{i=1}^n \eps_i \delta_{X_i}\bigg\|_{\calG_{\varphi}} &\leq 2L(\varphi)\bbE_{\eps}\bigg\|n^{-1}\sum_{i=1}^n \eps_i \delta_{X_i}\bigg\|_{\calF} \\
&=2L(\varphi) R_n(\calF) \label{A4}.
\end{align}
From \eqref{A1}, \eqref{A2}, \eqref{A3}, and \eqref{A4}, with probability $1-2\exp\big(-2t^2\big)$,  we have for all $f \in \calF$,
\begin{align}
P\{f\leq 0\}\leq P_n\varphi(f)+ 4 L(\varphi) R_n(\calF)+ t\sqrt{\frac{\tau_{\min}}{n}}+B_n \label{A5}.
\end{align}
Now, we use \eqref{A5} with $\varphi=\varphi_k$ and $t$ is replaced by $t+\sqrt{\log k}$ to obtain
\begin{align}
&\bbP\bigg(\exists f\in \calF: P\{f\leq 0\}> \inf_{k>0}\bigg[P_n\varphi_k(f)+ 4 L(\varphi_k) R_n(\calF)+ \big(t+\sqrt{\log k}\big)\sqrt{\frac{\tau_{\min}}{n}}+ B_n\bigg]\bigg)\nn\\
&\qquad \leq 2\sum_{k=1}^{\infty}\exp\big(-2\big(t+\sqrt{\log k}\big)^2\big)\\
&\qquad \leq 2 \sum_{k=1}^{\infty} k^{-2}\exp\big(-2t^2\big)\\
&\qquad=\frac{\pi^2}{3}\exp\big(-2t^2\big) \label{asp},
\end{align} where \eqref{asp} follows from
\begin{align}
\frac{\pi^2}{6}=\sum_{k=1}^{\infty} k^{-2}.
\end{align}
Next, we prove \eqref{pet2}. By the equivalence of Rademacher and Gaussian complexity \citep{VaartWellnerbook}, we have
\begin{align}
\bbE\bigg\|n^{-1}\sum_{i=1}^n \eps_i \delta_{X_i}\bigg\|_{\calG_{\varphi}}\leq \sqrt{\frac{\pi}{2}}\bbE\bigg\|n^{-1}\sum_{i=1}^n g_i \delta_{X_i}\bigg\|_{\calG_{\varphi}} \label{A6}.
\end{align}
Hence, from \eqref{A3} and \eqref{A6}, we obtain
\begin{align}
\bbE\bigg[\big\|P_n-P\big\|_{\calG_{\varphi}}\bigg]
&\leq \sqrt{2\pi} \bbE\bigg\|n^{-1}\sum_{i=1}^n g_i \delta_{X_i}\bigg\|_{\calG_{\varphi}}+ B_n
\label{A7}.
\end{align}
Now, define Gaussian processes
\begin{align}
Z_1(f,\sigma):=\sigma n^{-1/2}\sum_{i=1}^n g_i(\varphi \circ f)(X_i),
\end{align}
and
\begin{align}
Z_2(f,\sigma):=L(\varphi)n^{-1/2}\sum_{i=1}^n g_i f(X_i)+\sigma g,
\end{align}
where $\sigma=\pm 1$ and $g$ is standard normal independent of the sequence $\{g_i\}$. Let $\bbE_g$ be the expectation on the probability space $(\Omega_g,\Sigma_g,\bbP_g)$ on which the sequence $\{g_i\}$ and $g$ are defined, then by \citep{LedouxT1991book, Koltchinskii2002}, we have
\begin{align}
\bbE_g\bigg[\sup\{Z_1(f,\sigma): f\in \calF,\sigma=\pm 1\}\bigg]\leq \bbE_g\bigg[\sup\{Z_2(f,\sigma): f\in \calF, \sigma=\pm 1\}\bigg] \label{A8}.
\end{align}
On the other hand, it holds that
\begin{align}
\bbE_g\bigg\|n^{-1/2}\sum_{i=1}^n g_i\delta_{X_i}\bigg\|_{\calG_{\varphi}}&=\bbE_g\bigg[n^{-1/2}\sup_{h \in \bar{\calG}_{\varphi}}\sum_{i=1}^n g_i h(X_i)\bigg]\\
&=\bbE_g \bigg[\sup\{Z_1(f,\sigma):f\in \calF,\sigma =\pm 1\}\bigg] \label{ma1},
\end{align} where $\bar{\calG}_{\varphi}:=\{\varphi(f),-\varphi(f): f\in \calF\}$, and similarly
\begin{align}
L(\varphi)\bbE_g\bigg\|n^{-1/2}\sum_{i=1}^n g_i \delta_{X_i}\bigg\|_{\calF}+\bbE|g|\geq \bbE_g \bigg[\sup\{Z_2(f,\sigma): f \in \calF,\sigma=\pm 1\}\bigg] \label{ma2}.
\end{align}
From \eqref{A8}, \eqref{ma1}, and \eqref{ma2}, we have
\begin{align}
\bbE_g\bigg\|n^{-1}\sum_{i=1}^n g_i \delta_{X_i}\bigg\|_{\calG_{\varphi}}\leq L(\varphi)\bbE_g\bigg\|n^{-1}\sum_{i=1}^n g_i \delta_{X_i}\bigg\|_{\calF}+n^{-1/2}\bbE|g| \label{A9}.
\end{align}
By combining \eqref{A7} and \eqref{A9}, we obtain
\begin{align}
\bbE\big[\big\|P_n-P\big\|_{\calG_{\varphi}}\big] &\qquad \leq \sqrt{2\pi} \bigg(L(\varphi)\bbE\bigg[\bigg\|n^{-1}\sum_{i=1}^n g_i \delta_{X_i}\bigg\|_{\calF}\bigg]+n^{-1/2}\bbE|g| \bigg)+ B_n\label{A10}.
\end{align}
Hence, from \eqref{A2}, \eqref{A7}, and \eqref{A10}, we finally obtain \eqref{pet2}.
\section{Proof of Theorem \ref{cor:cor1}} \label{proofcor:cor1}
We can assume, without loss of generality, that the range of $\varphi$ is $[0,1]$ (otherwise, we  can replace $\varphi$ by $\varphi \wedge 1$). Let $\delta_k=2^{-k}$ for all $k\geq 0$.  In addition, set $\Phi=\{\varphi_k:k\geq 1\}$, where
\begin{align}
\varphi_k(x):=\begin{cases} \varphi(x/\delta_k), & x\geq 0,\\ \varphi(x/\delta_{k-1}),& x<0\end{cases}.
\end{align}
Now, for any $\delta \in (0,1]$, there exists $k$ such that $\delta \in (\delta_k,\delta_{k-1}]$. Hence, if $f(X_i)\geq 0$, it holds that $f(X_i)/\delta_k \geq f(X_i)/\delta$, so we have
\begin{align}
\varphi_k(f(X_i))&=\varphi\bigg(\frac{f(X_i)}{\delta_k}\bigg)\\
&\leq \varphi\bigg(\frac{f(X_i)}{\delta}\bigg) \label{b1},
\end{align} where \eqref{b1} follows from the fact that $\varphi(\cdot)$ is non-increasing.

On the other hand, if $f(X_i)<0$, then $f(X_i)/\delta_{k-1} \geq f(X_i)/\delta$. Hence, we have
\begin{align}
\varphi_k(f(X_i))&=\varphi\bigg(\frac{f(X_i)}{\delta_{k-1}}\bigg)\\
&\leq \varphi\bigg(\frac{f(X_i)}{\delta}\bigg) \label{b2},
\end{align} where \eqref{b1} follows from the fact that $\varphi(\cdot)$ is non-increasing. 

From \eqref{b1} and \eqref{b2}, we have
\begin{align}
P_n\varphi_k\big(f\big)&=\frac{1}{n}\sum_{i=1}^n \varphi_k(f(X_i))\\
&\leq \frac{1}{n}\sum_{i=1}^n \varphi\bigg(\frac{f(X_i)}{\delta}\bigg)\\
&= P_n\varphi\bigg(\frac{f}{\delta}\bigg) \label{b3}. 
\end{align}
Moreover, we also have
\begin{align}
\frac{1}{\delta_k}&\leq \frac{2}{\delta},
\end{align}
and
\begin{align}
\log k=\log \log_2 \frac{1}{\delta_k}\leq \log \log_2 2 \delta^{-1} \label{atm}.
\end{align}
Furthermore, observe that
\begin{align}
L(\varphi_k)&=\sup_{x\in \bbR}\bigg|\frac{d \varphi_k(x)}{dx}\bigg|\\
&=\sup_{x\in \bbR}\bigg|\frac{d \varphi(x/\delta_k)}{dx}\bigg|\bone\{x\geq 0\}+\bigg|\frac{d \varphi(x/\delta_{k-1})}{dx}\bigg|\bone\{x< 0\}\\
&\leq \frac{L(\varphi)}{\min\{\delta_k,\delta_{k-1}\}} \\
&=\frac{ L(\varphi)}{\delta_k}\\
&\leq \frac{2}{\delta} L(\varphi).
\end{align}
By combining the above facts and using Theorem \ref{thm:main}, we obtain \eqref{pet1a} and \eqref{pet2a}.
\section{Proof of Theorem \ref{thm:main3}} \label{markb}
Let $\eps>0$ and $\delta>0$. Define recursively
\begin{align}
r_0:=1, \qquad r_{k+1}=C\sqrt{r_k\eps}\wedge 1, \qquad  \gamma_k:=\sqrt{\frac{\eps}{r_k}}
\end{align}
some sufficiently large constant $C>1$ (to be determined later) such that $\eps<C^{-4}$. Denote by
\begin{align}
\delta_0&:=\delta,\\ \delta_k&:=\delta(1-\gamma_0-\cdots-\gamma_{k-1}),\\
\delta_{k,\frac{1}{2}}&:=\frac{1}{2}\big(\delta_k+\delta_{k+1}\big), \qquad k\geq 1.	
\end{align}		

For $k\geq 0$, let $\varphi_k$ be a continuous function from $\bbR$ into $[0,1]$ such that $\varphi_k(u)=1$ for $u\leq \delta_{k,\frac{1}{2}}, \varphi_k(u)=0$ for $u\geq \delta_k$, and linear $\delta_{k,\frac{1}{2}}\leq u \leq \delta_k$. For $k\geq 1$ let $\varphi'_k$ be a continuous function from $\bbR$ into $[0,1]$ such that $\varphi'_k(u)=1$ for $u\leq \delta_k$, $\varphi'_k(u)=0$ for $u\geq \delta_{k-1,\frac{1}{2}}$, and linear for $\delta_k\leq u\leq \delta_{k-1,\frac{1}{2}}$.

To begin with, we prove the following lemma.
\begin{lemma} 	
	Define $\calF_0:=\calF$, and further recursively
	\begin{align}
	\calF_{k+1}:=\bigg\{f\in \calF_k: P\{f\leq \delta_{k,\frac{1}{2}}\}\leq \frac{r_{k+1}}{2}\bigg\}.
	\end{align}	
	For all $k\geq 1$, define
	\begin{align}
	\calG_k:=\big\{\varphi_k\circ f: f \in \calF_k \big\}, \qquad k\geq 0 \label{defGk}
	\end{align}
	and
	\begin{align}
	\calG'_k:=\big\{\varphi'_k\circ f: f \in \calF_k \big\}, \qquad k\geq 1 \label{defGprimek}.
	\end{align}
	Assume that
	\begin{align}
	E^{(k)}&:=\bigg\{\|P_n-P\|_{\calG_{k-1}}\leq  \bbE\|P_n-P\|_{\calG_{k-1}} + \big(K_2 \sqrt{r_{k-1}\eps}+K_3 \eps\big)\sqrt{\tau_{\min}} \bigg\}\nn\\
	&\qquad  \cap \bigg\{\|P_n-P\|_{\calG'_k}\leq  \bbE\|P_n-P\|_{\calG'_k} + \big(K_2 \sqrt{r_k\eps} +K_3\eps\big)\sqrt{\tau_{\min}} \bigg\}, \qquad k\geq 1 \label{defEk},
	\end{align}
	and
	\begin{align}
	E_N:=\bigcap_{k=1}^N E^{(k)}, \qquad N\geq 1 \label{defeventEN}.
	\end{align}
	Then, it holds that
	\begin{align}
	\bbP\big[E_N^c\big] \leq 4N \exp\bigg(-\frac{n\eps^2}{2}\bigg) \label{G5a}.
	\end{align}
\end{lemma}
\begin{proof}
	The proof is based on \citep[Proof of Theorem 5]{Koltchinskii2002}.
	
	For the case $C\sqrt{\eps}\geq 1$, by a simple induction argument, we have $r_k=1$. Now, without loss of generality, we assume that $C\sqrt{\eps}<1$. For this case, we have
	\begin{align}
	r_k&=C^{1+2^{-1}+\cdots+ 2^{-(k-1)}}\eps^{2^{-1}+\cdots+ 2^{-(k-1)}}\\
	&=C^{2(1-2^{-k})} \eps^{1-2^{-k}}\\
	&=(C\sqrt{\eps})^{2(1-2^{-k})} \label{defrk}.
	\end{align}
	From \eqref{defrk}, it is easy to see that $r_{k+1}<r_k$ for any $k\geq 0$.
	
	Now, observe that
	\begin{align}
	\sum_{i=0}^k\gamma_i&=C^{-1}\bigg[C\sqrt{\eps}+(C\sqrt{\eps})^{2^{-1}}+ \cdots+ (C\sqrt{\eps})^{2^{-k}}\bigg] \\
	&= C^{-1} \bigg[(C\sqrt{\eps})^{2^{-k}}+((C\sqrt{\eps})^{2^{-k}})^2+((C\sqrt{\eps})^{2^{-k}})^{2^2} \cdots+ ((C\sqrt{\eps})^{2^{-k}})^{2^k}\bigg] \\
	&\leq C^{-1} \bigg[(C\sqrt{\eps})^{2^{-k}}+((C\sqrt{\eps})^{2^{-k}})^2+((C\sqrt{\eps})^{2^{-k}})^3 \cdots+ ((C\sqrt{\eps})^{2^{-k}})^k \bigg] \label{ubi}\\
	&\leq C^{-1} \sum_{i=1}^{\infty} \big((C\sqrt{\eps})^{2^{-k}}\big)^i \label{ubi2}\\
	&\leq C^{-1}(C\sqrt{\eps})^{2^{-k}}\big(1-(C\sqrt{\eps})^{2^{-k}}\big)^{-1}\leq \frac{1}{2},
	\end{align} for $\eps\leq C^{-4}, C>2(2^{\frac{1}{4}}-1)^{-1}$ and $k\leq \log_2 \log_2 \eps^{-1}$, where \eqref{ubi} follows from $i+1 \leq 2^i$ for all $i\geq 1$ and $C\sqrt{\eps}<1$. Hence, for small enough $\eps$ (note that our choice of $\eps\leq C^{-4}$ implies $C\sqrt{\eps}<1$), we have
	\begin{align}
	\gamma_0+\gamma_1+\cdots+\gamma_k\leq \frac{1}{2}, \qquad k\geq 1.
	\end{align}
	Therefore, for all $k\geq 1$, we get $\delta_k\in (\delta/2,\delta)$. Note also that below our choice of $k$ will be such that the restriction $k\leq \log_2 \log_2 \eps^{-1}$ for any fixed $\eps>0$ will always be fulfilled.
	
	From the definitions of \eqref{defGk} and \eqref{defGprimek}, for $k\geq 1$, we have
	\begin{align}
	\sup_{g \in \calG_k} Pg^2 \leq \sup_{f \in \calF_k} P\{f\leq \delta_k\}\leq \sup_{f \in \calF_k} P\{f\leq \delta_{k-1,\frac{1}{2}}\}\leq \frac{r_k}{2}\leq r_k \label{fact0},
	\end{align}
	and
	\begin{align}
	\sup_{g \in \calG'_k} Pg^2 \leq \sup_{f \in \calF_k} P\{f\leq \delta_{k-1,\frac{1}{2}}\}\leq \frac{r_k}{2}\leq r_k \label{fact1}.
	\end{align}
	Since $r_0=1$, it is easy to see that \eqref{fact0} and \eqref{fact1} trivially holds at $k=0$. 
	
	Now, by the union bound, from \eqref{defEk}, we have
	\begin{align}
	\bbP\big[\big(E^{(k)}\big)^c\big]&\leq \bbP\bigg[\|P_n-P\|_{\calG_{k-1}}> \bbE\|P_n-P\|_{\calG_{k-1}} + K_2 \sqrt{r_{k-1}\eps}+K_3\eps \bigg]\nn\\
	&\qquad + \bbP\bigg[\|P_n-P\|_{\calG'_{k-1}}>  \bbE\|P_n-P\|_{\calG'_{k-1}} + K_2 \sqrt{r_{k-1}\eps}+K_3 \eps\bigg] \label{T1}.
	\end{align}
	In addition, by similar arguments which leads to \eqref{A1}, we have
	\begin{align}
	\bbP\bigg\{\|P_n-P\|_{\calG_{k-1}} \geq \bbE\big[\|P_n-P\|_{\calG_{k-1}}\big]+u\sqrt{\frac{\tau_{\min}}{n}}\bigg\}\leq 2\exp\big(-2u^2\big) \label{B11}. 
	\end{align} 
	and
	\begin{align}
	\bbP\bigg\{\|P_n-P\|_{\calG'_{k-1}} \geq \bbE\big[\|P_n-P\|_{\calG'_{k-1}}\big]+u\sqrt{\frac{\tau_{\min}}{n}}\bigg\}\leq 2\exp\big(-2u^2\big)  \label{B12}. 
	\end{align} 
	By replacing $u=\big(K_2 \sqrt{r_{k-1}\eps }+K_3\eps\big)\sqrt{n}$ to  \eqref{B11}  and \eqref{B12} for $K_2> 0$ and $K_3>0$, we obtain
	\begin{align}
	&\bbP\bigg\{\|P_n-P\|_{\calG_{k-1}} \geq \bbE\big[\|P_n-P\|_{\calG_{k-1}}\big]+\big(K_2 \sqrt{r_{k-1}\eps}+K_3\eps\big)\sqrt{\tau_{\min}}\bigg\}\nn\\
	&\qquad \leq 2\exp\bigg(-2n(K_2 \sqrt{r_{k-1}\eps}+K_3)^2\bigg) \label{C11}
	\end{align}
	and
	\begin{align}
	&\bbP\bigg\{\|P_n-P\|_{\calG_{k-1}} \geq \bbE\big[\|P_n-P\|_{\calG'_{k-1}}\big]+\big(K_2 \sqrt{r_{k-1}\eps}+K_3\eps\big)\sqrt{\tau_{\min}}\bigg\}\nn\\
	&\qquad \leq 2\exp\bigg(-2n(K_2 \sqrt{r_{k-1}\eps}+K_3\eps )^2\bigg) \label{C12}.
	\end{align}
	Now, since $0<C\sqrt{\eps}\leq 1$, by \eqref{defrk}, we have
	\begin{align}
	r_{k-1}=(C\sqrt{\eps})^{2(1-2^{-(k-1)})}\geq C^2\eps  \label{defrkb}.
	\end{align}
	Hence, from \eqref{T1}, \eqref{C11}, \eqref{C12}, and \eqref{defrkb}, that
	\begin{align}
	\bbP\big[\big(E^{(k)}\big)^c\big]&\leq 4\exp\bigg(-2n(K_2 \sqrt{r_{k-1}\eps}+K_3\eps)^2\bigg)\\
	&\qquad \leq  4\exp\bigg(-2n(K_2 C+K_3)^2\eps^2\bigg)\\
	&\qquad \leq 4 \exp\bigg(-\frac{n\eps^2}{2}\bigg) \label{asto}, \qquad \forall k\geq 1
	\end{align} if we choose $K_2$ and $K_3$ such that
	\begin{align}
	K_2 C+K_3 \geq \frac{1}{2} \label{aslico}.
	\end{align} 
	Then, by the union bound and \eqref{asto}, we have 
	\begin{align}
	\bbP\big[E_N^c\big]\leq 4N \exp\bigg(-\frac{n\eps^2}{2}\bigg) \label{G5}.
	\end{align}
\end{proof}
\begin{lemma}\label{lem:aux} Let $\eps>0$ and $0<\alpha<2$ such that
	\begin{align}
	\eps \geq \bigg(\frac{1}{n\delta^{\alpha}}\bigg)^{\frac{2}{2+\alpha}}\vee \sqrt{\frac{2\log n}{n}} \vee B_n\label{condkeyx}
	\end{align} for all large enough $n$, where $B_n$ is defined in \eqref{defBtn}. Denote by $\calL:=\big\{f \in \calF: P_n\{f\leq \delta\}\leq \eps \}$. Then, on the event $E_N \cap \calL$, we have: 
	\begin{itemize}
		\item (i) $\sup_{f \in \calF_k} P_n\{f\leq \delta_k\}\leq r_k$, \qquad $0\leq k\leq N$
	\end{itemize}
	and
	\begin{itemize}
		\item (ii) $\forall f \in \calL$ \quad  $P_n\{f\leq \delta\}\leq \eps \qquad \Rightarrow  f \in \calF_N$	
	\end{itemize}
	for all positive integer $N$ satisfying
	\begin{align}
	N\leq \frac{1}{\eta}\log_2\log_2 \eps^{-1} \qquad \mbox{and}\qquad r_N\geq \eps \label{cond33},
	\end{align} and $\eta$ is some implicit positive constant.
\end{lemma}
\begin{proof} We will use induction with respect to $N$. For $N=0$, the statement is obvious. Suppose it holds for some $N\geq 0$, such that $N+1$ still satisfies condition \eqref{cond33} of the lemma. Then, on the event $E_N\cap \calL$ we have
	\begin{align}
	(i)\qquad \sup_{f\in \calF_k} P_n\{f\leq \delta_k\}\leq r_k, \qquad 0\leq k\leq N
	\end{align}
	and
	\begin{align}
	(ii)\qquad \forall f\in \calF \quad P_n\{f\leq \delta\}\leq \eps \quad \Rightarrow f \in \calF_N.
	\end{align}
	Suppose now that $f \in \calF$ is such that $P_n\{f\leq \delta\}\leq \eps$. By the induction assumptions, on the event $E_N$ defined in \eqref{defeventEN}, we have $f\in \calF_N$. Because of this, we obtain on the event $E_{N+1}$
	\begin{align}
	P\{f\leq \delta_{N,\frac{1}{2}}\}&=P_n\{f\leq \delta_{N,\frac{1}{2}}\}+(P-P_n)\{f\leq \delta_{N,\frac{1}{2}}\}\\
	&\leq P_n\{f\leq \delta_N \}  +(P-P_n)\{f\leq \delta_{N,\frac{1}{2}}\}\\
	&\leq  P_n\{f\leq \delta_N \} + (P-P_n)(\varphi_N(f))\\
	&\leq P_n\{f\leq \delta_N \} + \|P-P_n\|_{\calG_N}\\
	&\leq \eps+ \bbE\|P_n-P\|_{\calG_N} +\big(K_2 \sqrt{r_N \eps}+K_3 \eps\big)\sqrt{\tau_{\min}} \label{eq60}.
	\end{align}
	For the class $\calG_N$, define
	\begin{align}
	\hatR_n(\calG_N):=\bigg\|n^{-1}\sum_{i=1}^n \eps_i \delta_{X_i}\bigg\|_{\calG_N},
	\end{align} where $\eps_i$ is a sequence of i.i.d. Rademacher random variables. By Lemma \ref{lem:key}, it holds that
	\begin{align}
	\bbE\big[\big\|P_n-P\big\|_{\calG_N}\big]&\leq 2 \bbE\big[\|P_n^0\|_{\calG_N} \big]+ B_n\\
	&=2\bbE\big[\hatR_n(\calG_N)\big]+ B_n\label{D1}.
	\end{align}
	From \eqref{D1}, we have
	\begin{align}
	&\bbE\big[\big\|P_n-P\big\|_{\calG_N}\big]\nn\\
	&\qquad \leq 2\bbE\big[\hatR_n(\calG_N)\big]+B_n \label{E1a}\\
	&\qquad = 2\bbE\big[\bone\{E_N\}\bbE_{\eps}\big[\hatR_n(\calG_N)\big]\big]+2\bbE\big[\bone\{E_N^c\}\bbE_{\eps}\big[\hatR_n(\calG_N)\big]\big]+B_n\label{E1}.
	\end{align}
	Next, by the well-known entropy inequalities for subgaussian process \citep{VaartWellnerbook}, we have
	\begin{align}
	\bbE_{\eps}\big[\hatR_n(\calG_N)\big]&\leq \inf_{g \in \calG_N}\bbE_{\eps} \bigg|n^{-1}\sum_{j=1}^n \eps_j g(X_j)\bigg|\nn\\
	&\qquad + \frac{\eta}{\sqrt{n}}\int_0^{(2\sup_{g\in \calG_N} P_n g^2)^{1/2}}H_{d_{P_n,2}}^{1/2}(\calG_N;u)du \label{E2}
	\end{align} for some implicit positive constant $\eta>0$.
	
	By the induction assumption, on the event $E_N \cap \calL$,
	\begin{align}
	\inf_{g \in \calG_N}\bbE_{\eps} \bigg|n^{-1}\sum_{j=1}^n \eps_j g(X_j)\bigg|&\leq \inf_{g \in \calG_N}\sqrt{\bbE_{\eps} \bigg|n^{-1}\sum_{j=1}^n \eps_j g(X_j)\bigg|^2}\\
	&\leq \frac{1}{\sqrt{n}}\inf_{g \in \calG_N}\sqrt{P_n g^2}\\
	&\leq \frac{1}{\sqrt{n}}\inf_{f\in \calF_N} \sqrt{P_n\{f\leq \delta_N\}}\\
	&\leq \sqrt{\frac{\eps}{n}} \label{Ee3}\\
	&\leq \eps \label{E3},
	\end{align} where \eqref{Ee3} follows from $\inf_{f\in \calF_N} P_n\{f\leq \delta_N\}\leq \inf_{f\in \calF_N} P_n\{f\leq \delta\}\leq P_n\{f\leq \delta\} \leq \eps$ by the induction assumption with $f \in \calF_N$.
	
	We also have on the event $E_N \cap \calL$, by \eqref{fact0}, it holds that
	\begin{align}
	\sup_{g \in \calG_N} P_n g^2 \leq \sup_{f \in \calF_N} P_n\{f\leq \delta_N\}\leq r_N.
	\end{align}
	The Lipschitz norm of $\varphi_{k-1}$ and $\varphi'_k$ is bounded by
	\begin{align}
	L=2(\delta_{k-1}-\delta_k)^{-1}=2\delta^{-1}\gamma_{k-1}^{-1}=\frac{2}{\delta}\sqrt{\frac{r_{k-1}}{\eps}}
	\end{align} which implies the following bound on the distance:
	\begin{align}
	d_{P_n,2}^2(\varphi_N\circ f;\varphi_N\circ g)&=n^{-1}\sum_{j=1}^n \big|\varphi_N(f(X_j))-\varphi_N(g(X_j))\big|^2\\
	&\leq \bigg(\frac{2}{\delta}\sqrt{\frac{r_N}{\eps}}\bigg)^2 d_{P_n,2}^2(f,g).
	\end{align}
	Therefore, on the event $\calE_N\cap \calL$,
	\begin{align}
	&\frac{1}{\sqrt{n}}\int_0^{(2\sup_{g\in \calG_N} P_n g^2)^{1/2}} H_{d_{P_n},2}^{1/2}\big(\calG_N;u\big)du\nn\\
	&\qquad \leq \frac{1}{\sqrt{n}}\int_0^{(2r_N)^{1/2}} H_{d_{P_n},2}^{1/2}\bigg(\calF;\frac{\delta \sqrt{\eps} u}{2 \sqrt{r_N}}\bigg)du\\
	&\qquad \leq \bigg( \frac{2\sqrt{D}}{1-\alpha/2}\bigg)\bigg(\frac{r_N}{\eps}\bigg)^{\alpha/4}\frac{r_N^{1/2-\alpha/4}}{\sqrt{n}\delta^{\alpha/2}}\\
	&\qquad \leq \bigg( \frac{2^{1/2+\alpha/4} \sqrt{D}}{1-\alpha/2}\bigg) \frac{r_N^{1/2}}{\eps^{\alpha/4}}\eps^{\frac{2+\alpha}{4}}\\
	&\qquad =\bigg( \frac{2^{1/2+\alpha/4}\sqrt{D}}{1-\alpha/2}\bigg) \sqrt{r_N\eps} \label{E8},
	\end{align}
	where \eqref{E8} follows from the condition \eqref{condkeyx}, which implies that
	\begin{align}
	\frac{1}{n^{1/2}\delta^{\alpha/2}}\leq \eps^{\frac{2+\alpha}{4}}.
	\end{align}
	From \eqref{E2} and \eqref{E8}, we obtain that on the event $E_N\cap \calL$,
	\begin{align}
	\bbE_{\eps}\big[\hatR_n(\calG_N)\big] \leq \eps+  \frac{2^{1/2+\alpha/4}\eta \sqrt{D}}{1-\alpha/2} \sqrt{r_N\eps} \label{est1}.
	\end{align}
	On the other hand, we also have
	\begin{align}
	\bbE_{\eps} \big[\hatR_n(\calG_N)\big] \leq 1 \label{est2}.
	\end{align}
	Hence, by combining with \eqref{est1} and \eqref{est2}, from \eqref{E1}, we obtain 
	\begin{align}
	\bbE\big[\big(\big\|P_n-P\big\|_{\calG_N}\big)\big]&= 2\bbE\big[\bone\{E_N\}\bbE_{\eps}\big[\hatR_n(\calG_N)\big]\big]+2\bbE\big[\bone\{E_N^c\}\bbE_{\eps}\big[\hatR_n(\calG_N)\big]\big]+B_n\\
	&\qquad\leq 2\bigg[\eps+ \eta \bigg( \frac{2^{1/2+\alpha/4}\eta \sqrt{D}}{1-\alpha/2}\bigg)\sqrt{r_N \eps} \bigg] + 8N \exp\bigg(-\frac{n\eps^2}{2}\bigg) +B_n \label{E7}.
	\end{align}
	Now, by the condition \eqref{condkeyx}, it holds that
	\begin{align}
	\eps \geq B_n \label{E9},
	\end{align}
	and
	\begin{align}
	8\eta N\exp\bigg(-\frac{n\eps^2}{2}\bigg)&\leq  \frac{8 \eta N}{n} \label{eq262}\\
	&\leq \frac{8 \log_2 \log_2 \eps^{-1}}{n} \label{esmo}\\
	&\leq \frac{8 \log_2 \log_2 \sqrt{\frac{n}{2\log n}}}{n} \label{eq264}\\
	&\leq \eta \sqrt{\frac{2\log n}{n}} \label{aska}\\
	&\leq \eta \eps \label{aska2}, 
	\end{align} for $n$ sufficiently large, where \eqref{eq262} and \eqref{eq264} follows from $\eps\geq \sqrt{\frac{2\log n}{n}}$, \eqref{esmo} follows from \eqref{cond33}, and \eqref{aska} holds for $n$ sufficiently large.
	
	From \eqref{E7}, \eqref{E9}, and \eqref{aska2}, it holds that
	\begin{align}
	\bbE\big[\big\|P_n-P\big\|_{\calG_N}\big] &\leq  2\bigg[\eps+ \eta \bigg( \frac{2^{1/2+\alpha/4}\sqrt{D}}{1-\alpha/2}\bigg)\sqrt{r_N \eps}\bigg]  + 2\eps\\
	&=4\eps+ 2\eta \bigg( \frac{2^{1/2+\alpha/4}\sqrt{D}}{1-\alpha/2}\bigg)\sqrt{r_N \eps} 
	\label{E10a}. 
	\end{align}
	In addition, we have
	\begin{align}
	r_N=(C\sqrt{\eps})^{2(1-2^{-N})}\geq C^2 \eps, 
	\end{align} or
	\begin{align}
	\eps \leq \frac{r_N}{C^2} \label{E10}.
	\end{align}
	Hence, from \eqref{E10a} and \eqref{E10}, we conclude that with some constant $\tilde{\eta}>0$,
	\begin{align}
	\bbE\big[\big\|P_n-P\big\|_{\calG_N}\big] \leq \tilde{\eta}\sqrt{r_N \eps} \label{E12}.
	\end{align}
	From \eqref{eq60} and \eqref{E12}, on the event $E_{N+1}\cap \calL$, we have
	\begin{align}
	P\{f\leq \delta_{N,\frac{1}{2}}\} &\leq \eps+ \tilde{\eta}\sqrt{r_N \eps}  +\big(K_2 \sqrt{r_N \eps}+K_3 \eps\big)\sqrt{\tau_{\min}} \label{eq70}\\
	&\leq \frac{1}{2} C \sqrt{r_N \eps} \\
	&=r_{N+1}/2 \label{eq71},
	\end{align} by a proper choice of the constant $C>0$, where \eqref{eq71} follows from \eqref{E10}. This means that $f \in \calF_{N+1}$ and the induction step for (ii) is proved.
	
	Now, we prove (i). We have on the event $E_{N+1}$,
	\begin{align}
	\sup_{f\in \calF_{N+1}}P_n\{f\leq \delta_{N+1}\}&\leq \sup_{f\in \calF_{N+1}}P\{f\leq \delta_{N+1}\}+\sup_{f\in \calF_{N+1}} (P_n-P)\{f\leq \delta_{N+1}\}\\
	&\leq \sup_{f\in \calF_{N+1}}P\{f\leq \delta_{N+1}\}+\|P_n-P\|_{\calG'_{N+1}}\\ 
	&\leq r_{N+1}/2 + \bbE\|P_n-P\|_{\calG'_{N+1}}+\big(K_2 \sqrt{r_{N+1}\eps}+K_3 \eps\big)\sqrt{\tau_{\min}} \label{E11}.
	\end{align}
	By Lemma \ref{lem:key}, we have 
	\begin{align}
	&\bbE\big\|P_n-P\big\|_{\calG'_{N+1}}\nn\\
	&\qquad \leq 2\bbE\big[\bone\{E_N\}\bbE_{\eps}\big[\hatR_n(\calG'_{N+1})\big]\big]+2\bbE\big[\bone\{E_N^c\}\bbE_{\eps}\big[\hatR_n(\calG'_{N+1})\big]\big]+ B_n\\
	&\qquad \leq 2\bbE\big[\bone\{E_N\}\bbE_{\eps}\big[\hatR_n(\calG'_{N+1})\big]\big]+2\bbE\big[\bone\{E_N^c\}\bbE_{\eps}\big[\hatR_n(\calG'_{N+1})\big]\big]+\eps \label{E1c},
	\end{align} where \eqref{E1c} follows from the condition \eqref{condkeyx}.
	
	As above, we have
	\begin{align}
	\bbE_{\eps}\big[\hatR_n(\calG'_{N+1})\big] &\leq \inf_{g \in \calG'_{N+1}} \bbE_{\eps}\bigg|n^{-1}\sum_{j=1}^n \eps_j g(X_j)\bigg|\nn\\
	&\qquad \qquad+ \frac{\eta}{\sqrt{n}} \int_0^{(2\sup_{g\in \calG'_{N+1}} P_n g^2)^{1/2}} H_{d_{P_n,2}}^{1/2}\big(\calG'_{N+1};u\big)du.
	\end{align}
	Since we already proved (i), it implies that on the event $E_{N+1}\cap \calL$,
	\begin{align}
	\inf_{g \in \calG'_{N+1}} \bbE_{\eps}\bigg|n^{-1}\sum_{j=1}^n \eps_j g(X_j)\bigg|& \leq \inf_{g \in \calG'_{N+1}}\sqrt{\bbE_{\eps}\bigg|n^{-1}\sum_{j=1}^n \eps_j g(X_j)\bigg|^2}\\
	&\leq \frac{1}{\sqrt{n}}\inf_{g \in \calG'_{N+1}}\sqrt{P_n g^2}\\
	&\leq \frac{1}{\sqrt{n}} \inf_{f \in \calF_{N+1}}\sqrt{P_n\{f\leq \delta_{N,1/2}\}}\\
	&\leq \frac{1}{\sqrt{n}} \inf_{f \in \calF_N}\sqrt{P_n\{f\leq \delta_{N,1/2}\}}\\
	&\leq \sqrt{\frac{\eps}{n}}\leq \eps \label{E13}.
	\end{align}
	By the induction assumption, we also have on the event $E_{N+1}\cap \calL$,
	\begin{align}
	\sup_{g \in \calG'_{N+1}} P_n g^2 \leq \sup_{f\in \calF_{N+1}} P_n\{f\leq \delta_{N,1/2}\}\leq \frac{r_{N+1}}{2} \leq r_N \label{E14}.
	\end{align}
	The bound for the Lipschitz norm of $\varphi'_k$ gives the following bound on the distance
	\begin{align}
	d^2_{P_n,2}\big(\varphi'_{N+1}\circ f; \varphi'_{N+1}\circ f\big)&=n^{-1}\sum_{j=1}^n\big|\varphi'_{N+1}\circ f(X_j)-\varphi_{N+1}'\circ g(X_j)\big|^2\\
	&\leq \bigg(\frac{2}{\delta}\sqrt{\frac{r_N}{\eps}}\bigg)^2 d^2_{P_n,2}(f,g) \label{E15}.
	\end{align}
	Therefore, on the event $E_{N+1}\cap \calL$, we get quite similarly to \eqref{E8},
	\begin{align}
	&\frac{1}{\sqrt{n}}\int_0^{(2\sup_{g \in \calG'_{N+1}} P_n g^2)^{1/2}} H_{d_{P_n,2}}^{1/2}\big(\calG'_{N+1};u\big)du\nn\\
	&\qquad \leq \frac{1}{\sqrt{n}}\int_0^{(2r_N)^{1/2}}H_{d_{P_n,2}}^{1/2}\bigg(\calF; \frac{\delta \sqrt{\eps} u}{2\sqrt{r_N}}\bigg)du\\
	&\qquad \leq \bigg( \frac{2^{1/2+\alpha/4}\sqrt{D}}{1-\alpha/2}\bigg)\bigg(\frac{r_N}{\eps}\bigg)^{\alpha/4} \frac{r_N^{1/2-\frac{\alpha}{4}}}{\sqrt{n} \delta^{\alpha/2}}\\
	&\qquad = \bigg( \frac{2^{1/2+\alpha/4} \sqrt{D}}{1-\alpha/2}\bigg)\sqrt{r_N \eps}.
	\end{align}
	We collect all bounds to see that on the event $\calE_{N+1}\cap \calL$,
	\begin{align}
	\sup_{f\in \calF_{N+1}} P_n\{f \leq \delta_{N+1}\}\leq \frac{r_{N+1}}{2}+ \bar{\eta} \sqrt{r_N \eps}
	\end{align} for some constant $\bar{\eta}>0$.
	
	Therefore, it follows that with a proper choice of constant $C>0$ in the recurrence relationship defining the sequence $\{r_k\}$, we have on the event $E_{N+1} \cap \calL$
	\begin{align}
	\sup_{f \in \calF_{N+1}} P_n\{f\leq \delta_{N+1}\}\leq C \sqrt{r_N \eps}=r_{N+1},
	\end{align} which proves the induction step for (i) and, therefore, the lemma is proved.
	Finally, using Lemma \ref{lem:sup} and the same arguments as \citep[Proof of Theorem 5]{Koltchinskii2002}, we can prove Theorem \ref{thm:main3}.
\end{proof}
\begin{lemma} \label{lem:sup} Suppose that for some $\alpha \in (0,2)$ and for some $D>0$ such that the condition \eqref{cond31} holds. Then for any constant $\xi>C^2$, for all $\delta\geq 0$ and
	\begin{align}
	\eps \geq \bigg(\frac{1}{n\delta^{\alpha}}\bigg)^{\frac{2}{2+\alpha}}\vee \sqrt{\frac{2\log n}{n}} \vee B_n \label{condkey},
	\end{align} and for all large enough $n$, the following:
	\begin{align}
	\bbP\bigg[\exists f\in \calF: P_n\{f \leq \delta\}\leq \eps \quad \mbox{and}\quad P\bigg\{f\leq \frac{\delta}{2}\bigg\}\geq \xi \eps\bigg]\leq 4/\eta  \log_2 \log_2 \eps^{-1}\exp\bigg\{-\frac{n\eps^2}{2}\bigg\}
	\end{align}
	and
	\begin{align}
	\bbP\bigg[\exists f\in \calF: P\{f \leq \delta\}\leq \eps \quad \mbox{and}\quad P_n\bigg\{f\leq \frac{\delta}{2}\bigg\}\geq \xi \eps\bigg]\leq 4/\eta \log_2 \log_2 \eps^{-1}\exp\bigg\{-\frac{n\eps^2}{2}\bigg\},
	\end{align} where $\eta$ is some constant.
\end{lemma}
\begin{proof}
	Observe that
	\begin{align}
	&\bbP\bigg[\exists f \in \calF: P_n\{f\leq \delta\}\leq \eps \wedge P\{f \leq \delta/2\}\geq \xi \eps\bigg]\nn\\
	&\qquad \leq \bbP\bigg[\bigg\{\exists f \in \calF:\{ P_n\{f\leq \delta\}\leq \eps\} \wedge \{P\{f \leq \delta/2\}\geq \xi \eps\}\bigg\} \cap E_N \bigg] + \bbP[E_N^c] \\
	&\qquad \leq \bbP\bigg[\bigg\{\exists f \in \calF_N \cap \calL\} \wedge \{P\{f \leq \delta/2\}\geq \xi \eps\}\bigg\} \cap E_N \bigg] + \bbP[E_N^c] \label{akala} \\
	&\qquad \leq \bbP\bigg[\bigg\{\exists f \in \calF_N\cap L \} \wedge \{P\{f \leq \delta_N \}\geq \xi \eps\}\bigg\} \cap E_N \bigg] + \bbP[E_N^c]\\
	&\qquad \leq \bbP\bigg[\bigg\{\exists f \in \calF_N \cap L\} \wedge \{P\{f \leq \delta_N \}> r_N \}\bigg\} \cap E_N \bigg] + \bbP[E_N^c] \label{akala2}\\
	&\qquad =\bbP[E_N^c] \label{akala3}\\
	&\qquad\leq 4N \exp\bigg(-\frac{n\eps^2}{2}\bigg)\\
	&\qquad \leq 4/\eta \big(\log_2\log_2 \eps^{-1}\big) \exp\bigg(-\frac{n\eps^2}{2}\bigg) \label{abc},
	\end{align} where \eqref{akala} follows from (ii) in Lemma \ref{lem:aux}, \eqref{akala2} follows from $r_N\leq (C\sqrt{\eps})^2 <\xi\eps$ for some constant $\xi>C^2$, and \eqref{akala3} follows from (i) in Lemma \ref{lem:aux}, and \eqref{abc} follows from the condition \eqref{cond33} in Lemma \ref{lem:aux}, which holds for $n$ sufficiently large. 
\end{proof}
Now, we return to prove Theorem \ref{thm:main3}.
\begin{proof}[Proof of Theorem \ref{thm:main3}]. Consider sequences $\delta_j:=2^{-j \frac{2}{\gamma}}$, and
	\begin{align}
	\eps_j:=\bigg(\frac{1}{n\delta_j^{\alpha'}}\bigg)^{\frac{1}{2+\alpha'}}, \qquad j\geq 0,
	\end{align} where 
	\begin{align}
	\alpha':=\frac{2\gamma}{2-\gamma}\geq \alpha.
	\end{align}
	Then, we have
	\begin{align}
	\eps_j=n^{(\gamma-2)/4}2^{j}.
	\end{align}
	Let
	\begin{align}
	\calE:=\{\exists j\geq 0 \quad \exists f \in \calF: P_n\{f\leq \delta_j\} \wedge P\{f \leq \delta_j/2\}\geq \xi \eps_j\}.
	\end{align}
	By Lemma \ref{lem:sup}, the condition \eqref{cond31} implies that there exists $\xi>0$ such that
	\begin{align}
	\bbP\big[\calE\big]& \leq 4/\eta \sum_{j=0}^{\infty} \big(\log_2\log_2 \eps_j^{-1}\big) \exp\bigg(-\frac{n\eps_j^2}{2}\bigg)\\
	& \leq 4 \upsilon' \sum_{j=0}^{\infty} \big(\log_2 \log_2 n\big)\sum_{j\geq 0} \exp\bigg[-\frac{n^{\frac{\gamma}{2}}}{2}2^{2j}\bigg]\\
	& \leq \upsilon  \big(\log_2 \log_2 n\big) \exp\bigg[-\frac{n^{\frac{\gamma}{2}}}{2} \bigg] \label{eq306}
	\end{align} for some $\upsilon,\upsilon'>0$. Now, since $\hat{\delta}_n(\gamma;f)\in (0,1]$, there exists some $j\geq 1$ such that
	\begin{align}
	\hat{\delta}_n(\gamma;f) \in (\delta_j,\delta_{j-1}].
	\end{align} Then, by the definition of $\hat{\delta}_n(\gamma;f)$ in \eqref{defhatdeltan}, we have
	\begin{align}
	P_n\{f\leq \delta_j\}&\leq P_n\{f\leq \hat{\delta}_n(\gamma;f) \}\\
	&\leq \sqrt{\delta_j^{-\gamma}  n^{-1+\frac{\gamma}{2}}}\\
	&=\eps_j.
	\end{align}
	Suppose that for some $f \in \calF$, the inequality $\zeta^{-1}\hat{\delta}_n(\gamma;f)\leq \delta_n(\gamma;f)$ fails, which leads to
	\begin{align}
	\delta_n(\gamma;f)&<\zeta^{-1}\hat{\delta}_n(\gamma;f)\\ 
	&\leq \frac{\delta_{j-1}}{\zeta}.
	\end{align}
	Then, if $\zeta>2^{1+\frac{2}{\gamma}}$, from the definition of $\delta_n(\gamma;f)$ in \eqref{defdeltan}, it holds that
	\begin{align}
	P\{f\leq \delta_j/2\}&\geq P\bigg\{f\leq \frac{\delta_{j-1}}{\zeta}\bigg\} \\
	&\geq \sqrt{\bigg(\frac{\delta_{j-1}}{\zeta}\bigg)^{-\gamma} n^{-1+\gamma/2}}\\
	&=\sqrt{\frac{1}{2}2^{2j}\zeta^{\gamma} n^{-1+\frac{\gamma}{4}}}\\
	&=\eps_j\sqrt{\frac{\zeta^{\gamma}}{4}}\\
	&>\xi \eps_j \label{AM1}
	\end{align} by choosing $\zeta$ sufficiently large, where $\xi$ is defined in Lemma \ref{lem:sup}. From \eqref{AM1}, it holds that $\zeta^{-1}\hat{\delta}_n(\gamma;f)\leq \delta_n(\gamma;f)$ fails for some $f \in \calF$ to hold with probability at most $\bbP[\calE]$. 
	
	Similarly, we can show that the event $\delta_n(\gamma;f) \leq \zeta \hat{\delta}_n(\gamma;f)$ fails for some $f \in \calF$ with probability at most $\bbP[\calE]$.
	
	By using the union bound, we finally obtain \eqref{muta}.
\end{proof}
\section{Proof of Lemma \ref{lem13}}\label{proofoflem13}
\begin{proof}  Observe that
	\begin{align}
	&\big|P_n(f\leq y)-P(f\leq y)\big|\nn\\
	&\qquad =\frac{1}{n}\sum_{i=1}^n \big(\bone\{f(X_i)\leq y\} - \bbP(f(X_i)\leq y)\big)+ \frac{1}{n}\sum_{i=1}^n \bbP(f(X_i)\leq y)- P(f\leq y) \label{ask1}.
	\end{align}	
	Now, let
	\begin{align}
	f_n(x):=\bone\{f(x)\leq y\}-\bbP(f(X_n)\leq y),
	\end{align} for all $x \in \calS$ and $n\in \bbZ^+$. It is clear that
	\begin{align}
	\|f_n\|_{\infty}\leq 1.
	\end{align}
	Now, let
	\begin{align}
	g(\bx):=\frac{1}{n}\sum_{i=1}^n f_i(x_i). 
	\end{align} 
	Then, we have
	\begin{align}
	\big|g(\bx)-g(\by)\big|&=\frac{1}{n}\bigg|\sum_{i=1}^n \big(f_i(x_i)-f_i(y_i)\big)\bigg|\\
	&\leq \frac{1}{n}\bigg|\sum_{i=1}^n \big(\bone\{f(x_i)\leq y\}-\bone\{f(y_i)\leq y\}\big)\bigg|\\
	&\leq \frac{1}{n}\sum_{i=1}^n\bone\{x_i \neq y_i\}. 
	\end{align}
	Then, by applying Lemma \ref{lem:HD}, it holds that
	\begin{align}
	\bbP\bigg[\bigg|\frac{1}{n}\sum_{i=1}^n f_i(X_i)\bigg|\geq t\sqrt{\frac{\tau_{\min}}{n}}\bigg]\leq 2\exp\big(-2t^2\big) \label{HD1}.
	\end{align}
	On the other hand, for all $y \in \bbR$, let
	\begin{align}
	\tilf_y(x):=\bone\{f(x)\leq y\}.
	\end{align} It is clear that $\sup_{y \in \bbR} \|\tilf_y\|_{\infty}=1$. Hence, by Lemma \ref{lem:key}, it holds that
	\begin{align}
	\bigg|\frac{1}{n}\sum_{i=1}^n \bbP(f(X_i)\leq y)- P(f\leq y)\bigg|&=\bigg|\frac{1}{n}\sum_{i=1}^n \bbE[\tilf_y(X_i)]-\bbE_{\pi}\big[ \tilf_y(X)\big]\bigg|\\
	&\leq \sqrt{B_n} \label{keygz},
	\end{align} where \eqref{keygz} follows from Lemma \ref{lem:berryessen} (with $M=1$) and the Cauchy–Schwarz inequality.
	
	From \eqref{ask1}, \eqref{HD1}, and \eqref{keygz}, we have
	\begin{align}
	\sup_{f \in \calF} \sup_{y \in \bbR} \big|P_n(f\leq y)-P(f\leq y)\big|\leq \sqrt{B_n}+t\sqrt{\frac{\tau_{\min}}{n}}
	\end{align} with probability at least $1-2\exp(-2t^2)$.
\end{proof}
\section{Proof of Lemma \ref{lem:lem14}}\label{prooflem:lem14}
Let $\delta>0$. Let $\varphi(x)$ be equal to $1$ for $x\leq 0$, $0$ for $x\geq 1$ and linear in between. Observe that
\begin{align}
F_f(y)&=P\{f\leq y\}\\
&\leq P\varphi\bigg(\frac{f-y}{\delta}\bigg)\\
&\leq P_n\varphi\bigg(\frac{f-y}{\delta}\bigg)+ \big\|P_n-P\big\|_{\tilde{\calG}_{\varphi}}\\
&\leq F_{n,f}(y+\delta)+\big\|P_n-P\big\|_{\tilde{\calG}_{\varphi}} \label{eq205},
\end{align}	
and
\begin{align}
F_{n,f}(y)&\leq P_n\{f\leq y\}\\
&\leq P_n\varphi\bigg(\frac{f-y}{\delta}\bigg)\\
&\leq P\varphi\bigg(\frac{f-y}{\delta}\bigg)+ \big\|P_n-P\big\|_{\tilde{\calG}_{\varphi}}\\
&\leq F_f(y+\delta)+\big\|P_n-P\big\|_{\tilde{\calG}_{\varphi}} \label{eq209}.
\end{align}

Now, by applying Lemma \ref{lem:HD} (see \eqref{A1}), we have
\begin{align}
\bbP\bigg[\|P_n-P\|_{\tilde{\calG}_{\varphi}}\geq \bbE\big[\|P_n-P\|_{\tilde{\calG}_{\varphi}}\big] + t\sqrt{\frac{\tau_{\min}}{n}}  \bigg]\leq 2 \exp\big(-2t^2\big) \label{bukhu1}.
\end{align}
From \eqref{bukhu1}, with probability at least $1- 2 \exp(-2t^2)$, it holds that
\begin{align}
\|P_n-P\|_{\tilde{\calG}_{\varphi}}\leq \bbE\big[\|P_n-P\|_{\tilde{\calG}_{\varphi}}\big] + t\sqrt{\frac{\tau_{\min}}{n}} \label{bukhu2}.
\end{align}
On the other hand, from Lemma \ref{lem:key}, we have
\begin{align}
\bbE\big[\big\|P_n-P\big\|_{\tilde{\calG}_{\varphi}}\big]\leq  2\bbE\big[\|P_n^0\|_{\tilde{\calG}_{\varphi}} \big]+B_n \label{F2eqmod}.
\end{align}

From \eqref{bukhu2} and \eqref{F2eqmod}, with probability at least $1- 2 \exp(-2t^2)$, it holds that
\begin{align}
\|P_n-P\|_{\tilde{\calG}_{\varphi}} \leq 2\bbE\big[\|P_n^0\|_{\tilde{\calG}_{\varphi}} \big]+ B_n+ t\sqrt{\frac{\tau_{\min}}{n}} \label{eq2030}.
\end{align}
From \eqref{eq205}, \eqref{eq209}, and \eqref{eq2030}, with probability at least $1- 2 \exp(-2t^2)$, we have
\begin{align}
L(F_f,F_{f,n})\leq \delta+ 2\bbE\big[\|P_n^0\|_{\tilde{\calG}_{\varphi}} \big]+B_n+ t\sqrt{\frac{\tau_{\min}}{n}} \label{eto}. 
\end{align}

Furthermore, by Talagrand's contraction lemma \cite{LedouxT1991book,Truong2022OnRC} for the class of function $\tilde{\varphi}(x):=\varphi(x)-1$, we have
\begin{align}
\bbE\big[\|P_n^0\|_{\tilde{\calG}_{\varphi}}\big]&\leq 2\bbE\bigg[\sup_{f\in \calF, y \in [-M,M]} \bigg|\sum_{i=1}^n \eps_i \frac{f(X_i)-y}{\delta}\bigg|\bigg]\\
&=   \frac{2}{\delta}\bbE\bigg[n^{-1}\sup_{f\in \calF} \bigg|\sum_{i=1}^n f(X_i)\bigg|\bigg]+ \frac{2M}{\delta n}\bbE\bigg|\sum_{i=1}^n \eps_i\bigg|\\
&\leq \frac{2}{\delta}\bbE\big[\|P_n^0\|_{\calF}\big]+ \frac{2M}{\delta \sqrt{n}} \label{eto2}.
\end{align}

Hence, by setting $\delta:=\sqrt{4\bbE[\|P_n^0\|_{\calF}]+4M/\sqrt{n}}$, from  \eqref{eto} and \eqref{eto2}, it holds with probability at least $1-2\exp(-2t^2)$ that
\begin{align}
L(F_f,F_{f,n})\leq 4\sqrt{\bbE[\|P_n^0\|_{\calF}]+M/\sqrt{n}} +B_n+ t\sqrt{\frac{\tau_{\min}}{n}}.
\end{align}
\section{Proof of Theorem \ref{thm1}} \label{proofthm1}
Fix $M>0$. Since $\calF_M \in \rm{GC}(P)$, we have
\begin{align}
\bbE\big[\|P_n-P\|_{\calF_M}\big]\to 0 \quad \mbox{a.s.} \quad n\to \infty,
\end{align} which, by Lemma \ref{lem:key} with $t=\sqrt{\log n}$,
\begin{align}
\bbE\big[\big\|P_n-P\big\|_{\calF_M}\big]&\geq  \frac{1}{2}\bbE\big[\|P_n^0\|_{\calF_M} \big]-\tilA_n \label{F2eqext}.
\end{align}
By taking $n\to \infty$, from \eqref{F2eqext}, we obtain
\begin{align}
\bbE\big[\|P_n^0\|_{\calF_M} \big]\to 0,
\end{align}
or
\begin{align}
\bbE\bigg[\bigg\|n^{-1}\sum_{i=1}^n \eps_i \delta_{X_i}\bigg\|\bigg]_{\calF_M}\to 0, \quad \mbox{as} \quad n\to \infty \label{laba}.
\end{align}
Furthermore, with $t=\sqrt{\log n}$, by Lemma \ref{lem:lem14}, we have
\begin{align}
\bbP\bigg\{\sup_{f\in \calF_M} L(F_{n,f},F_f)&\geq 4\sqrt{\bbE[\|P_n^0\|_{\calF}]+M/\sqrt{n}} +B_n+ \log n\sqrt{\frac{\tau_{\min}}{n}} \bigg\}\nn\\
&\qquad \qquad  \leq 2\exp\big(-2 \log n\big)=\frac{2}{n^2} \label{abiabi}.
\end{align}
It follows from \eqref{abiabi} that
\begin{align}
\sum_{n=1}^{\infty}\bbP\bigg\{\sup_{f\in \calF_M} L(F_{n,f},F_f)&\geq 4\sqrt{\bbE[\|P_n^0\|_{\calF}]+M/\sqrt{n}} +B_n+ \log n\sqrt{\frac{\tau_{\min}}{n}} \bigg\}\nn\\
&\qquad \qquad  \leq 2\sum_{n=1}^{\infty}\frac{1}{n^2}<\infty.
\end{align}
Hence, by Borel-Cantelli's lemma \citep{Billingsley}, $  B_n\to 0$, and \eqref{laba}, we obtain
\begin{align}
\sup_{f\in \calF_M} L(F_{n,f},F_f)\to 0, \qquad a.s. \label{tma1}.
\end{align}
Since $\sup_{f\in \calF} L(F_{n,f_M},F_{f_M})=\sup_{f\in \calF_M} L(F_{n,f},F_f)$, from \eqref{tma1}, we have
\begin{align}
\sup_{f\in \calF} L(F_{n,f_M},F_{f_M})=\sup_{f\in \calF_M} L(F_{n,f},F_f)\to 0, \qquad a.s. \label{tma2}.
\end{align}
Now, by \citep{Koltchinskii2002}, the following facts about Levy's distance holds:
\begin{align}
\sup_{f \in \calF} L(F_f,F_{f_M})\leq \sup_{f\in \calF} P\{|f|\geq M\}
\end{align}
and
\begin{align}
\sup_{f \in \calF} L(F_{n,f},F_{n,f_M})\leq \sup_{f\in \calF} P_n\{|f|\geq M\}.
\end{align}
Now, by the condition \eqref{cond41}, we have
\begin{align}
\sup_{f \in \calF} P\{|f|\geq M\} \to 0 \qquad a.s. \qquad M\to \infty \label{cond41x},
\end{align} so
\begin{align}
\sup_{f\in \calF} L(F_f,F_{f_M})\to 0, \qquad a.s. \qquad M\to \infty. 
\end{align}
To prove that
\begin{align}
\lim_{M\to \infty} \limsup_{n\to \infty}\sup_{f \in \calF} L(F_{n,f},F_{n,f_M})=0, \qquad a.s., 
\end{align} it is enough to show that
\begin{align}
\lim_{M\to \infty} \limsup_{n\to \infty}\sup_{f \in \calF} P_n\{|f|\geq M\}=0, \qquad a.s. 
\end{align}
To this end, consider the function $\varphi$ from $\bbR$ into $[0,1]$ that is equal to $0$ for $|u|\leq M-1$, is equal to $1$ for $|u|>M$ and is linear in between. We have
\begin{align}
\sup_{f\in \calF} P_n\{|f|\geq M\}&=\sup_{f \in \calF_M} P_n\{|f|\geq M\}\\
&\leq \sup_{f \in \calF_M} P_n\varphi(|f|)\\
&\leq \sup_{f \in \calF_M} P\varphi(|f|)+ \|P_n-P\|_{\calG}\\
&\leq \sup_{f \in \calF_M} P\{|f|\geq M-1\}+ \|P_n-P\|_{\calG},
\end{align} where
\begin{align}
\calG:=\big\{\varphi \circ |f|: f \in \calF_M \big\}.
\end{align}
Then, by using the same arguments to obtain \eqref{A1}, it holds with probability $1-2\exp(-2t^2)$ that 
\begin{align}
\|P_n-P\|_{\calG} &\leq  \bbE\big[\|P_n-P\|_{\calG}\big]+  t \sqrt{\frac{\tau_{\min}}{n}}\\
&\leq 2\bbE[\|P_n^0\|_{\calF}]+A_n +  t \sqrt{\frac{\tau_{\min}}{n}}
\label{eq2030b}.
\end{align}
Then, by setting $t=\log n$ and using the Borel-Cantelli's lemma \citep{Billingsley}, the following holds almost surely:
\begin{align}
\|P_n-P\|_{\calG} \leq 2\bbE\big[\|P_n^0\|_{\calG} \big]+ A_n+ \log n\sqrt{\frac{\tau_{\min}}{n}} \label{TH1}.
\end{align}
Now, since $\varphi\circ f \in \varphi \circ \calF_M$, by \eqref{laba} and Talagrand contraction lemma \cite{VaartWellnerbook,Truong2022OnRC}, we have
\begin{align}
\bbE\big[\|P_n^0\|_{\calG} \big]\to 0 \qquad \mbox{as} \quad n\to \infty \label{TH2}.
\end{align}
From \eqref{TH1} and \eqref{TH2}, we obtain
\begin{align}
\|P_n-P\|_{\calG}\to 0 \quad a.s. \label{TH3}.
\end{align}
Hence, we obtain (ii) from (i), the condition \eqref{cond41}, and \eqref{TH3}. 

To prove that (ii) implies (i), we use the following bound \citep{Kont05}
\begin{align}
\bigg|\int_{-M}^M t d(F-G)(t)\bigg|\leq \upsilon L(F,G),
\end{align} which holds with some constant $\upsilon=\upsilon(M)$ for any two distribution functions on $[-M,M]$. This bound implies that
\begin{align}
\|P_n-P\|_{\calF_M}&=\sup_{f \in \calF_M}|P_n f- Pf|\\
&\leq  \sup_{f \in \calF_M} \bigg|\int_{-M}^M t d(F_{n,f}-F_{f} )(t)\bigg| +M \sup_{f\in \calF_M} \big|P(|f|\geq M)-P_n(|f|\geq M)\big|\\
&\leq \upsilon \sup_{f \in \calF_M} L(F_{n,f},F_f) + M \sup_{f\in \calF_M} \big|P(|f|\geq M)-P_n(|f|\geq M)\big|\\
&\leq \upsilon \sup_{f \in \calF} L(F_{n,f},F_f)+M \sup_{f \in \calF_M} \big|P(|f|\geq M)-P_n(|f|\geq M)\big| \label{A13}.
\end{align}
Now, by Lemma \ref{lem13}, with probability at least $1-2\exp(-2t^2)$, the following holds:
\begin{align}
\sup_{y \in \bbR} \sup_{f \in \calF_M} \big|P_n(f\leq y)-P(f\leq y)\big|\leq  t\sqrt{\frac{\tau_{\min}}{n}}+\sqrt{B_n} \label{B13}.
\end{align} 
By setting $t=\sqrt{\log n}$ and using the Borel-Cantelli's lemma, from \eqref{B13}, we obtain
\begin{align}
\sup_{y \in \bbR} \sup_{f \in \calF_M} \big|P_n(f\leq y)-P(f\leq y)\big|\to 0, \quad a.s. \label{C13}.
\end{align}
Finally, from \eqref{B13}, \eqref{C13}, and (ii), we obtain (i). This concludes our proof of Theorem \ref{thm1}.
\section{Proof of Theorem \ref{thm4}}\label{thm4proof}
The proof is based on \citep[Proof of Theorem 9]{Koltchinskii2002}. Since $\calF$ is uniformly bounded, we can choose $M>0$ such that $\calF_M=\calF$. To prove the first statement, note that $\calF\in \rm{BCLT}(P)$ means that
\begin{align}
\bbE\big[\|P_n-P\|_{\calF}\big]=O(n^{-1/2}).
\end{align}
Now, from Lemma \ref{lem:key}, we have
\begin{align}
\bbE\big[\big\|P_n-P\big\|_{\calF}\big]&\geq  \frac{1}{2}\bbE\big[\|P_n^0\|_{\calF} \big]-\tilA_n \label{F2eqex}, 
\end{align} for all  $t> 0$. By applying \eqref{F2eqex}, it easy to see that
\begin{align}
\bbE\bigg[\bigg\|n^{-1}\sum_{i=1}^n \eps_i \delta_{X_i}\bigg\|_{\calF}\bigg]&=\bbE\big[\|P_n^0\|_{\calF} \big]\\
&\leq 2\tilA_n+2\bbE\big[\big\|P_n-P\big\|_{\calF}\big]\\
&\leq O\bigg(\sqrt{\frac{\log n}{n}}\bigg) \label{aka1}
\end{align}  since $\tilA_n=O\big(\sqrt{\frac{\log n}{n}}\big)$ by \eqref{defAtn}.

Now, from Lemma \ref{lem:lem14}, for $t=\sqrt{\log n}$,
\begin{align}
\bbP\bigg\{\sup_{f \in \calF} L(F_f,F_{f,n})\geq 4\sqrt{\bbE[\|P_n^0\|_{\calF}]+M/\sqrt{n}} +B_n+ \sqrt{\frac{\tau_{\min}\log n}{n}} \bigg\} \leq \frac{2}{n^2} \label{abixq}.
\end{align}
From \eqref{aka1} and \eqref{abixq}, it holds that
\begin{align}
\bbP\bigg\{\bigg(\frac{n}{\log n}\bigg)^{1/4}\sup_{f\in \calF} L(F_{n,f},F_f) \geq D\bigg\} \to 0 
\end{align} as $n\to \infty$ for some constant $D$, or
\begin{align}
\sup_{f\in \calF} L(F_{n,f},F_f)=O_P\bigg(\bigg(\frac{\log n}{n}\bigg)^{1/4}\bigg).
\end{align}

Now, recall
\begin{align}
\tilde{\calG}_{\varphi}:=\bigg\{\varphi \circ \bigg(\frac{f-y}{\delta}\bigg)-1: f\in \calF, y \in [-M,M] \bigg\}.
\end{align}
To prove the second statement, we use the following fact \citep[p.29]{Koltchinskii2002}:
\begin{align}
\bbE_{\eps}\big[\|P_n^0\|_{\tilde{\calG}_{\varphi}} \big]\leq \frac{d}{\sqrt{n}}\bigg[\int_0^{\sqrt{2}} H_{d_{P_n,2}}^{1/2}(\calF;\delta u) du+ \sqrt{\log \frac{4M}{\delta}}+1 \bigg] \label{aka10}
\end{align} for some constant $d$, which, under the condition \eqref{cond42}, satisfies
\begin{align}
\bbE_{\eps}\big[\|P_n^0\|_{\tilde{\calG}_{\varphi}} \big]\leq d\bigg[\frac{1}{\delta^{\alpha/2}} \sqrt{\frac{1}{n}}+ \frac{1}{\sqrt{n}} \bigg(\sqrt{\log \frac{4M}{\delta}}+1\bigg)\bigg] \label{mutat}.
\end{align} 
Now, by Lemma \ref{lem:lem14}, it holds for all $t>0$ and $\delta>0$ that 
\begin{align}
\bbP\bigg\{\sup_{f \in \calF} L(F_f,F_{f,n})\geq \delta+ \bbE\big[\|P_n^0\|_{\tilde{\calG}_{\varphi}} \big]+ B_n+ t\sqrt{\frac{\tau_{\min}}{n}}\bigg\} \leq 2\exp(-2t^2) \label{etopaphu}. 
\end{align} 
Since $\bbE\big\|n^{-1}\sum_{i=1}^n\eps_i \delta_{X_i}\big\|_{{\tilde{\calG}_{\varphi}}}= \bbE[\|P_n^0\|_{\tilde{\calG}_{\varphi}}]\leq d\big[\sqrt{\frac{\log n}{n}}\delta^{-\alpha/2} + \frac{1}{\sqrt{n}} \big(\sqrt{\log \frac{4M}{\delta}}+1\big)\big]$, from \eqref{etopaphu}, for all $t>0$, we have
\begin{align}
&\bbP\bigg\{\sup_{f \in \calF} L(F_f,F_{f,n})\geq \delta+ d\bigg[\sqrt{\frac{\log n}{n}}\delta^{-\alpha/2} + \frac{1}{\sqrt{n}} \bigg(\sqrt{\log \frac{4M}{\delta}}+1\bigg)\bigg]+  B_n+ t\sqrt{\frac{\tau_{\min}}{n}}\bigg\}\nn\\
&\qquad  \leq 2 \exp\big(-2t^2\big) \label{abiem}.
\end{align}
Now, by choosing $t=\sqrt{\log n}$ and $\delta= (\log n) n^{-\frac{1}{2+\alpha}}$, we have  
\begin{align}
\bbP\bigg\{\sup_{f \in \calF} L(F_f,F_{f,n})\geq  \nu (\log n)n^{-\frac{1}{2+\alpha}}\bigg\}  \leq \frac{2}{n^2} \label{abiemmod}
\end{align} for some constant $\nu$ and for $n\geq N_0$ for some finite $N_0$ big enough.

From \eqref{abiemmod}, we have
\begin{align}
\sum_{n=1}^{\infty}\bbP\bigg\{\sup_{f \in \calF} L(F_f,F_{f,n})\geq  \nu (\log n) n^{-\frac{1}{2+\alpha}}\bigg\}  \leq N_0+\sum_{n=N_0}^{\infty}\frac{2}{n^2} <\infty.
\end{align}
Hence, by Borel-Cantelli's lemma \citep{Billingsley}, it holds that
\begin{align}
\sup_{f \in \calF} L(F_f,F_{f,n})=O_P\big((\log n) n^{-\frac{1}{2+\alpha}}\big), \qquad a.s.
\end{align}
This concludes our proof of Theorem \ref{thm4}.
%\bibliographystyle{unsrtnat}
%\bibliography{isitbib}
\end{document}